\documentclass[openany]{nownew} % creates the journal version

\usepackage{natbib}
\usepackage{epsf}
\usepackage{epsfig}
\usepackage{amsmath}
\usepackage{subfigure}
\usepackage{amssymb}
\usepackage{multirow}
\usepackage{algorithm}
\usepackage{algorithmic}

\def\A{{\bf A}}
\def\a{{\bf a}}
\def\B{{\bf B}}
\def\b{{\bf b}}
\def\C{{\bf C}}

\def\D{{\bf D}}
\def\d{{\bf d}}
\def\E{{\bf E}}
\def\e{{\bf e}}

\def\K{{\bf K}}

\def\H{{\bf H}}
\def\G{{\bf G}}
\def\I{{\bf I}}
\def\R{{\bf R}}
\def\X{{\bf X}}
\def\Y{{\bf Y}}
\def\Q{{\bf Q}}
\def\bq{{\bf q}}
\def\PP{{\bf P}}

\def\S{{\bf S}}
\def\T{{\bf T}}
\def\x{{\bf x}}
\def\y{{\bf y}}
\def\z{{\bf z}}
\def\Z{{\bf Z}}
\def\M{{\bf M}}
\def\m{{\bf m}}

\def\N{{\bf N}}
\def\U{{\bf U}}
\def\u{{\bf u}}
\def\V{{\bf V}}
\def\v{{\bf v}}
\def\W{{\bf W}}

\def\0{{\bf 0}}
\def\1{{\bf 1}}

\def\FM{{\mathcal F}}

\def\XM{{\mathcal X}}

\def\OM{{\mathcal O}}

\def\GM{{\mathcal G}}

\def\RB{{\mathbb R}}

\def\Pii{\mbox{\boldmath$\Pi$\unboldmath}}

\def\Omeb{\mbox{\boldmath$\Omega$\unboldmath}}

\def\Si{\mbox{\boldmath$\Sigma$\unboldmath}}
\def\sib{\mbox{\boldmath$\sigma$\unboldmath}}
\def\Gam{\mbox{\boldmath$\Gamma$\unboldmath}}
\def\lamb{\mbox{\boldmath$\lambda$\unboldmath}}
\def\Lam{\mbox{\boldmath$\Lambda$\unboldmath}}

\def\vp{\mbox{\boldmath$\varphi$\unboldmath}}

\def\argmin{\mathop{\rm argmin}}
\def\liml{\mathop{\lim}\limits}

\def\nystrom{{Nystr\"{o}m} }
\def\CUR{{CUR} }

\def\sgn{\mathrm{sgn}}
\def\tr{\mathrm{tr}}
\def\rg{\mathrm{range}}
\def\rk{\mathrm{rank}}
\def\diag{\mathrm{diag}}
\def\dg{\mathsf{dg}}
\def\vect{\mathsf{vec}}
\def\nnz{\mathsf{nnz}}
\def\poly{\mathsf{poly}}

\def\norm3{{|\!|\!|}}
\newcommand{\nm}{{\vert\!\vert\!\vert}}

\title{The Singular Value Decomposition, Applications and  Beyond}

\author{
Zhihua Zhang \\
Shanghai Jiao Tong University \\
zhihua@sjtu.edu.cn
}

\begin{document}

% the following settings can be set or left blank at first
\copyrightowner{Z.~Zhang}
% \volume{1}
% \issue{3}
% \pubyear{2014}
% \copyrightyear{2013}
% \isbn{978-0521833783}
% \doi{1234567890}
% \firstpage{23}
% \lastpage{94}

\frontmatter  % title page, contents, catalog information

\maketitle

\tableofcontents

\mainmatter

\begin{abstract}
The singular value decomposition (SVD) is not only a classical theory  in matrix computation and analysis, but also is a powerful tool in machine learning and modern data analysis. In
this tutorial  we first study the basic notion of SVD and then show the central role of SVD in matrices. Using majorization theory, we consider variational principles of singular values and eigenvalues.
Built on SVD and a theory of symmetric gauge functions, we discuss unitarily invariant norms, which are then used to formulate  general results for matrix low rank approximation. We study the subdifferentials of unitarily invariant norms.
These results would be potentially useful  in many machine learning problems such as matrix completion and matrix data classification. Finally,  we discuss matrix low rank approximation and its recent developments  such as randomized SVD, approximate matrix multiplication,  CUR decomposition, and Nystr\"{o}m approximation. Randomized algorithms are important approaches  to large scale SVD as well as fast matrix computations.  
\end{abstract}
%\begin{keywords} nonconvex penalization, Laplace exponents, subordinators, Bernstein functions,
%local shrinkage
%\end{keywords}
%

%\tableofcontents

%\newpage

\chapter{Introduction}

The singular value decomposition (SVD) is a classical  matrix theory and a key computational technique,  and it
has also received wide applications in science and engineering.  Compared with an eigenvalue decomposition (EVD) which only works on some of square matrices, SVD applies to all matrices. Moreover, many matrix concepts and properties such as matrix pseudoinverses, variational principles and unitarily invariant norms can be induced from SVD.
Thus,  SVD plays a fundamental role in matrix computation and analysis.

Furthermore,
due to recent great developments of machine learning,  data mining and theoretical computer science,
SVD has been found to be more and more important. It is not only a powerful tool and theory but also  an art.
SVD makes matrices become  a ``Language" of data science.

%From today view, singular value decomposition or singular values are misleading terms.
The terminology of \emph{singular values} has been proposed by Horn in 1950 and 1954~\citep{Horn:1950,Horn:1954}.
The first proof of the SVD for general $m\times n$ matrices might be given by  \cite{EckartYoung:1939}.
But the theory of singular values can date back to  the 19th century when it had been studied
by the Italian differential geometer E. Beltrami, the French algebraist C. Jordan, the English mathematician  J. J. Sylvester, the French mathematician L. Autonne, etc.
Please refer to Chapter 3 of \cite{horn1991matrix} in which the authors presented an excellent historical retrospection about  the SVD or theory of singular values.

There is a rich literature involving singular values or SVD.  Chapter 3 of \cite{horn1991matrix}  provides  exhaustive studies  about inequalities of singular values as well as unitarily invariant norms, and the primary focus is on the matrix theory. The books by \cite{Watkins:1991,DemmelBook:1997,golub2012matrix, trefethenbau} present a detailed introduction to SVD, the primary focus of which is on numerical linear algebra.

This tutorial  is motivated  by recent successful applications  of SVD in machine learning and theoretical computer science~\citep{Hastie:Book:SL,Burges:2010,halko2011finding,woodruff2014sketching,Mahoney:2011,hopcroft2012computer}. The primary focus is on a perspective of machine learning.
The main purpose of the tutorial  includes two aspects.
First, it provides a  systematic tutorial to the SVD theory and illustrates its functions in matrix and data analysis.
%for the machine learning community.
Second, it provides an advanced review about recent developments of the SVD theory in
applications of machine learning and theoretical computer science.

\begin{table}
\centering
\caption{Comparison of Matrix Factorization Methods}
\begin{tabular}{l|l|l|l}
\hline
Matrices & Geometry   & Data  & Computation \\ \hline
$m\times n$  & Polar   & CX  & QR  \\
$m\times n$ & SVD     & \CUR  & QR  \\
SPSD      &  Spectral   & \nystrom  & (Incomplete) Cholesky \\
\hline
\end{tabular}
\label{tab:mf}
\end{table}

\section{Roadmap}

The preliminaries about matrices please refer to the book of   \cite{Horn:1985}.
This tutorial  involves matrix differential calculus, majorization theory, and symmetric gauge functions.
For them, the detailed materials can be found in \cite{MacnusNeudecker,MarshallOlkinArnold,Schatten,BhatiaMatrix}.
In Chapter~\ref{ch:prelin}  we review some preliminaries  such as Kronecker produces and vectorization operators, majorization theory, and derivatives.

In Chapter~\ref{ch:svd}  we introduce the basic notion of SVD, including the existence, construction, and uniqueness.
We then rederive some important matrix  concepts and properties via SVD. We also study generalized SVD problems, which are concerned with joint decomposition  of two matrices.  In Chapter~\ref{ch:app1}  we further illustrate the application of SVD in definition of the matrix pseudoinverse  and solution of the Procrustes analysis problem. We discuss the role that SVD plays in subspace machine learning methods.

From the viewpoint of computation and modern data analysis, matrix factorization techniques should be the most important issue of matrices.
In Table~\ref{tab:mf} we summary matrix factorization methods, which are categorized into three types. In particular, the Polar decomposition, SVD, and spectral decomposition consider geometric representation of a data matrix, whereas the CX, CUR, and Nystr\"{o}m dcompositions consider a compact representation of the data themselves. That is, the latters use a portion of the data to
represent the whole data. The primary focus of the QR and Cholesky decomposition is on fast  computation. 
In Chapter~\ref{ch:variant} we give reviews about the QR and CUR decompositions.

In Chapter~\ref{ch:variation}  we consider variational principles for singular values and eigenvalues. Specifically, we apply matrix differential calculus to rederive the von Neumann theorem~\citep{Neumann:1937} and the Ky Fan theorem~\citep{Fan:1951}.
Accordingly, we give some inequalities for singular values and eigenvalues.

Built on the inequalities for singular values, Chapter~\ref{ch:uinorm}   discusses unitarily invariant norms. 
Unitarily invariant norms include the nuclear norm, Frobenius norm and spectral norm as their special cases.  
There is a one-to-one correspondence between a unitarily invariant norm of a matrix and a symmetric gauge function on the singular values of the matrix. This helps us to establish properties of unitarily invariant norms.
%We also show that the connection of
%the unitarily invariant norm with the matrix operator norm and the matrix-vectorization norm.

In Chapter~\ref{ch:subdiff}  we study subdifferentials of  unitarily invariant norms. We especially present the subdifferentials of
the spectral norm and the nuclear norm as well as the applications in matrix low rank approximation.
We illustrate several examples  in optimization, which are solved via the subdifferentials of
the spectral  and nuclear norms.
The subdifferentials of  unitarily invariant norms would have potentially useful in machine learning and optimization.

Matrix low rank approximation is a promising theme in machine learning and theoretical computer science.
Chapter~\ref{ch:lowrank}  gives two important theorems about matrix low rank approximation based on errors of unitarily invariant norms. The first one is an extension of the ordinal least squares estimation problem. The second one was proposed by \cite{mirsky:1960}, which is an extension of the novel Eckart Young theorem~\citep{EckartYoung:1936}.
We also discuss approximate matrix multiplication, which can be regarded as an inverse process of matrix low rank approximation. 

In Chapter~\ref{ch:lsma} we study 
randomized SVD,   CUR approximation, and Nystr\"{o}m methods to make the applications scalable. 
The randomized SVD  and CUR approximation can  be also  viewed as  matrix low rank approximation techniques.
%The SVD  leads us to a geometrical representation, which has little concrete meaning. 
%This makes it difficult for us to understand and interpret the data in question.
%Therefore, it is great interest to represent a data matrix in terms of a small number part  of the matrix. This is a job that
%the CUR decomposition does. In Chapter 9 we study the CUR decomposition as well as its special case: the Nystr\"{o}m approximation.
%The CUR is parallel with a truncated SVD, and
The Nystr\"{o}m approximation is a special case of the CUR decomposition and  has been widely used to speed up kernel methods.
%The 

\section{Notation and Definitions}

Throughout this tutorial, vectors and matrices are denoted by
boldface lowercase letters  and  boldface uppercase letters, respectively.
$\RB_{+}^n =\{\u =(u_1, \ldots, u_n)^T \in \RB^n: u_j \geq 0 \mbox{ for } j=1, \ldots, n\}$
and $\RB_{++}^n =\{\u =(u_1, \ldots, u_n)^T \in \RB^n: u_j > 0 \mbox{ for } j=1, \ldots, n\}$.
Furthermore, if $\u \in \RB_{+}^n$ (or $\u \in \RB_{++}^n$), we also denote $\u\geq 0$ (or $\u > 0$).

Given a vector $\x =(x_1, \ldots, x_n)^T \in \RB^n$, let $|\x| = (|x_1|, \ldots, |x_n| )^T$,
let $\|\x\|_p= (\sum_{i=1}^n |x_i|^p)^{1/p}$ for $p\geq 1$ be the $\ell_p$-norm of $\x$, and let $\diag(\x)$ be an $n\times n$ diagonal matrix with the $i$th diagonal element as $x_i$.

Let $[m]=\{1, 2, \ldots, m\}$,
$\I_m$ be the $m\times m$ identity matrix,  $\1_m$ be the $m\times 1$ vector of ones, and $\0$ be the zero vector or matrix with appropriate size.  Let $\A\oplus \B = \begin{bmatrix} \A & \0 \\ \0 & \B \end{bmatrix}$.

For a matrix $\A=[\a_1, \a_2, \ldots, \a_n ]=[a_{ij}] \in \RB^{m{\times} n}$,  $\A^T$ denotes the transpose of $\A$,  $\rk(\A)$ denotes the rank,     $\rg(\A)$
represents the range which is the space spanned by the columns (i.e., $\rg(\A) =  \{\y\in \RB^m: \y = \A \x \mbox{ for some } \x \in \RB^n \}= \mathrm{span}\{\a_1, \a_2, \ldots, \a_n \}$),
$\mathrm{null} (\A)$ is the null space (i.e., $\mathrm{null}(\A) =\{\x: \A \x = 0\}$), and for $p=\min\{m, n\}$ $\dg(\A)$ denotes the $p$-vector with $a_{ii}$ as the $i$th  element.
Sometimes we also use Matlab Colon to represent a submatrix of $\A$. For example, let $I \subset [m]$ and $J \subset [n]$. $\A_{I, J}$
denotes the submatrix of $\A$ with rows indexed by $I$ and columns indexed by $J$, $\A_{I, :}$ consists of those rows of $\A$ in $I$, and $\A_{:, J}$ consists of those columns of $\A$ in $J$.

Let $\|\A\|_F= \sqrt{\sum_{ij} a_{ij}^2}$  denote the Frobenius norm, $\|\A\|_2$ denote the spectral norm, and $\|\A\|_*$ denote the nuclear norm.
When $\A$ is square, we let $\A^{-1}$ be the inverse  (if exists) of $\A$,  $\tr(\A) = \sum_{i=1}^n a_{ii}$ be the trace, and $\mathrm{det}(\A)$ be the determinant of $\A$.

An $m\times m$ real matrix $\U$ is symmetric if $\A^T=\A$, and skew-symmetric if $\A^T = - \A$, and normal if $\A \A^T = \A^T \A$. Clearly, symmetric and skew-symmetric matrices are normal.
An $m\times m$ real matrix $\U$ is said to be orthonormal (or orthogonal) if $\U^T \U =\U \U^T=\I_m$. An $m\times n$ real matrix $\Q$ for $m>n$ is  column orthonormal (or column orthogonal) if $\Q^T \Q = \I_n$, and a column orthonormal $\Q$ is always able to be extended to an orthonormal matrix. A matrix $\M \in \RB^{m\times m}$ is said to be  positive semidefinite (PSD) or positive definite if  for any nonzero vector $\x \in \RB^{m}$ $\x^T \M \x \geq 0$ or $\x^T \M \x>0$.

\chapter{ Preliminaries}
\label{ch:prelin}

In this chapter we present some preliminaries, including Kronecker products and vectorization operators,
majorization theory, and derivatives.
We list some basic results  that will be used in this monograph  but  omit their detailed derivations.

\section{Kronecker Products and Vectorization Operators}

Given two matrices$\A \in  \RB^{m\times n}$ and $\B \in \RB^{p\times q}$, the  the Kronecker product of $\A$ and $\B$
is defined by
\[
\A \otimes \B \triangleq  \left[\begin{array}{ccc}
a_{11}\B & \cdots & a_{1n}\B \\
\vdots & \ddots & \vdots \\
a_{m 1}\B & \cdots & a_{m n}\B
\end{array} \right],
\]
which is $mp \times nq$. The following properties can be found in \cite{Muirhead:1982}.
\begin{proposition} \label{prof:iron} The Kronecker product has the following properties.
\begin{enumerate}
\item[(a)] $(\alpha \A) \otimes (\beta \B) = \alpha \beta (\A \otimes  \B)$ for any scalars $\alpha, \beta \in \RB$.
\item[(b)] $(\A \otimes \B)^T = \A^T \otimes \B^T$.
\item[(c)] $(\A \otimes \B)\otimes \C = \A \otimes (\B \otimes \C)$.
\item[(d)]  If $\A$ and $\C$ are both $m\times n$ and $\B$ is $p\times q$, then $(\A {+} \C) \otimes \B = \A \otimes \B {+} \C \otimes \B$ and $\B \otimes (\A {+} \C) = \B \otimes \A {+} \B \otimes \C$.
\item[(e)] If $\A$ is $m\times n$, $\B$ is $p\times q$, $\C$ is $n\times r$, and $\D$ is $q\times s$, then
\[
(\A \otimes \B)(\C \otimes \D) = (\A \C) \otimes (\B \D).
\]
\item[(f)] If $\U$ and $\V$ are both orthogonal matrices, so is $\U \otimes \V$.
\item[(g)] If $\A$ and $\B$ are symmetric positive semidefinite (SPSD), so is $\A \otimes \B$.
\end{enumerate}
\end{proposition}

Kronecker products often work with  vectorization operators together.
Let $\vect(\A)= (a_{11},  \ldots,  a_{m1}, a_{12}, \ldots,  a_{mn})^{T}\in  \RB^{mn}$ be vectorization of the matrix $\A \in \RB^{m\times n}$.
The following lemma gives the connection between Kronecker products and vectorization operators.
\begin{lemma}  \label{lem:kron2} $\quad$
\begin{enumerate}
\item[(1)] If $\B$ is $p\times m$, $\X$ is $m\times n$, and $\C$ is $n\times q$, then
\[
\vect(\B \X \C) = (\C^T \otimes \B) \vect(\X).
\]
\item[(2)]  If $\A \in \RB^{m\times n}$,  $\B \in \RB^{n \times p}$, and $\C\in \RB^{p\times m}$, then
\[
\tr(\A \B \C ) = (\vect(\A^T))^T (\I_m \otimes \B) \vect(\C).
\]
\item[(3)]  If $\A \in \RB^{m\times p}$, $\X\in \RB^{n \times p}$,  $\B \in \RB^{n \times n}$, and $\C\in \RB^{p\times m}$, then
\begin{align*}
\tr(\A \X^T \B \X \C) &= (\vect(\X))^T ((\C \A)^T \otimes \B ) \vect(\X) \\
   & =  (\vect(\X))^T ((\C \A) \otimes \B^T) \vect(\X).
      \end{align*}
 \end{enumerate}
\end{lemma}

\section{Majorization}

Given a vector $\x=(x_1, \ldots, x_n)^T \in \RB^{n}$, let  $\x^{\downarrow}=(x_{1}^{\downarrow}, \ldots, x_n^{\downarrow} )$
be  such a permutation of $\x$ that $x_{1}^{\downarrow} \geq x_{2}^{\downarrow} \geq \cdots \geq x_n^{\downarrow}$.
%and $y_{1}^{\downarrow} \geq y_{2}^{\downarrow}  \geq \cdots \geq y_n^{\downarrow}$.
%and $\y=(y_1, \ldots, y_n)^T \in \RB^{n}$, let $\x^{\downarrow}=(x_{1}^{\downarrow}, \ldots, x_n^{\downarrow} )$  and $\y^{\downarrow}=(y_{1}^{\downarrow}, \ldots, y_n^{\downarrow} )$  be the permutations
%of $\x$ and $\y$, respectively,  where $x_{1}^{\downarrow} \geq x_{2}^{\downarrow} \geq \cdots \geq x_n^{\downarrow}$
%and $y_{1}^{\downarrow} \geq y_{2}^{\downarrow}  \geq \cdots \geq y_n^{\downarrow}$.
Given two vectors $\x$ and $\y \in \RB^n$,  $\x\geq \y$ means $x_i - y_i\geq 0$ for all $i\in [n]$.
We say that  $\x$ is majorized by $\y$ (denoted $\x \prec \y$) if $\sum_{i=1}^k x^{\downarrow}_i \leq \sum_{i=1}^k y^{\downarrow}_i $  for $k=1, \ldots, n-1$ and $\sum_{i=1}^n x^{\downarrow}_i = \sum_{i=1}^n y^{\downarrow}_i$.
Similarly,  $\x \succ \y$ if  $\sum_{i=1}^k x^{\downarrow}_i \geq \sum_{i=1}^k y^{\downarrow}_i $ for $k=1, \ldots, n{-}1$
and $\sum_{i=1}^n x^{\downarrow}_i = \sum_{i=1}^n y^{\downarrow}_i$.

We say  that  $\x$ is weakly submajorized by $\y$  (denoted $\x \prec_{w} \y$)   if $\sum_{i=1}^k x^{\downarrow}_i \leq \sum_{i=1}^k y^{\downarrow}_i $  for $k=1, \ldots, n$,
and $\x$ is weakly superrmajorized by $\y$  (denoted $\x \prec^{w} \y$)   if  $\sum_{i=1}^k x^{\downarrow}_i \geq \sum_{i=1}^k y^{\downarrow}_i $ for $k=1, \ldots, n$,

An  $n \times n$ matrix $\W=[w_{ij}]$ is said to be  doubly stochastic if the $w_{ij}\geq 0$,  $\sum_{j=1}^n w_{ij} =1$ for all $i \in [n]$, and  $\sum_{i=1}^n w_{ij} =1$ for all $j \in [n]$.  Note that if $\Q=[q_{ij}] \in \RB^{n\times n}$ is orthonormal, then $\W\triangleq [q_{ij}^2]$ is a doubly stochastic matrix. It is thus called \emph{orthostochastic}.

The following three lemmas are classical results in majorization theory. They will be used in investigating unitarily invariant norms.

\begin{lemma} \citep{HardyLittlewoodPolya} \label{lem:001}   Given two vectors $\x, \y \in \RB^{n}$, then $\x \prec \y$ if and only if there exists a doubly stochastic matrix $\W$ such that $\x = \W \y$.
\end{lemma}

\begin{lemma}[Birkhoff] \label{lem:dsm} Let $\W \in \RB^{n\times n}$. Then it is a doubly stochastic matrix if and only if it can be expressed as a convex combination of a set of permutation matrices.
\end{lemma}

\begin{lemma} \label{lem:shur} Let $u_1, \ldots, u_n$ and $v_1, \ldots, v_n$ be given nonnegative real numbers such that $u_1\geq \cdots \geq u_n$
and $v_1\geq \cdots \geq v_n$. If
\[ \prod_{i=1}^k u_i \leq \prod_{i=1}^k v_i \; \mbox{ for } k\in [n],\]
then
\[
\sum_{i=1}^k u_i \leq \sum_{i=1}^k v_i \;  \mbox{ for }  k \in [n].
\]
More generally, assume $f$ is a real-valued function such that $f(\exp(u))$ is increasing and convex. Then
\[
\sum_{i=1}^k f(u_i) \leq \sum_{i=1}^k f(v_i) \;  \mbox{ for }  k \in [n].
\]
\end{lemma}

%For a square matrix $\A$, we let $\tr(\A)$ be  its trace.
%$\RB_{+}^p=\{\u =(u_1, \ldots, u_p)^T \in \RB^p: u_j \geq 0 \mbox{ for } j=1, \ldots, p\}$
%and $\RB_{++}^p=\{\u =(u_1, \ldots, u_p)^T \in \RB^p: u_j > 0 \mbox{ for } j=1, \ldots, p\}$.
%Furthermore, if $\u \in \RB_{+}^p$ (or $\u \in \RB_{++}^p$),
%we also denote $\u\geq 0$ (or $\u > 0$). Let $C^k (S)$ be  the set of $k$ times continuously differentiable functions from $S$ to $
%\RB$ where $S \subset \RB$.

\section{Derivatives and Optimality}

First let $f : \XM \subset \R^n \to \RB$ be a continuous function. The  directional derivative of $f$ at $\bar{\x}$ in a direction $\u \in \XM$ is defined as
\[
f'(\bar{\x};  \u) = \lim_{t  \downarrow 0} \frac{f( \bar{\x}+t \u) - f(\bar{\x})}{t},
\]
when this limit exists. When the directional derivative $f'(\bar{\x}; \u)$ is linear in $\u$ (that is, $f'(\bar{\x}; \u)=\langle \a, \u \rangle$ for some $\a \in \XM$) then we say $f$ is (G\^{a}beaux) differentiable at $\bar{\x}$, with derivative $\nabla f(\bar{\x}) =\a$. If $f$ is differentiable at every point in $\XM$ then we say $f$ is differentiable on $\XM$.

When $f$ is not differentiable but convex, we consider a notion of subdifferentials.  We say  $\z$ is the subgradient of $f$ at $\bar{\x}$ if it satisfies
\[
f(\bar{\x}) \leq f(\x) + \langle \z, \bar{\x}-\x \rangle \; \mbox{  for all points  } \; \z  \in \XM.
\]
The set of subgradients is called the subdifferential, and denoted by $\partial f(\bar{\x})$. The subdifferential is always a closed convex set. The following result shows a connection between subgradients and directional derivatives.

\begin{lemma}[Max Formula] \label{lem:subg-dd} If the function $f : \XM \to (-\infty, +\infty]$ is convex, then any point $\bar{\x}$ in $\mathrm{core}(\mathrm{dom} f )$ and any direction $\u$ in $\XM$ satisfy
\[
f'(\bar{\x}; \u) = \max\; \big\{ \langle \z, \u\rangle: \z \in \partial f(\bar{\x}) \big\}.
\]
\end{lemma}
The further details of these results can be found from \cite{BorweinLewis}.  The following lemma then shows the fundamental role of subgradients in optimization.
\begin{lemma} \label{lem:optima} For any proper convex function $f: \XM \to (-\infty, +\infty]$, the point $\bar{\x}$ is a minimizer of $f$ if and only if the condition $\0 \in \partial f(\bar{\x})$ holds.
\end{lemma}

Now let $f$ be a differentiable  function from $\RB^{m\times n}$  to $\RB$. For a matrix $\X=[x_{ij}] \in \R^{m\times n}$,  $\frac{d f(\X)}{d \X}= \left( \frac{d f}{d x_{ij}} \right)$ ($m\times n$)
defines the derivative of $f$ w.r.t.\ $\X$. The Hessian matrix of $f$ w.r..t.\ $\X$ is defined as $\frac{d^2 f(\X)}{d \vect(\X) d \vect(\X)^T}$, which is an $mn {\times} mn$ matrix. Let us see an example.
\begin{example}
We define the function  $f$ as
\[
f(\X) =\tr(\X^T \M \X),
\]
where $\M =[m_{ij}] \in \RB^{m\times m}$ is a given constant matrix.  It is directly computed that $\frac{d f }{d x_{ij}}= \sum_{l=1}^m (m_{il} + m_{li}) x_{lj}$. This  implies that $\frac{d f}{d \X}= (\M + \M^T) \X$.
In fact, the derivative can be computed as follows.  %(1) compute the differential $d f$, (2) vectorize to obtain $d $
Compute
\[
d f = \tr(d \X^T \M \X + \X^T \M  d\X) = \tr( (\M + \M^T) \X d \X^T).
\]
We thus have that $\frac{d f }{d \X}= (\M + \M^T) \X$.

Additionally, it follows from Lemma~\ref{lem:kron2} that $f(\X) = \vect(\X)^T (\I_n \otimes \M) \vect(\X)$. Thus,
we have
\[
\frac{d f }{d \vect(\X)} = \vect\Big( \frac{d f}{d \X} \Big)=  [\I_n \otimes (\M + \M^T)] \vect(\X),
\]
and hence,
\[
\frac{d^2 f(\X)}{d \vect(\X) d \vect(\X)^T} = \I_n \otimes (\M + \M^T).
\]
\end{example}

%%%%%%%%%%%%%%%%%%%%%%%%%%%%%%%%%%%%%
%%%%%%%%%%%%%%%%%%%%%%%%%%%%%%%%%%%%%%%
\chapter{The Singular Value Decomposition}
\label{ch:svd}

The singular value decomposition (SVD) is a classical matrix theory and computational tool.
In modern data computation and  analysis, SVD becomes more and more important.
In this chapter we aim to provide a systematical  review  about the basic principle of SVD.

We will see that there are four approaches to SVD. The first approach is depart from the spectral decomposition of a symmetric positive semidefinite (SPSD) matrix. The second approach gives a construction process via induction.
In the third approach   the SVD problem is equivalently formulated into an eigenvalue decomposition problem of a symmetric matrix  (see Theorem~\ref{thm:h-a}).
The fourth approach is based on the equivalent relationship between the SVD and polar decomposition (see Theorem~\ref{thm:polarD}).

We also study uniqueness of SVD (see Theorem~\ref{thm:svdvector} and Corollary~\ref{cor:svd-uni}).  These  results will be used in derivation of subdifferentials of unitarily invariant norms (see Chapter~\ref{ch:subdiff}).
Additionally, we present a generalized SVD (GSVD), which addresses joint decomposition problems of two matrices.
When the two matrices form a column orthonormal matrix, the resulting  GSVD process is called the CS decomposition.

\section{Formulations}

Given a nonzero SPSD matrix $\M \in \RB^{n\times n}$,  let $\gamma_i$ for $i=1, \ldots, n$ be the  eigenvalues of $\M$ and  $\x_i$ be the corresponding  eigenvectors. That is,
\begin{equation} \label{eqn:eigen}
\M \x_i = \gamma_i \x_i, \quad i=1, \ldots, n.
\end{equation}
It is well known that  the $\x_i$ can be assumed to be mutually orthonormal. Let $\Gam = \diag(\gamma_1, \ldots, \gamma_n)$ and $\X=[\x_1, \ldots, \x_n]$  such that $\X^T \X = \I_n$.
We write \eqref{eqn:eigen} in matrix form as
\[
\M \X = \X \Gam.
\]
This gives rise to an \emph{eigenvalue decomposition} (EVD)  of $\M$:
\[
\M = \X \Gam \X^T.
\]
%in which  the $\gamma_i$ are usually arranged in descent  order, i.e., $\gamma_1\geq  \cdots \geq \gamma_n$.
Since the $\gamma_i$ are  nonnegative, this decomposition is also called a \emph{spectral decomposition} of the SPSD matrix $\M$.

Note that the above EVD always exists when $\M$ is symmetric but not PSD. However, the current eigenvalues $\gamma_i$
are not necessarily nonnegative. Let $\hat{\Gam} = \diag(|\gamma_1|, \ldots, |\gamma_n|) $ and ${\Y}=[\y_1, \ldots, \y_n]$ with $\y_i=\sgn(\gamma_i) \x_i$ where $\sgn(0)=1$.  Then the decomposition is reformulated as
\[
\M = \X \hat{\Gam} \Y,
\]
where $\X ^T  \X = \I_n$, $\Y^T \Y =\I_n$, and $\hat{\Gam}$ is a nonnegative diagonal matrix. This new formulation
defines a singular value decomposition (SVD) of the symmetric matrix $\M$.

Naturally, a question emerges:  does an SVD exist for an arbitrary matrix?   Let $\A \in \RB^{m\times n}$ of rank $r$ where $r \leq \min\{m, n\}$. Without loss of generality,  we  assume  $m\geq n$  for ease of exposition,  because we can consider $\A^T$ when $m<n$.

Consider that $\A \A^T$ is SPSD, so it has the spectral decomposition, which
is defined  as
\[
\A \A^T = \U \Lam \U^T,
\]
where $\Lam = \diag(\lambda_1, \ldots, \lambda_m)$ and $\U^T \U = \I_m$. Since $\rk(\A \A^T) =\rk(\A)=r$, $\A \A^T$ has and only has $r$ positive eigenvalues and the corresponding
eigenvectors can form a column orthonormal matrix.

Assume that  $\Lam_r =\diag(\lambda_1,  \lambda_2, \ldots, \lambda_r)$  and $\U_r = [\u_1, \u_2, \ldots, \u_r ]$   where $\lambda_1\geq \lambda_2  \geq  \cdots \geq  \lambda_r$ are the positive eigenvalues of $\A \A^T$
and  $\U_r$ is the  $m\times r$ matrix of the corresponding eigenvectors such that $\U_r^T \U_r = \I_r$.
Thus, it follows from the spectral decomposition that
\[
\U_r^T \A \A^T \U_r =  \Lam_r
\]
and $\U_{-r}^T \A \A^T \U_{-r} =\0$ where
$\U_{-r}$ consists of the last $m{-}r$ columns of $\U$. Thus, we have $\A^T \U_{-r}=\0$.
Let $\V_r =[\v_1, \ldots, \v_r] \triangleq \A^T  \U_r  \Lam_r^{-1/2}$. Then it satisfies $\V_r^T \V_r = \I_r$.
Note that
\[ \A^T \U (\Lam_r^{-1/2} \oplus \I_{m-r}) = [\V_r, \A^T \U_{-r}] = [\V_r, \0],
\]
which implies that $\A^T = [\V_r, \0] (\Lam_r^{\frac{1}{2}} \oplus \I_{m-r}) \U^T = \V_r \Lam_r^{\frac{1}{2}} \U_r^T$.
Hence,
\begin{equation} \label{eqn:csvd}
\A = \U_r \Si_r \V_r^T,
\end{equation}
where $\Si_r = \diag(\sigma_1, \sigma_2, \ldots, \sigma_r)$ with $\sigma_i = \lambda_i^{1/2}$ for $i=1, \ldots, r$. Clearly, $\sigma_1\geq \sigma_2 \geq \cdots \geq \sigma_r >0$.

%We now have
%\[
%\A \V_r = \U_r \Si_r,
%\]
%
%Let $\V_{- r}$ ($n\times (n{-}r)$) be the orthonormal complement of $\V_r$ such that $\V=[\V_r, \V_{-r}]$ is orthonormal,
%i.e., $\V^T \V = \V \V^T =\I_n$. Recall that ${\mathrm{range}}(\V_r) \subseteqq   {\mathrm{range}} (\A^T)$ and $\mathrm{rank}(\V_r)= \mathrm{rank}(\A^T)=r$.  Consequently,    ${\mathrm{range}}(\V_r) =   {\mathrm{range}} (\A^T)$.  This implies that $\A \V_{-r} =0$ due to $\V_{-r}^T \V=0$. Thus,   we have
%\[
%\A = \A \V \V^T = \A \V_r \V_r^T + \A \V_{-r} \V_{-r}^T =  \U_r \Si_{r} \V_{r}^T.
%\]
We refer to \eqref{eqn:csvd} as the \emph{condensed SVD} of $\A$, where  $\sigma_i$'s are called the singular values, the columns $\u_i$ of $\U_r$ and the columns $\v_i$ of $\V_r$  are called respectively  the left and right singular vectors of $\A$.

Recall that we always assume that $\sigma_1\geq \sigma_2\geq \cdots \geq \sigma_r>0$.
%Typically,    $\A = \U_r \Si_r \V_r^T$ is called  the condensed singular value decomposition of $\A$.
Let $\Si_n=\diag(\sigma_1, \ldots, \sigma_r, 0, \ldots, 0)$ be the $n \times n$ diagonal matrix, and $\U_n$ be an $m\times n$ column-orthonormal matrix consisting of $\U_r$ in the first $m\times r$ block. In this case, we can equivalently write the condensed SVD of $\A$ as
\begin{equation} \label{eqn:svdthin}
\A = \U_n \Si_n \V^T,
\end{equation}
which is called a \emph{thin (or reduced) SVD} of $\A$.  Furthermore, we extend $\U_n$ to a square orthonormal matrix (denoted $\U$), and $\Si_n$ to an $m\times n$ matrix $\Si$ by adding $m-n$ rows of zeros below. Then SVD can be also expressed as
\begin{equation} \label{eqn:svdfull}
\A = \U \Si \V^T,
\end{equation}
which is called a \emph{full SVD} of $\A$.
%Figure~\ref{fig:svd} depicts these three forms of SVD. 

%\begin{figure}[!ht]
%\subfigtopskip = 0pt
%\begin{center}
%%\centering
%{\includegraphics[width=100mm,height=50mm]{full_new.png}} \\
%{\includegraphics[width=100mm,height=50mm]{thin_new.png}}\\
%{\includegraphics[width=100mm,height=50mm]{condensed_new.png}}
%\end{center}
%  \caption{The SVD.}
%\label{fig:svd}
%\end{figure}

As we have seen, these three expressions are mutually equivalent. We will sometimes use $\A= \U \Si \V^T$  for the thin SVD for notational simplicity. In a thin SVD version, let us always keep it in mind that $\Si$ is square and $\U$ or $\V$ is column orthonormal.
We now present the formal formation  of SVD of an arbitrary $\A \in \RB^{m\times n}$
in which $m\geq n$ is not necessarily required.

\begin{theorem} \label{thm:svd} Given an arbitrary $\A \in \RB^{m\times n}$, its full SVD defined in \eqref{eqn:svdfull} always exists. Furthermore, the singular values $\sigma_i$ are uniquely determined.
\end{theorem}

Based on the spectral decomposition of $\A \A^T$, we have previously shown the existence proof of the SVD theorem.
Here we  present a constructive proof, which has been widely given in the literature.

\begin{proof} If $\A$ is zero, the result is trivial. Thus, let $\A$ be a nonzero matrix. Define $\sigma_1\triangleq \max_{\|\x\|_2 =1} \; \|\A \x\|_2$, which exists because  $\x \mapsto \|\A \x\|_2$ is continuous and the set $\{\x \in \RB^n: \|\x\|_2=1 \}$ is compact.  Moreover, $\sigma_1>0$.  Let $\v_1 \in \RB^n$ be such a vector that $\sigma_1= \|\A \v_1\|_2$.
Define $\u_1= \A \v_1/\sigma_1$, which satisfies $\|\u_1\|_2 =1$.

We extend $\u_1$ and $\v_1$ to orthonormal matrices $\U = [\u_1, \U_{- 1}]$ and $\V = [\v_1,  \V_{- 1}]$, respectively.  Then
\[
\U^T \A \V = \begin{bmatrix} \sigma_1 & \u_1^T \A \V_{-1} \\ \0 & \U_{-1}^T \A  \V_{-1}  \end{bmatrix} \triangleq \B,
\]
where we use the fact $\U_{-1}^T \A \v_1 = \sigma_1 \U_{-1}^T \u_1=\0$.
Note that
\[ \max_{\|\x\|_2=1} \|\B \x \|_2^2 = \max_{\|\x\|_2=1} \|\U^T \A \V \x \|_2^2 = \max_{\|\x\|_2=1} \|\A \x\|_2^2 = \sigma_1^2.
\]
However,
\[
\frac{1}{\sigma_1^2 + \z^T \z} \left\| \B \begin{bmatrix} \sigma_1 \\ \z \end{bmatrix} \right\|_2^2 \geq \sigma_1^2 + \z^T \z,
\]
where $\z =  \V_{-1}^T \A^T \u_1 $. This implies that $\z$ must be zero.

The proof is completed by induction. In particular, assume  $(m-1) \times (n-1)$ matrix $\U_{-1}^T \A \V_{-1}$ has a full SVD $\U_{-1}^T \A \V_{-1} = \tilde{\U} \tilde{\Si} \tilde{\V}^T$. Then $\A$ has a full SVD:
\begin{align*}
\A & = [\u_1, \U_{-1}] \begin{bmatrix} 1 & \0 \\ \0 & \tilde{\U} \end{bmatrix}  \begin{bmatrix} \sigma_1 & \0 \\ \0 & \tilde{\Si} \end{bmatrix}
\begin{bmatrix} 1 & \0 \\ \0 & \tilde{\V}^T \end{bmatrix}  \begin{bmatrix} \v_1^T \\  \V_{-1}^T \end{bmatrix}\\
& =  [\u_1,  \U_{-1} \tilde{\U}]
\begin{bmatrix} \sigma_1 & \0 \\ \0 & \tilde{\Si} \end{bmatrix} \begin{bmatrix} \v_1^T \\  (\V_{-1} \tilde{\V})^T \end{bmatrix},
\end{align*}
because the matrices $ [\u_1,  \U_{-1} \tilde{\U}]$ and $ [\v_1,  \V_{-1} \tilde{\V}]$ are orthonormal.
\end{proof}

%We have also proved  the existence of SVD.
As for the uniqueness of the  singular values is due to that the $\sigma_i^2$  are  eigenvalues  of $\A \A^T$ which are unique. Unfortunately, the left and right singular matrices $\U_r$ and $\V_r$ are not unique.  However, we have the following result.

\begin{theorem} \label{thm:svdvector}  Let  $\A= \U_r \Si_r \V_r^T$  be a given condensed SVD of $\A$. Assume there are  $\rho$ distinct values among the nonzero singular values $\sigma_1, \ldots, \sigma_r$,  with respective multiplicities $r_i$ (satisfying $\sum_{i=1}^{\rho} r_i =r$).   Then $\A= \tilde{\U}_r \Si_r \tilde{\V}_r^T$ is a  condensed SVD if and only if
\[
 \tilde{\U}_r = \U_r (\Q_1 \oplus \Q_2 \oplus \ldots \oplus \Q_{\rho}) \; \mbox{ and } \;  \tilde{\V}_r = \V_r (\Q_1 \oplus \Q_2 \oplus \ldots \oplus \Q_{\rho}),
 \]
where $\Q_i$ is an arbitrary  $r_i \times r_i$ orthonormal matrix.

Furthermore, if all the nonzero singular values are distinct,
then the $\Q_i$ are either 1 or $-1$. In other words, the left and right singular vectors are uniquely determined up to signs.
\end{theorem}
\begin{proof} Let $\delta_1>  \delta_2>  \ldots >  \delta_{\rho}$ be the $\rho$ distinct values among the $\sigma_1, \ldots, \sigma_r$.
This implies that
\begin{equation} \label{eqn:bSi}
\Si_r = \delta_1 \I_{r_1} \oplus \delta_2 \I_{r_2} \oplus \ldots \oplus  \delta_{\rho}  \I_{r_{\rho}}.
\end{equation}
The sufficiency follows from the fact that
\[
(\Q_1 \oplus  \ldots \oplus \Q_{\rho}) (\delta_1 \I_{r_1} \oplus  \ldots \oplus  \delta_{\rho}  \I_{r_{\rho}})  (\Q_1^T \oplus \ldots \oplus \Q_{\rho}^T) = \Si_r.
\]

We now prove the necessary condition.
Consider that $\mathrm{range}(\U_r) = \mathrm{range}(\A) = \mathrm{range}(\tilde{\U}_r)$ and $\mathrm{range}(\V_r) = \mathrm{range}(\A^T) = \mathrm{range}(\tilde{\V}_r)$.  Thus,  we have
\[
\tilde{\U}_r = \U_r \S \; \mbox{ and }  \; \tilde{\V}_r = \V_r \T,
\]
where $\S$ and $\T$ are some $r\times r$ orthonormal matrices.  Hence,
$\Si_r = \S \Si_r \T^T$, or equivalently,  $\Si_r \T = \S \Si_r$. As in \eqref{eqn:bSi} for $\Si$, partition $\S$ and $\T$ into
\[\S=\begin{bmatrix} \S_{11} &  \ldots &  \S_{1 \rho} \\  \vdots &  \ddots &  \vdots \\  \S_{\rho 1} &  \ldots & \S_{\rho \rho}
\end{bmatrix} \; \mbox{ and } \; \T =\begin{bmatrix} \T_{11} &  \ldots &  \T_{1 \rho} \\  \vdots &  \ddots &  \vdots \\  \T_{\rho 1} &  \ldots & \T_{\rho \rho}
\end{bmatrix},
\]
where $\S_{i j} $ and $\T_{ij}$ are $r_i \times r_j$. It follows from $\Si_r  \T = \S \Si_r$ that  $\delta_i \T_{ii} = \delta_i \S_{ii}$ for $i=1, \ldots, \rho$ and $\delta_i \T_{ij} = \delta_j \S_{ij}$. As a result,  we obtain that $\S_{ii} = \T_{ii}$
for $i=1, \ldots, \rho$.
Since $\S$ and $\T$ are orthonormal,  we have
\[
\sum_{j=1}^\rho \S_{ij} \S_{ij}^T = \I_{r_i} =  \sum_{j=1}^\rho \T_{ij} \T_{ij}^T.
\]
Note that $\sum_{j=1}^\rho \T_{\rho j} \T_{\rho j}^T = \sum_{j=1}^\rho  \frac{\delta_{j}^2}{\delta_{\rho}^2}  \S_{\rho j} \S_{\rho j}^T $, which implies that
\begin{equation} \label{eqn:cond01}
\sum_{j< \rho}  \Big[1- \frac{\delta_{j}^2}{\delta_{\rho}^2} \Big]  \S_{\rho j} \S_{\rho j}^T = \0.
\end{equation}
Since $1- \frac{\delta_{j}^2}{\delta_{\rho}^2} < 0$ for $j < \rho$ and $\S_{\rho j} \S_{\rho j}^T$ is always PSD, we must have
$\S_{\rho j}=\0$ for all $j< \rho$, for otherwise,  if there were a $k<\rho$ such that $\S_{\rho k}\neq \0$, there would exist a nonzero $\x \in \RB^{r_\rho}$ such that $\x^T \S_{\rho k}  \S_{\rho k}^T \x>0$.  It would lead to
\[
\sum_{j< \rho}  \Big[1- \frac{\delta_{j}^2}{\delta_{\rho}^2} \Big] \x^T  \S_{\rho j} \S_{\rho j}^T \x < 0,
 \]
which conflicts with \eqref{eqn:cond01}.  Accordingly,  $\S_{\rho j}= \T_{\rho j}=\0$ for all $j< \rho$, and hence, $\S_{\rho \rho} \S_{\rho \rho}^T = \T_{\rho \rho} \T_{\rho \rho}^T = \I_{r_\rho}$. It also follows from the orthogonality  of $\S$  and of $\T$ that  for any $i< \rho$,
 \[
 \0= \sum_{j=1}^{\rho} \S_{i j} \S_{\rho j}^T = \S_{i \rho} \S_{\rho \rho}^T \; \mbox{ and } \;   \0= \sum_{j=1}^{\rho} \T_{i j} \T_{\rho j}^T = \T_{i \rho} \T_{\rho \rho}^T, \]
 which leads to $\S_{i \rho} = \T_{i \rho } =\0$ for $i< \rho$.

Similarly, consider the $\rho-1$, $\rho-2, \ldots, 2$  cases. We have $\S_{i j}= \T_{i j}=\0$ for  $i \neq j$,  $\S_{ii}=\T_{ii}$ and $\S_{i i} \S_{i i}^T = \T_{i  i} \T_{i i}^T = \I_{r_{i}}$ for $i \in [\rho]$. As a result,  setting $\Q_i= \S_{ii}$ completes the proof.
\end{proof}

We now extend the result in  Theorem~\ref{thm:svdvector} to the full SVD and  thin SVD of $\A$.  The following corollary is immediately obtained.

\begin{corollary}  \label{cor:svd-uni} Let $\A= \U \Si \V^T$ be a given full SVD of $\A \in \RB^{m\times n}$. Then $\A = \tilde{\U} \Si \tilde{\V}^T$ is a full SVD if and only if $\tilde{\U} = \U \Q$ and $\tilde{\V} = \V \PP$ where $\Q= \Q_1 \oplus \cdots \oplus \Q_{\rho} \oplus \Q_0$ and $\PP=\Q_1 \oplus \cdots \oplus \Q_{\rho} \oplus \PP_0$. Here $\Q_1, \ldots, \Q_{\rho}$ are defined as in Theorem~\ref{thm:svdvector}, and $\Q_0 \in \RB^{(m-r)\times(m-r)}$ and $\PP_0 \in \RB^{(n-r)\times (n-r)}$ are any orthonormal matrices.
Obviously, $\Q \Si \PP^T = \Si$ and $\Q^T \Si \PP = \Si$ hold.

Assume $m\geq n$ and  $\A= \U \Si \V^T$ is  a given thin SVD of $\A \in \RB^{m\times n}$. Then $\A = \tilde{\U} \Si \tilde{\V}^T$ is a thin SVD if and only if $\tilde{\U} = \U \Q$ and $\tilde{\V} = \V \PP$ where $\Q= \Q_1 \oplus \cdots \oplus \Q_{\rho} \oplus \Q_0$ and $\PP=\Q_1 \oplus \cdots \oplus \Q_{\rho} \oplus \PP_0$.  Currently, $\Q_0 \in \RB^{(n-r)\times(n-r)}$  is any orthonormal matrix. Obviously,
$\Q \Si = \Si \Q$,  $\Si \PP^T = \PP^T \Si$, and $\Q \Si \PP^T = \Si$ hold.
\end{corollary}
%This corollary implies $\Q \Si \PP^T = \Si$.

Theorem~\ref{thm:svdvector} and Corollary~\ref{cor:svd-uni}  will be used in derivation of subdifferentials of unitarily invariant norms (see Chapter~\ref{ch:subdiff}).
When the matrix in question is SPSD, the  spectral decomposition and SVD are identical.  That is, $\U =\V$ in this case.
Moreover, the eigenvalues and singular values are identical.

The construction proof of Theorem~\ref{thm:svd} shows that
\begin{itemize}
\item[] $\sigma_1(\A) = \max\{\|\A \v\|_2: \v \in \RB^n, \|\v\|_2 = 1\}$, so there exists a unit vector $\v_1 \in \RB^n$
such that $\sigma_1(\A) = \|\A \v_1\|_2$;
\item[] $\sigma_2(\A) = \max\{\|\A \v\|_2: \v \in \RB^n, \|\v\|_2 = 1, \v^T \v_1=0\}$, so there exists a unit vector $\v_2 \in \RB^n$
such that $\v_2^T \v_1=0$ and $\sigma_2(\A) = \|\A \v_2 \|_2$;
\item[] $\vdots$
\item[] $\sigma_k(\A) = \max\{\|\A \v\|_2: \v \in \RB^n, \|\v\|_2 = 1, \v^T [\v_1, \ldots, \v_{k{-}1}]=\0\}$, so there exists a unit vector $\v_k \in \RB^n$ such that $\v_k^T [\v_1, \ldots, \v_{k{-}1}] = \0$
and $\sigma_k(\A) = \|\A \v_k \|_2$;
\item[] $\vdots$
\end{itemize}

The following theorem is  the generalization of the Courant-Fischer theorem for singular values.
\begin{theorem} \label{thm:CF-svd}  Given a matrix $\A \in \RB^{m\times n}$, let $\sigma_1\geq \sigma_2 \geq \cdots \geq \sigma_p$
be the singular values of $\A$ where $p = \min\{m, n\}$. For any $k \in [p]$, then
\begin{align*}
\sigma_k &= \min_{\v_1, \ldots, \v_{k-1} \in \RB^n} \max_{\begin{array}{c} \v \in \RB^n, \|\v\|_2 = 1\\ \v^T[\v_1, \ldots, \v_{k-1}] = \0 \end{array}} \| \A \v\|_2 \\
& = \max_{\v_1, \ldots, \v_{n-k} \in \RB^n} \min_{\begin{array}{c} \v \in \RB^n, \|\v\|_2=1 \\ \v^T[\v_1, \ldots, \v_{n-k}] = \0 \end{array}} \| \A \v\|_2.
\end{align*}
\end{theorem}

\section{Matrix Properties via SVD}
\label{sec: mpsvd}

In what follows,
we list some matrix properties which can be induced from SVD. These properties show that SVD is fundamental not only in matrix computation but also in matrix analysis.
\begin{proposition} \label{pro:012} Let $\A= \U \Si \V^T$ be a full SVD of $m\times n$ matrix $\A$, and $\A= \U_r \Si_r \V_r^T$ be a condensed SVD.  Let $p=\min\{m, n\}$. Then
\begin{enumerate} \item[{(1)}] The rank of $\A$ is equal to the number of the nonzero singular values  $\sigma_i$ of $\A$.
\item[(2)] $\|\A\|_2 = \sigma_1$ is the spectral norm and $\| \A\|_F =  \sqrt{\sum_{i,j} a_{ij}^2} = \sqrt{\sum_{i=1}^{p}  \sigma_i^2}$ is the Frobenius norm.
\item[(3)]  $\rg(\A) = \rg(\A \A^T)= \rg(\U_r)=\mathrm{span}(\u_1, \ldots, \u_r)$ and $\mathrm{null}(\A) = \rg(\V_{-r}) =\mathrm{span}(\v_{r+1}, \ldots, \v_n)$.
\item[(4)]
$\rg(\A^T) = \rg(\A^T \A) = \rg(\V_r) =\mathrm{span}(\v_1, \ldots, \v_r)$ and $\mathrm{null}(\A^T) = \rg(\U_{-r}) =\mathrm{span}(\u_{r+1}, \ldots, \u_m)$.
\item[(5)] The eigenvalues of $\A^T \A$ are $\sigma_i^2$ for $i=1, \ldots, r$ and $n-r$ zeros. The right singular vectors $\v_i$ are  the corresponding orthonormal eigenvectors.
\item[(6)] The eigenvalues of $\A \A^T$ are $\sigma_i^2$ for $i=1, \ldots, r$ and $m-r$ zeros. The left singular vectors $\u_i$ are the corresponding orthonormal eigenvectors.
\item[(7)] Let $\B = \U_{B} \Si_{B} \V_{B}$ be the condensed SVD of $\B$. Then $\A \oplus \B = (\U \oplus \U_B) (\Si \oplus \Si_B) (\V^T \oplus \V_{B}^T)$ is the condensed SVD of $\A {\oplus}  \B$, and $\A \otimes \B = (\U \otimes \U_B) (\Si \otimes \Si_B) (\V^T \otimes \V_{B}^T)$ is the condensed SVD of $\A {\otimes}  \B$.
\item[(8)] If $\A$ is square and invertible, then $\A^{-1} = \V \Si^{-1} \U^T$ and $|\det(\A)| =\prod_{i=1}^n \sigma_i(\A)$.
%\item[(6)] Let $\H = \begin{bmatrix}  \0 & \A^T \\ \A & \0 \end{bmatrix}$. Then $\H$ has $2 r$ nonzero eigenvalues, which  are $\pm  %\sigma_i$, with corresponding orthnormal eigenvectors $ \frac{1}{\sqrt{2}} \begin{bmatrix} \v_i \\  \pm \u_i \end{bmatrix}$, $i=1, \ldots, r%$.
\end{enumerate}
\end{proposition}

\begin{theorem}  \label{thm:h-a} Given a matrix $\A \in \RB^{m {\times} n}$, let $\H = \begin{bmatrix}  \0 & \A^T \\ \A & \0 \end{bmatrix}$.
If $\A= \U_r \Si_r \V_r^T$ be the condensed SVD, then $\H$ has $2 r$ nonzero eigenvalues, which  are $\pm \sigma_i$, with the corresponding orthonormal eigenvectors $ \frac{1}{\sqrt{2}} \begin{bmatrix} \v_i \\  \pm \u_i \end{bmatrix}$, $i=1, \ldots, r$.

Conversely, if   $\gamma_i$ is the eigenvalue of $\H$, with the corresponding  eigenvector $\z_i=\begin{bmatrix} \z_i^{(1)} \\ \z_i^{(2)} \end{bmatrix}$ where
$\z_i^{(1)} \in \RB^n$ and $\z_i^{(2)} \in \RB^m$, then  $-\gamma_i$ is the eigenvalue of $\H$, with the corresponding  eigenvector $\z_i=\begin{bmatrix} \z_i^{(1)} \\ - \z_i^{(2)} \end{bmatrix}$. Furthermore,  let the $\sigma_i$ denote the $r$ positive values among the $\pm \gamma_i$,  and  $\frac{1}{\sqrt{2}}\begin{bmatrix} \v_i \\  \u_i \end{bmatrix}$
denote the corresponding orthonormal eigenvectors. Then $\A = \U_r \Si_r \V_r^{T}$,  where $\U_r=[\u_1, \ldots, \u_r]$, $\V_r=[\v_1, \ldots, \v_r]$, and $\Si_r=\diag(\sigma_1, \ldots, \sigma_r)$, is  a condensed SVD of $\A$.  \end{theorem}

\begin{proof} The first part is directly obtained from the fact that
\begin{align*}
\H & = \begin{bmatrix}  \0 & \A^T \\ \A & \0 \end{bmatrix} = \begin{bmatrix}  \0 & \V_r \Si_r \U_r^T \\ \U_r \Si_r \V_r^T & \0 \end{bmatrix}
 \\
&= \frac{1}{2} \begin{bmatrix}  \V_r & \V_r  \\ \U_r  &  - \U_r  \end{bmatrix}  \begin{bmatrix}  \Si_r &  \0  \\  \0 & - \Si_r   \end{bmatrix}   \begin{bmatrix}  \V_r^T & \U_r^T \\ \V_r^T  & - \U_r^T \end{bmatrix}. \end{align*}

Conversely, consider that
\begin{align*}
\begin{bmatrix}  \0 & \A^T \\ \A & \0 \end{bmatrix}  \begin{bmatrix} \z_i^{(1)} \\ - \z_i^{(2)} \end{bmatrix} &  = \begin{bmatrix} -\A^T \z_i^{(2)} \\ \A \z_i^{(1)} \end{bmatrix}
= \begin{bmatrix} -\gamma_i \z_i^{(1)} \\ \gamma_i \z_i^{(2)} \end{bmatrix}  = -\gamma_i  \begin{bmatrix} \z_i^{(1)} \\ - \z_i^{(2)} \end{bmatrix},
\end{align*}
which shows that  $-\gamma_i$ is the eigenvalue of $\H$, with the corresponding eigenvector $\begin{bmatrix} \z_i^{(1)} \\ - \z_i^{(2)} \end{bmatrix} $.   Now using the notation of $\Si_r$, $\U_r$, and $\V_r$, we have the EVD of $\H$:
\[
\H = \begin{bmatrix}  \0 & \A^T \\ \A & \0 \end{bmatrix} =  \frac{1}{2} \begin{bmatrix}  \V_r & \V_r  \\ \U_r  &  - \U_r  \end{bmatrix}  \begin{bmatrix}  \Si_r &  \0  \\  \0 & - \Si_r   \end{bmatrix}   \begin{bmatrix}  \V_r^T & \U_r^T \\ \V_r^T  & - \U_r^T \end{bmatrix}.
\]
It  also follows from the orthogonality of the eigenvectors that  $ \U_r^T \U_r  + \V_r^T \V_r =  2\I_r$  and $ \U_r^T \U_r  - \V_r^T \V_r=  \0$.  This implies that $ \U_r^T  \U_r  =  \V_r^T  \V_r =  \I_r$. Thus, $\A = \U_r \Si_r \V_r^T$ is a condensed SVD of $\A$.
\end{proof}

Theorem~\ref{thm:h-a} establishes an  interesting  connection of the SVD of a general matrix with the EVD of a symmetric matrix.
This  provides an  approach to handling an SVD problem of an arbitrary matrix. That is,   one transforms the SVD problem  into an EVD problem of an associated symmetric matrix.
The theorem
also gives an alternative proof for the SVD theory.

The following theorem shows that the \emph{Polar Decomposition}
of a matrix can be induced from its SVD.  Note that SVD can be also derived from the Polar decomposition.
%In fact, the original Polar decomposition on a nonsingular square matrix has been established by L. Autonne in 1902.
Here we do not give the detail of this derivation.

\begin{theorem}[Polar Decomposition]  \label{thm:polarD} Let $\A \in \RB^{m\times n}$ be a given matrix where $m\geq n$. Then its polar decomposition exists; that is,  there are a column orthonormal matrix $\Q$ and  a unique   SPSD matrix $\S$ such that $\A= \Q \S$.  Furthermore, if $\A$ is full column rank, then $\Q$ is unique.
%Moreover, $\S$ is uniques.
\end{theorem}

\begin{proof} Let $\A = \U \Si \V^T$ be a thin SVD of $\A$. Then
\[
\A = \U \V^T \V \Si \V^T \triangleq \Q \S,
\]
where $\Q \triangleq \U \V^T$ is column orthonormal and $\S \triangleq \V \Si \V^T$ is SPSD.

Assume that $\A$ has two Polar decompositions: $\A = \Q_ 1 \S_1$ and $\Q_2 \S_2$. Make the full SVDs (spectral decomposition) of $\S_1$ and $\S_2$ as $\S_1 = \V_1 \Si_1 \V_1^T$ and $\S_2 = \V_2 \Si_2 \V_2^T$, respectively.  Then $\A =  (\Q_1 \V_1) \Si_1 \V_1^T$ and $\A= (\Q_2 \V_2) \Si_2 \V_2^T$ be two thin SVDs of $\A$. This implies that $\Si_1= \Si_2 \triangleq \Si$. Moreover, it follows from Corollary~\ref{cor:svd-uni} that $\V_2 = \V_1 \PP_1$ and $\Q_2 \V_2 = \Q_1 \V_1 \PP_2$ where $\PP_1$ and $\PP_2$ are orthonormal matrices such that $ \Si  \PP_1^T = \PP_1^T \Si$. Thus, $\S_2 = \V_2 \Si \V_2^T = \V_1 \PP_1 \Si \PP_1^T \V_1^T = \V_1 \Si \V_1^T = \S_1$.

If $\A$ is full column rank, then $\S$ is invertible. Hence,  $\Q_1 = \Q_2$.
\end{proof}

As we see from the proof, $\S= \V \Si \V^T = (\A^T \A)^{1/2}$; that is, $\S$ is identical to the square root of the matrix $\A^T \A$.

\section{Matrix Concepts via SVD}
\label{sec:concept}

All matrices have SVD, so  SVD plays a central role in matrix analysis and computation.
As we have seen in the previous section,   many matrix concepts and properties can be induced from SVD.
Here we present other several matrix notions, which are used in modern matrix computations.  

\begin{definition} Assume  $\A \in \RB^{m\times n}$ and $\B \in \RB^{m\times n}$ are of rank $k$ and rank $l$, respectively, and $l\geq k$. Let  $\A = \U_{A, k} \Si_{A, k} \V_{A, k}^T$ and $\B = \U_{B, l} \Si_{B, l} \V_{B, l}^T$ be the condensed SVDs of $\A$ and $\B$. The cosines of the canonical angles between $\A$ and $\B$ are defined as 
\[
\cos \theta_i(\A, \B) = \sigma_i(\U_{A, k}^T \U_{B, l}), \; i=1, \ldots, k.
\] 
%Thus, the principal angle is characterized by $\sigma_k(\U_{A, k}^T \U_{B, l})$.  
Consider that
\[
 \sigma^2(\U_{A, k}^T \U_{B, l}) = \lambda(\U_{A, k}^T \U_{B, l} \U_{B, l}^T \U_{A, k})
 \]
 and $\U_{A, k}^T \U_{B, l} \U_{B, l}^T \U_{A, k} + \U_{A, k}^T \U_{B, -l} \U_{B, -l}^T \U_{A, k} =\I_k$, where $\U_{\B, -l} \in \RB^{m\times {n-l}}$ is an orthonormal complement of $\U_{B, l}$.  Thus, we have that
 \[
 \lambda(\U_{A, k}^T \U_{B, l} \U_{B, l}^T \U_{A, k} ) =1 - \lambda(\U_{A, k}^T \U_{B, -l} \U_{B, -l}^T \U_{A, k} ).
 \]
 In other words, $\sigma^2(\U_{A, k}^T \U_{B, l}) = 1- \sigma^2(\U_{A, k}^T \U_{B, -l}  )$. Hence, 
 \[
 \sin \theta_i (\A, \B) = \sigma_{k+1-i}(\U_{A, k}^T \U_{B, -l}  ), \; i=1, \ldots, k.
  \]
Note that $ \sigma_1(\U_{A, k}^T \U_{B, -l}) = \| \U_{A, k}^T \U_{B, -l}\|_2$, which is also cased the distance between two subspaces spanned by $\U_{A, k}$ and $\U_{B, l}$. 
\end{definition}

\begin{definition}  \label{def:grank} Given a nonzero matrix  $\A \in \RB^{m\times n}$, let $\sigma_1\geq \cdots \geq  \sigma_p$  where $p=\min\{m, n\}$. The stable rank of $\A$ is defined as $\sum_{i=1}^p \frac{\sigma_i^2}{\sigma_1^2}$, and the nuclear rank is defined as $\sum_{i=1}^p \frac{\sigma_i}{\sigma_1}$.
\end{definition}

Clearly, $\sum_{i=1}^p \frac{\sigma_i^2}{\sigma_1^2} \leq \sum_{i=1}^p \frac{\sigma_i}{\sigma_1} \leq \rk(\A)$. The concepts  have been recently proposed for describing error bounds of  matrix multiplication approximation~\citep{magen2011low,CohenNelsonWoodrull,kyrillidis2014approximate}.

\begin{definition}[Statistical Leverage Score] \label{def:sls} Given an $m\times n$ matrix $\A$ with $m>n$, let $\A$ have a thin SVD $\A = \U \Si \V^T$, and let
$\u^{(i)}$ be the $i$th row of $\U$. Then the statistical leverage scores of the rows of $\A$ are defined as
\[
l_i = \|\u^{(i)}\|_2^2 \; \mbox{ for } \; i=1, \ldots, m.
\]
The coherence of the rows of $\A$ is defined as
\[
\gamma \triangleq \max_{i} \; l_i.
\]
The $(i,j)$-cross leverage scores are defined as
\[
c_{ij} = (\u^{(i)})^T \u^{(j)}.
\]
\end{definition}

The  statistical leverage~\citep{HoaglinWelsch} measures the extent to which the singular vectors of a matrix are
correlated with the standard basis. Recently,  it has found usefulness  in large-scale data analysis
and in the analysis of randomized matrix algorithms \citep{DrineasCUR:2008,mahoney2009matrix,ma2014statistical}. A related notion is that of matrix coherence, which has been of interest in  matrix completion and Nystr\"{o}m-based low rank matrix approximation \citep{CandesRecht:2009,TalwalkarRostamizadeh,WangZhangJMLR:2013,nelson2013osnap}.

\section{Generalized Singular Value Decomposition}
\label{sec:gsvd}

This section studies simultaneous SVD of two given matrices $\A$ and $\B$.
This leads us to  a  generalized SVD (GSVD) problem.

\begin{theorem}[GSVD]  \label{thm:gsvdVon} Suppose two matrices $\A \in \RB^{m\times p}$ and $\B \in \RB^{n \times p}$ with $n\geq p$ are given.  Let  $q=\min\{m, p\}$.
Then there exist two orthonormal matrices $\U_A \in \RB^{m \times m}$ and  $\U_B \in \RB^{n \times n}$, and an invertible matrix $\X \in \RB^{p\times p}$ such that
\[
\U_A^T \A \X = \diag(\alpha_1, \ldots, \alpha_q) \; \mbox{ and } \; \U_B^T \B \X = \diag(\beta_1, \ldots, \beta_p),
\]
where  $ \alpha_1 \geq  \cdots \alpha_q \geq 0$, and $0 \leq \beta_1 \leq \cdots \beta_p$.
\end{theorem}
The  GSVD theorem was originally proposed by \cite{LoanGSVD:1976}, in which  $n \geq p$ (or $m\geq p$) is required. Later on,
 \cite{PaigeSIAM:1981} developed a more general formulation for   GSVD in which  matrix pencil $\A$ and $\B$ are  required only to have the same number of columns.  \cite{PaigeSIAM:1981}  also studied a GSVD of submatrices of a column orthonormal matrix.
That is a so-called CS decomposition~\citep{GolubT:1999} given as follows.

\begin{theorem}[The CS Decomposition] \label{thm:cs-decomp} Let $\Q \in \RB^{(m{+}n) \times p}$ be a column orthonormal matrix. Partition it as $\Q^T=[\Q_1^T,  \Q_2^T]$ where $\Q_1$  and $\Q_2$ are $m  \times p$ and $n  \times p$. Then there exist orthonormal matrices $\U_1 \in \RB^{m{\times} m}$, $\U_2\in \RB^{n \times n}$, and $\V_1 \in ^{p\times p}$ such that
\[
\U_1^T \Q_1 \V_1 = \C  \; \mbox{ and } \; \U_2^T \Q_2 \V_1 = \S,
\]
where
\[
\C =\bordermatrix{& r & s  & p{-} r {-} s  \cr r & \I_r   &  \0 & \0 \cr  s & \0 &  \C_1 & \0  \cr m{-}r{-}s & \0 & \0  & \0 \cr },  \]
\[
\S =\bordermatrix{ & r  & s & p{-}r{-} s  \cr  n{+}r{-} p & \0   &  \0 & \0 \cr  s & \0 &  \S_1 & \0  \cr p{-}r{-}s & \0 & \0  & \I_{p{-}r{-}s} \cr },
 \]
 $\C_1 = \diag(\alpha_1,   \ldots, \alpha_s)$ and $\S_1= \diag( \sqrt{1-\alpha_1^2},    \ldots,  \sqrt{1- \alpha_s^2})$,  and $1> \alpha_1 \geq \alpha_2 \geq \cdots \alpha_s  > 0$.
\end{theorem}
 \begin{proof} Since $\Q_1^T \Q_1 + \Q_2^T \Q_2 = \Q^T \Q = \I_p$, the largest eigenvalue of $\Q_1^T \Q_1$ (reps. $\Q_2^T \Q_2$) is at most 1. This implies $\|\Q_1\|_2 = \sigma_1(\Q_1) \leq 1$ (resp. $\|\Q_2\|_2 \leq 1$).  Let $q =\min\{m, p\}$. Make  a full SVD of $\Q_1$ as
 \[
 \Q_1 = \U_1\C \V_1^T,
 \]
 where $\C= \diag(c_1, \ldots, c_{q})$ is an $m{\times} p$ diagonal matrix. Assume
 \[
 1= c_1 =\cdots = c_r > c_{r+1} \geq \cdots \geq c_{r+s}> c_{r+s+1} = \cdots c_{p} = 0.
 \]
 Let $\D= \diag(c_{r+1}, \ldots, c_{r+s}) \oplus  \0$, which is  $(m-r) \times (p-r)$,  and
 \[
 \Q_2 \V_1 = [\underbrace{\W_1}_{r},  \underbrace{\W_2}_{p-r} ].
 \]
 Then
 \[
 \begin{bmatrix} \U_1 & \0 \\ \0 & \I_{n} \end{bmatrix}^T \begin{bmatrix} \Q_1 \\ \Q_2 \end{bmatrix} \V_1 = \begin{bmatrix} \I_r & \0 \\
 \0 & \D \\ \W_1 & \W_2 \end{bmatrix}
 \]
 is column orthonormal. This implies that $\W_1 =\0$ and
 \[
 \W_2^T \W_2 = \I_{p-r} - \D^T \D = \diag(1- c_{r+1}^2, \ldots, 1- c_p^2)
 \]
 is nonsingular. Define $s_i= \sqrt{1-c_i^2}$ for $i\in [p]$. Then
 \[
 \Z \triangleq \W_2 \diag(1/s_{r+1}, \ldots, 1/s_p)
 \]
 is column orthonormal. We now extend $\Z$ to an $n \times n$ orthonormal matrix $\U_2$, the last $p-r$ columns of which constitute $\Z$.  When setting $\alpha_1=c_{r+1}, \cdots, \alpha_s = c_{r+s}$, we have
 \[
 \U_2^T \Q_1 \V_1 =  \S.
 \]
 Thus, the theorem follows.
 \end{proof}

\paragraph{Remarks}  It is  worth pointing out that $\Q_1 = \U_2 \S \V_1^T$ is not certainly a full SVD of $\Q_1$,  because some of the nonzero elements of $\S$ might not  lie on  the principal diagonal.
However, if $n  \geq p$, then we can move the first $n -p$ rows of $\S$ to be  the last $n-p$ rows by pre-multiplying some permutation matrix $\PP$.  That  is, %So $\S_1$ and $\I_{q-r-s}$ lie on the  leading
\[
\PP^T \U_2^T \Q_1 \V_1 = \bordermatrix{ & r  & s & p{-}r{-} s  \cr  r  & \0   &  \0 & \0 \cr  s & \0 &  \S_1 & \0  \cr p{-}r{-}s & \0 & \0  & \I_{p{-}r{-}s} \cr  n{-}p  & \0 & \0 & \0 \cr }.
\]
This is the reason why the restriction $n \geq p$   is required in Theorem~\ref{thm:gsvdVon} ($\A$ and $\B$ correspond to $\Q_1$ and $\Q_2$, respectively).

The following theorem gives a more general  version of Theorem~\ref{thm:gsvdVon} as well as Theorem~\ref{thm:cs-decomp}.
Compared with Theorem~\ref{thm:gsvdVon},  $m\geq p$ or $n\geq p$ are no longer restricted. Compared with  Theorem~\ref{thm:cs-decomp}, the submatrices in question do not necessarily form a column orthonormal matrix.

\begin{theorem} \label{thm:gsvdPS} Suppose two matrices $\A \in \RB^{m\times p}$ and $\B \in \RB^{n \times p}$ are given. Let $\K^T \triangleq  [\A^T,  \B^T]$  with the rank $t$.  Then exist orthonormal matrices $\U_A \in \RB^{m\times m}$,  $\U_B \in \RB^{n \times n}$, $\W \in \RB^{t\times t}$, and  $\V \in \RB^{p \times p}$ such that
\[
\U_A^T \A \V = \Si_A [\underbrace{\W^T \R}_{t}, \;  \underbrace{\0}_{p-t}] \; \mbox{ and } \; \U_B^T \B \V = \Si_B [\underbrace{\W^T \R}_{t} , \; \underbrace{\0}_{p-t}],
\]
where $\R \in \RB^{t\times t}$ is a positive diagonal matrix with its diagonal elements equal to the nonzero of singular values of $\K$,
\begin{equation} \label{eqn:SiA}
\Si_A =\bordermatrix{& r & s  & t{-} r {-} s  \cr r & \I_r   &  \0 & \0 \cr  s & \0 &  \D_A & \0  \cr m{-}r{-}s & \0 & \0  & \0 \cr },
\end{equation}
\begin{equation} \label{eqn:SiB}
\Si_B  =\bordermatrix{ & r  & s & t{-}r{-} s  \cr  n{+}r{-} t & \0   &  \0 & \0 \cr  s & \0 &  \D_B  & \0  \cr t{-}r{-}s & \0 & \0  & \I_{t{-}r{-}s} \cr }.
 \end{equation}
Here $r$ and $s$ depend on the context,
\[
\D_A = \diag(\alpha_{r+1}, \ldots, \alpha_{r+s}) \; \mbox{ and } \; \D_B = \diag(\sqrt{1{-} \alpha^2_{r+1}}, \ldots, \sqrt{1{-} \alpha^2_{r+s}}),
\]
and $1> \alpha_{r+1} \geq \cdots \geq \alpha_{r+s}>0$.
\end{theorem}
Theorem~\ref{thm:gsvdPS} implies that
\[
\U_A^T \A \X = [\Si_A,  \0] \;  \mbox{ and } \; \U_B^T \B \X = [\Si_B, \0],
\]
where $\X \triangleq  \V  (\R^{-1} \W \oplus  \I_{p-t})$.  With the above remarks, Theorem~\ref{thm:gsvdVon} follows.
Thus, we now present the proof of
Theorem~\ref{thm:gsvdPS}.  %which was proposed by  \cite{PaigeSIAM:1981}.
\begin{proof}  %Let $\C^T =[\A^T, \B^T]$. Assume that  $\rk(\C) =k \leq \min\{n, m+p\}$.
Since $\rk(\K)=t$,  making  a full SVD of $\K$ yields %orthonormal matrices $\PP$ and $\Q$ such that
\[
\PP^T \K \V = \begin{bmatrix} \R & \0 \\ \0 & \0 \end{bmatrix},
\]
where $\PP \in \RB^{(m+n) \times (m+n)}$ and $\V \in \RB^{p\times p}$ are orthonormal matrices,   $\R$ is a $t \times t$ diagonal matrix with the diagonal elements as the nonzero singular values of $\K$. %This can be computed by  a complete orthogonal %decomposition~\citep{GolubT:1999}.
Partition $\PP$ as
\[
\PP = [\underbrace{\PP_1}_{t},  \underbrace{\PP_2}_{m+n-t}] = \begin{bmatrix}   \PP_{11} & \PP_{12} \\  \PP_{21} & \PP_{22}  \end{bmatrix} \; \mbox{ where } \; \PP_{11} \in \RB^{m\times t} \mbox{ and } \PP_{21} \in \RB^{n \times t}.
\]
Obviously, $\PP_1^T \PP_1 = \PP_{11}^T \PP_{11} + \PP_{21}^T \PP_{21} = \I_t$.
Moreover, we have
\[
\K \V  = [\PP_1 \R, \0].
\]
Applying Theorem~\ref{thm:cs-decomp} to $\PP_1$  yields that there exist orthonormal matrices $\U_A \in \RB^{m\times m}$, $\U_B \in \RB^{n\times n}$, and $\W \in \RB^{t\times t}$ such that
\[
\begin{bmatrix} \U_A^T & \0 \\ \0 & \U_B^T \end{bmatrix} \begin{bmatrix} \PP_{11} \\ \PP_{21} \end{bmatrix} \W = \begin{bmatrix}  \Si_A \\ \Si_B \end{bmatrix}
\]
where $\Si_A$ and $\Si_B$ are defined in \eqref{eqn:SiA} and \eqref{eqn:SiB}.
Hence,
\[
\begin{bmatrix} \U_A^T & \0 \\ \0 & \U_B^T \end{bmatrix} \begin{bmatrix} \A  \\  \B \end{bmatrix} \V  = \begin{bmatrix}  \Si_A \W^T \R & \0 \\  \Si_B  \W^T  \R & \0 \end{bmatrix}.
\]
That is, $\U_A^T \A \V = \Si_A[\W^T \R,  \0]$ and $\U_B^T \B \V = \Si_B [\W^T \R,  \0]$.
\end{proof}

In terms of Theorem~\ref{thm:gsvdVon}, if $\beta_i \neq 0$, then the column $\x_i$ of $\X$ satisfies
\[
\A^T \A \x_i = \lambda_i \B^T \B \x_i,
\]
where $\lambda_i = \frac{\alpha_i^2}{\beta_i^2}$. This implies GSVD can be used to solve  generalized eigenvalue problems.
Based on this observation,
\cite{HowlandSIAM:2003,ParkSIAM:2005} applied GSVD for solving Fisher linear discriminant analysis (FLDA) and generalized Fisher discriminant analysis~\citep{Baudat:2000,MikaNIPS:2000}.

Recall that the above GSVD procedure requires to implementing an SVD on the $(m{+}n)\times p$ matrix $\K$. The computational cost is
$O((m{+}n)  p * \min\{m{+}n, p\})$. Thus, when both $m{+}n$ and $p$ are very large, the GSVD is less efficient.
We now consider a special case in which $\B = \Z \A$ where $\Z \in \RB^{n \times m}$ is some given matrix.
We will see that it is no longer necessary to perform the SVD on $\K$.

\begin{theorem} \label{thm:gsvd-sp} Let $\A \in \RB^{m\times p}$ and $\B \in \RB^{n \times p}$ be two given matrices. Assume that $\B = \Z \A$ where $\Z \in \RB^{n \times m}$ is some matrix,    $\rk(\B)=s$, and $\rk(\A)=t$. Let  $\A= \U_t \Si_t \V_t^T$ be a condensed SVD of $\A$, and $\Y=\U_Y \Si_Y \V_Y^T $ be a full SVD of $\Y \triangleq \Z \U_t$. Then
%Then  there  exist a column orthonormal matrix $\U \in \RB^{r \times m}$, an orthonormal matrix  $\U_B \in \RB^{p\times p}$, and an %invertible $\X \in \RB^{n\times n}$ such that
\[
(\U_t \V_Y) ^T \A\V_t \Si_t^{-1} \V_Y=  \I_t   \; \mbox{ and } \; \U_Y^T \B \V_t \Si_t^{-1} \V_Y  =  \Si_Y.
\]
\end{theorem}
The proof is direct.
%\begin{proof}
Assume $\U_{t}$ and $\V_t$ are extended to orthonormal matrices $\U$ ($m\times m$) and $\V$ ($p \times p$).
Let
\[
\X = \V (\Si_t^{-1} \V_Y \oplus \I_{p-t} ).
\]
We now have that
\[
\A \X = \U \U^T \A \V  (\Si_t^{-1} \V_Y \oplus \I_{p-t} ) = [\U_t \V_Y, \0] = \U (\V_Y \oplus \I_{m-t}) (\I_t \oplus \0)
\]
and
\begin{align*}
\B \X  & = \Z \U  \U^T \A \V (\Si_t^{-1} \V_Y \oplus \I_{p-t} ) = \Z \U_t [\V_Y, \0] \\
 & = \U_Y \Si_Y \V_Y^T [ \V_Y, \0] = \U_Y [\Si_Y, \0].
\end{align*}
Thus,
\[
(\V_Y^T \oplus \I_{m-t}) \U^T \A \X =[\I_t \oplus \0]
\]
and
\[
\U_Y^T \B \X = [\Si_Y,  \0].
\]
%Accordingly, the proof completes.
%\end{proof}

In this special case, we only need to implement two SVDs on  two matrices with smaller sizes. The diagonal elements of $\Si_Y$  and the columns of $\V_t \Si_t^{-1} \V_Y$ are the generalized eigenvalues and eigenvectors  of the corresponding generalized eigenvalue problem.

\paragraph{Remarks}  Assume that $\A \in \RB^{m\times n}$ and $\B \in \RB^{m\times n}$ have the same size.
\cite{Gibson:1974} proved that  they have  joint factorizations  $\A=\U \Si_{A} \V^T$ and $\B = \U \Si_{B}  \V^T$ if and only if $\A \B^T$ and $\B^T \A$ are both normal. Here $\U$ and $\V$ are orthonormal matrices, and both $\Si_A$ and $\Si_B$ are diagonal but their diagonal elements are perhaps complex.  These diagonal elements are nonnegative only if both $\A \B^T$ and $\B^T \A$ are SPSD.

 %assume they have the same size. The following  theorem was proposed by .
%\begin{theorem} Given two matrices $\A, \B \in \RB^{m\times n}$, then they have Full SVD $\A=\U \Si_{A} \V^T$ and $\B = \U \Si_{B} %\V^T$ if and only if $\A \B^T$ and $\B^T \A$ are both normal.
%\end{theorem}

%%%%%%%%%%%%%%%%%%%%%%%%%%%%%%%%%%%
\chapter{Applications of SVD: Case Studies}
\label{ch:app1}

In the previous chapter we present the basic notion and some important properties of SVD. Meanwhile, we show that many matrix properties can be rederived via SVD.  In this chapter,  we further  illustrate  applications of  SVD in matrices, including in the definition of the Moore-Penrose pseudoinverse of an arbitrary matrix and in the  analysis of the Procrustes  problem.

For any matrix,  the  Moore-Penrose pseudoinverse exists and is unique. Moreover, it has been found to have many applications. Thus,  it is an important matrix notion. In this chapter we exploit the  matrix  pseudoinverse to solve least squares estimation, giving rise to a more general result. We also show that the  matrix  pseudoinverse can be used to deal with a class of generalized eigenvalue problems.

In fact, SVD has also wide applications in machine learning and data analysis.  For example, SVD is
an important tool in spectral analysis \citep{azar2001spectral}, latent semantic indexing
\citep{papadimitriou1998latent}, spectral clustering, and projective clustering \citep{feldman2013turning}.
%Many classification problems can be cast into a regularized regression framework
%\citep{Drineas:2006:SAL}.
We specifically show that SVD plays a  fundamental role  in subspace methods such as PCA, MDS, FDA and CCA.

%%%%%%%%%%%%%%%%%%%%%%%%%%%%%%%%%%%%
%%%%%%%%%%%%%%%%%%%%%%%%%%%%%%%%%%%
\section{The Matrix  MP Pseudoinverse}
\label{sec:pseudo}

Given a matrix $\A \in \RB^{m\times n}$ and a vector $\b \in \RB^m$, we are concerned with the least squares estimation problem:
\begin{equation} \label{eqn:lsp}
\hat{\x} =  \argmin_{\x \in \RB^n } \; \|\A \x - \b\|_2^2.
\end{equation}
The minimizer should satisfy the Karush-Kuhn-Tucker (KKT) condition: that is, it is the solution of the following normal equation:
\begin{equation} \label{eqn:normal-eqn}
\A^T \A {\x} = \A^T \b.
\end{equation}
Let $\A= \U_r \Si_r \V_r^T$ be the condensed SVD of $\A$.  Then
$\V_r \Si_r^2 \V_r^T {\x} = \V_r \Si_r \U_r^T \b$.  Define $\A^{\dag} = \V_r \Si_r^{-1} \U_r^T \in \RB^{n\times m}$. Obviously,
\[
\hat{\x} = \A^{\dag} \b
\]
is a minimizer. It is clear that  if $\A$ is invertible, then the minimizer is $\hat{\x} = \A^{-1} \b$. Thus, $\A^{\dag}$ is a generalization of $\A^{-1}$ in the case that  $\A$ is an arbitrary matrix, i.e., it is not necessarily invertible   and  even non-square.
This leads us to the notion of the matrix Moore-Penrose (MP) pseudoinverse~\citep{adi2003inverse}.

%In fact, $\A^{\dag}$ is called
\begin{definition} Given a matrix $\A \in \RB^{m\times n}$, a real $n\times m$  matrix $\B$ is called the MP pseudoinverse of $\A$ if it satisfies the following four conditions: (1) $\A \B \A= \A$, (2) $\B \A \B = \B$, (3) $(\A \B)^T = \A \B$, and (4) $(\B \A)^T = \B \A$.
 \end{definition}

It is easily verified that $\A^{\dag}= \V_r \Si_{r}^{-1} \U_r^T$ is a pseudoinverse of $\A$.
Moreover, when $\A$ is invertible, $\A^{\dag}$ is identical to $\A^{-1}$.
The following theorem then shows that  $\A^{\dag}$ is the unique  pseudoinverse of $\A$.
 \begin{theorem} \label{thm: pin verse}  Let $\A= \U_r \Si_r \V_r^T$ be the condensed SVD of  $\A \in \RB^{m\times n}$. Then   $\B$ is the pseudoinverse of $\A$ if and only if $\B=\A^{\dag}\triangleq \V_r \Si_r^{-1} \U_r^T$.
\end{theorem}
\begin{proof} To complete the proof, it suffices  to prove the uniqueness of the pseudoinverse. Assume that $\B$ and $\C$ are two   pseudoinverses of $\A$. Then
\begin{align*}
\A \B  & = (\A \B)^T = \B^T \A^T = \B^T (\A \C \A)^T = \B^T \A^T \C^T \A^T  \\
& = (\A \B)^T (\A \C)^T  = (\A \B  \A) \C = \A \C.
\end{align*}
Similarly, it also holds that  $\B \A = \C \A$. Thus,
\[
\B = \B \A \B = \B \A \C = \C \A \C = \C.
\]
\end{proof}

The matrix pseudoinverse   also has   wide applications. Let us see its application in solving generalized eigenproblems.
Given two matrices $\M \mbox{ and } \N \in \RB^{m{\times}m}$,  we refer to $(\Lam, \X)$ where $\Lam =\diag(\lambda_1, \ldots,
\lambda_q)$ and $\X=[\x_1, \ldots, \x_q]$ as $q$ eigenpairs of the matrix pencil $(\M, \N)$ if $\M  \X  = \N  \X \Lam$;
namely,
%\begin{equation} \label{eq:gep0}
\[
\M \x_i  =  \lambda_i \N  \x_i, \quad \mbox{ for } i=1, \ldots, q.
\] %\end{equation}
The problem of finding eigenpairs of $(\M,  \N)$ is known as a \emph{generalized eigenproblem}. Clearly, when $\N=\I_m$, the problem becomes the conventional eigenvalue problem.

Usually,  we are interested in the problem with the nonzero $\lambda_i$ for $i=1, \ldots,
q$ and refer to $(\Lam, \X)$ as the nonzero eigenpairs of $(\M,
\N)$.  If $\N$ is nonsingular,  $(\Lam,  \X)$ is also referred to
as the (nonzero) eigenpairs of $\N^{-1} \M$ because the
generalized eigenproblem is equivalent to the eigenproblem:
\[
\N^{-1} \M  \X  =  \X  \Lam.
\]
However, when $\N$ is singular,  \cite{ZhangJMLR:2010} suggested  to use a  pseudoinverse eigenproblem:
%\begin{equation}
\[
\N^{\dag}  \M  \X  =  \X  \Lam.
\]
%\end{equation}
Moreover, \cite{ZhangJMLR:2010}  established a connection between the solutions of the generalized eigenproblem and its corresponding
pseudoinverse eigenproblem. That is,

\begin{theorem} \label{thm:gep_pinv}
Let $\M$ and $\N$ be two  matrices in $\RB^{m{\times}m}$. Assume
$\rg(\M)  \subseteq  \rg (\N)$. Then, if $(\Lam, \X)$
are the nonzero eigenpairs of  $\N^{\dag}  \M$, we have that $(\Lam, \X)$
are the nonzero eigenpairs of the matrix pencil $(\M, \N)$.
Conversely, if $(\Lam,  \X)$ are the nonzero eigenpairs of the matrix
pencil $(\M,  \N)$, then $(\Lam,  \N^{\dag}\N  \X)$ are the
nonzero eigenpairs of  $\N^{\dag} \M$.
\end{theorem}

\begin{proof}
Let $\M = \U_1\Gam_1 \V_1^T$ and $\N = \U_2 \Gam_2 \V_2^T$ be the condensed SVD of $\M$ and $\N$. Thus, we have ${\rg}(\M)
= {\rg}(\U_1)$ and ${\rg}(\N) = {\rg}(\U_2)$. Moreover, we have $\N^{\dag} = \V_2 \Gam_2^{-1} \U_2^T$ and $\N \N^{\dag} =
\U_2 \U_2^T$. It follows from ${\rg}(\M) \subseteq {\rg}(\N)$ that ${\rg}(\U_1) \subseteq {\rg}(\U_2)$. This
implies that  $\U_1$ can be expressed as $ \U_1 = \U_2 \Q$  where
$\Q$ is some matrix of appropriate order.  As a result, we have
\[
\N  \N^{\dag} \M  =  \U_2  \U_2^T  \U_2 \Q \Gam_1 \V_1^T = \M.
\]
It is worth noting that the condition $\N \N^{\dag} \M  =
\M$ is not only necessary but also  sufficient for ${\rg}(\M)  \subseteq {\rg}(\N)$.

If $({\Lam},  {\X})$ are the eigenpairs of $\N^{\dag}  \M$, then it
is easily seen that $({\Lam},  {\X})$ are also the
eigenpairs of $(\M,  \N)$ due to $\N  {\N}^{\dag}  {\M}  =  {\M}$.

Conversely, suppose $({\Lam},  {\X})$ are the eigenpairs of $(\M, \N)$. Then
we have ${\N}  {\N}^{\dag}  {\M}   {\X} =  {\N}  {\X}  {\Lam}$.
This implies that $({\Lam},  {\N}^{\dag}  {\N}  {\X})$ are the eigenpairs of
$\N^{\dag}  \M$ due to ${\N}  {\N}^{\dag}  {\M}  =  {\M}$ and ${\N}^{\dag} {\N}  {\N}^{\dag} = {\N}^{\dag}$.
\end{proof}

Fisher  discriminant analysis (FDA) is a classical method for classification and dimension reduction simultaneously \citep{Mardia:1979}.
It is essentially a generalized eigenvalue problem in which the matrices $\N$ and $\M$ correspond to
a pooled scatter matrix and a between-class scatter matrix~\citep{JYeJMLR:2006,ZhangJMLR:2010}.  Moreover, the condition ${\rg}(\M)  \subseteq {\rg}(\N)$ meets. Thus, Theorem~\ref{thm:gep_pinv}
provides a solution when the pooled scatter matrix is singular or nearly singular. We will present more details about FDA in Section~\ref{sec:subspace}. 

\section{The Procrustes Problem}
\label{sec:procrustes}

Assume that  $\X\in \RB^{n{\times} p}$ and $\Y \in \RB^{n{\times}p}$ are two configurations of $n$  data points.
%Assume that there are two sets of observations $\X\in \RB^{n{\times} p}$ and $\Y \in \R^{n{\times}p}$ for a data sample.
The orthogonal Procrustes analysis aims to move $\Y$ relative into $\X$ through rotation~\citep{GowerBook:2004}.

In particular,  the  Procrustes   problem is defined as
\begin{equation} \label{eqn:procrustes}
\min_{\Q \in \RB^{p{\times} p}} \; \|\X - \Y \Q \|_F^2 \;  \mbox{ s.t. } \; \Q^T \Q = \I_p.
\end{equation}

\begin{theorem} \label{thm:procrustes} Let the full SVD of $\Y^T \X$ be $\Y^T \X = \U \Si \V^T$. Then $\U \V^T$ is the minimizer of
the Procrustes problem in \eqref{eqn:procrustes}.
\end{theorem}

\begin{proof} Since $\|\X - \Y \Q \|_F^2 = \tr((\X - \Y \Q)^T (\X - \Y \Q) )= \tr(\X^T \X) + \tr(\Y^T \Y)- 2 \tr(\Y^T \X \Q^T )$, the original problem is equivalent to
\[
\max \; \tr(\Y^T \X \Q^T) \; \mbox{ s.t. } \; \Q^T \Q = \I_p.
\]
Recall  that
the constants $\Q^T \Q = \I_p$ are equivalent to that $\bq_i^T \bq_i =1$ for $i=1, \ldots, p$, and $\bq_i^T \bq_j =0$
for $i\neq j$. Here the $\bq_i$ are the columns of $\Q$.  Thus, the Lagrangian function is
\[  \tr(\Y^T \X \Q^T) - \frac{1}{2} \sum_{i=1}^p c_{ii} (\bq_i^T \bq_i -1 ) - \frac{1}{2} \sum_{i>j} c_{ij} (\bq_i^T \bq_j -0 ),
\]
which is written in matrix form as
\[
L(\Q, \C) = \tr(\Y^T \X \Q^T) - \frac{1}{2}\tr[\C(\Q^T \Q -\I_p)],
\]
where $\C=[c_{ij}]$ is a symmetric matrix of the Lagrangian multipliers.
%The corresponding Lagrangian function  is given as
%\[
%L(\Q, \C) = \tr(\Y^T \X \Q^T) - \frac{1}{2} \tr(\C (\Q^T \Q - \I_p)),
%\]
%where $\C$ is a symmetric matrix of Lagrangian multipliers.

Since
\[
d L = \tr(\Y^T \X d \Q^T ) - \frac{1}{2} \tr(\C (d \Q^T \Q + \Q^T d \Q)),
\]
we have $\frac{d L}{d \Q} = \Y^T \X - \Q \C$. Letting the first-order derivative be zero yields %$\Y^T \X - \Q \C=0$. This
%shows that $\C = \Q^T \Y^T \X$.
\[
\Y^T \X - \Q \C= \0.
\]
Let $\hat{\Q}= \U \V^T$ and $\hat{\C}=\V \Si \V^T$, which are obviously the solutions of the above equation systems.

The Hessian matrix of $L$ w.r.t.\ $\Q$ at $\Q= \hat{\Q}$ and $\C = \hat{\C}$ is $- (\V \Si \V^T) \otimes \I_p$, which is negative definite.
Thus, $\Q= \U \V^T$ is the minimizer of the Procrustes problem.
\end{proof}

%Recall in  Theorem~\ref{thm:procrustes} there are
%the constants $\Q^T \Q = \I_p$, which are equivalent to that $\bq_i^T \bq_i =1$ for $i=1, \ldots, p$, and $\bq_i^T \bq_j =0$
%for $i\neq j$. Here the $\bq_i$ are the columns of $\Q$.  Thus, the Lagrange function is
%\[\tr(\Y^T \X \Q^T) - \frac{1}{2} \sum_{i=1}^p c_{ii} (\bq_i^T \bq_i -1 ) - \frac{1}{2} \sum_{i>j} c_{ij} (\bq_i^T \bq_j -0 ),
%\]
%which is written in matrix form as
%\[
%\tr(\Y^T \X \Q^T) - \frac{1}{2}\tr[\C(\Q^T \Q -\I_p)],
%\]
%where $\C=[c_{ij}]$ is a symmetric matrix of the Lagrangian multipliers.

%\section{Remarks}

%We conduct further discussions about SVD.  In particular, we present more materials of applications of SVD in machine learning.
%Especially, we illustrate the role of SVD in PCA, MDS, and FDA.
%We also give a very brief review about the QR factorization, because it is the most important counterpart  of SVD~\citep{hong1992rank}.

\section{Subspace Methods: PCA, MDS,  FDA, and CCA}
\label{sec:subspace}

Subspace methods, such as principal component analysis (PCA), multidimensional scaling (MDS),  Fisher discriminant analysis (FDA), and canonical correlation analysis (CCA), are a class of important machine learning methods. SVD plays a fundamental  role in subspace learning methods.

PCA~\citep{Jolliffe:2002,KittlerYoung:1973}  and MDS \citep{Cox:2000} are two classical dimension reduction methods.  Let  $\A=[\a_1, \ldots, \a_n]^T$ be a given data matrix in which each vector $\a_i$ represents
a data instance in $\RB^p$.  Let ${\m} = \frac{1}{n}\sum_{i=1}^n \a_i=\frac{1}{n} \A^T \1_n$ be  the sample mean and $\C_n = \I_n - \frac{1}{n} \1_n \I_n^T$ be a so-called centered matrix.  The  pooled scatter matrix  is defined as (a multiplier $1/n$ omitted)
\[
\S =\sum_{i=1}^n (\a_i - {\m}) ( \a_i - {\m} )^T=  \A^T \C _n \C_n \A = \A^T \C _n\A.
\]
%and $\C_n \A \A^T \C_n$ then defines a Gram matrix.

It is well known that PCA computes the spectral decomposition of $\S$, while the classical MDS or principal coordinate analysis (PCO) computes the spectral decomposition of  the Gram matrix $\C_n \A \A^T \C_n$. Proposition~\ref{pro:012}-(5)-(6) show that it is equivalent to computing SVD directly on the centerized data matrix $\C_n \A$. Thus,
SVD bridges PCA and PCO. That is, there is a duality relationship between PCA and PCO~\citep{Mardia:1979}.
This relationship has found usefulness  in  latent semantic analysis, face classification, and microarray data analysis   \citep{deerwester1990lsa,Turk:1991b,GolubT:1999,Belhumeur:1997,mullerSIAMREW:2004}.

FDA is a joint approach for dimension reduction and classification. Assume that the $\a_{i}$ are to be grouped into $c$ disjoint classes and that
each $\a_i$ belongs to one and only one class.
Let $V=\{1, 2, \ldots, n\}$ denote
the index set of the data points $\a_i$  and partition $V$ into $c$ disjoint subsets
$V_j$; that is, $V_i \cap V_j =\varnothing$ for $i\neq j$ and
$\cup_{j=1}^c V_j = V$, where the cardinality of $V_j$ is $n_j$ so
that $\sum_{j=1}^c n_j =n$. We also make use of a matrix representation
for the partitions. In particular, we let $\E=[e_{ij}]$  be an
$n{\times}c$ indicator matrix with $e_{ij} = 1$ if
input $\a_i$ is in class $j$ and $e_{ij} = 0$ otherwise.

Let $\m_j = \frac{1}{n_j} \sum_{i \in V_j} \a_i$ be the $j$th class
mean for $j = 1, \ldots, c$.  The
between-class scatter matrix is defined as  ${\S}_b =  \sum_{j=1}^c n_j (\m_j -
\m) (\m_j - \m)^T$.  Conventional FDA  solves the following generalized
eigenproblem:
%\begin{equation} \label{eq:lda_gep}
\[
{\S}_{b} \x_j = \lambda_j {\S} \x_j, \quad
\lambda_1 \geq \lambda_2 \geq \cdots \geq \lambda_{q}> \lambda_{q{+}1}= 0,
\]
%\end{equation}
where $q \leq
\min\{p, \; c{-}1\}$ and where we refer to $\x_{j}$ as the $j$th discriminant
direction.  The above generalized eigenproblem can can be expressed in matrix form:
\begin{equation} \label{eq:lda_m}
\S_b \X  = \S  \X \Lam,
\end{equation}
where $\X =[\x_1, \ldots, \x_q]$ ($n{\times}q$) and $\Lam = \diag(\lambda_1, \ldots, \lambda_q)$ ($q{\times}q$).

Let $\Pii = {\diag}(n_1,  \ldots, n_c)$. Then $\S_b$ can be rewritten as
\[
\S_b =  {\A}^T  \C_n \E  \Pii^{-1} \E^T  \C_n {\A}.
\]
Recall that   $\S= \A^T \C_n \C_n \A$.
Given these representations of $\S$ and $\S_{b}$, the problem in (\ref{eq:lda_m}) can be solved by
using the GSVD method~\citep{LoanGSVD:1976,PaigeSIAM:1981,golub2012matrix,HowlandSIAM:2003}.
Moreover,  it is obvious that $\rg(\S_b)\subseteq \rg(\A^T \C_n) = \rg(\S)$. Thus, Theorem~\ref{thm:gep_pinv}
provides a solution when $\S$ is singular or nearly singular.
Moreover, the method given in Theorem~\ref{thm:gsvd-sp} is appropriate for  solving the FDA problem.

CCA is another subspace learning model~\citep{HardoonCCA:2004}.  The primary focus is on the relationship between two groups of variables (or features), whereas PCA considers interrelationships within a set of variable. Mathematically, CCA is defined as a generalized eigenvalue problem, so its solution can be borrowed from that of FDA.

\subsection{Nonlinear Extensions}

Reproducing kernel theory \citep{Aronszajn:1950} provides an approach for nonlinear extensions  of  subspace methods. For example, kernel PCA~\citep{Scholkopf:1998},
kernel FDA~\citep{Baudat:2000,MikaNIPS:2000,Roth:2000}, kernel CCA~\citep{AkahoKCCA:2001,GestelICANN:2001,BachKICA:2004} have been successively proposed and received wide applications in data analysis.

Kernel methods  work in a feature space
$\FM$, which is related to the original input space $\XM \subset
\RB^p$ by a mapping,
\[
\vp: \XM \rightarrow \FM.
\]
That is, $\vp$ is a vector-valued function which gives a vector
$\vp(\a)$, called a \emph{feature vector}, corresponding to an
input $\a \in \XM$. In kernel methods, we are given
a reproducing kernel $K: \XM \times \XM
\rightarrow \RB$ such that $K(\a, \b)=\vp(\a)^T  \vp(\b)$ for $\a, \b
\in \XM$. The mapping $\vp(\cdot)$ itself is typically not given
explicitly.   
Rather, there exist only inner products between
feature vectors in $\FM$. In order to implement a kernel method
without referring to $\vp(\cdot)$ explicitly, one resorts to the
so-called \emph{kernel trick}~\citep{ScholkopfBook:2002,ShaweTaylorBook:2004}.

Let $L_2(\XM)$ be the square integrable Hilbert space of functions
whose elements are functions defined on $\XM$. It is a well-known
result that if $K$ is a reproducing kernel for the Hilbert space
$L_2(\XM)$, then $\{K(\cdot, \b)\}$ spans $L_2(\XM)$. Here
$K(\cdot, \b)$ represents a function that is defined on $\XM$ with
values at $\a \in \XM$ equal to $K(\a, \b)$. There are some common
kernel functions:
\begin{enumerate}
  \item[(a)] Linear kernel: $K(\a, \b)=\a^T \b$,
  \item[(b)] Gaussian kernel or radial basis function (RBF): $K(\a, \b)=\exp\big(-\sum_{j=1}^p \frac{(a_{j}{-}b_{j})^2}{\beta_j}\big)$
   with $\beta_j>0$,
   \item[(c)] Laplacian kernel: $K(\a, \b)=\exp\big( {-} \sum_{j=1}^p \frac{|a_{j}{-}b_{j}|}{\beta_j}\big)$
   with $\beta_j>0$,
  \item[(d)] Polynomial kernel: $K(\a, \b) =(\a^T \b +1)^{d}$ of degree $d$.
\end{enumerate}

Given a training set of  input vectors $\{\a_1, \ldots, \a_n\}$, the kernel matrix $\K = [K(\a_i, \a_j)]$ is an $n\times n$ SPSD matrix.

%%%%%%%%%%%%%%%%%%%%%%%%%%%%%%%%%%%%%%
%%%%%%%%%%%%%%%%%%%%%%%%%%%%%%%%%%%%%%%%%%%%%%%%%%
\chapter{The QR and CUR Decompositions}
\label{ch:variant}

The QR factorization and CUR decomposition are the two most important counterparts  of SVD.  These three factorizations  apply to all matrices.
In Table~\ref{tab:mf}
we have compared their primary focuses.  The SVD  and QR factorization are two classical matrix theories. The CUR decomposition aims to 
represent a data matrix in terms of a small number part  of the matrix, which  makes it easy for us to understand and interpret the data in question.  
Here we  present  very brief introductions to the QR factorization and CUR decomposition.

\section{The QR  Factorization}
\label{sec:qr}

%We now give a very brief review about the QR factorization, because it is the most important counterpart  of SVD. 
The QR factorization is another decomposition method applicable all matrices. Given a matrix $\A \in \RB^{m\times n}$, the QR factorization is given by
\[
\A = \Q \R,
\]
where $\Q \in \RB^{m\times m}$ is orthonormal and $\R \in \RB^{m\times n}$ is upper triangular (or low triangular).
Let $\D$ be an $m\times m$ diagonal matrix whose diagonal elements are either 1 or $-1$. Then $\A = (\Q \D) (\D\R)$
is still a QR factorization of $\A$. Thus, we always assume that $\R$ has nonnegative diagonal elements.

Assume  $m\geq n$. The matrix $\A$ also has a thin QR factorization:
\[
\A = \Q \R,
\]
where $\Q \in \RB^{m\times n}$ is currently column orthonormal, and $\R \in \RB^{n\times n}$ is
 upper triangular with nonnegative diagonal elements. If $\A$ is of rank $n$, $\R$ is uniquely determined.
In this case, $\Q=\A \R^{-1}$ is also uniquely determined. 

Asume  $\A$ has rank $r$ ($\leq \min\{m, n\}$). Then there exists an $m\times m$ orthonormal matrix $\Q$
and an $n\times n$ permutation matrix $\PP$ such that 
\[
\Q^T \A \PP = \begin{bmatrix}  \R_{11} & \R_{12} \\ \0 & \0  \end{bmatrix},
\]
where $\R_{11}$ is an $r\times r$ upper triangular matrix with positive diagonal elements. 
This is called a \emph{rank revealing QR factorization}.

Computation of the QR factorization  can be arranged by the novel Gram-Schmidt orthogonalization process or the modified Gram-Schmidt which is  numerically more stable~\citep{trefethenbau}.  %Please refer to any textbook in numerical linear algebra for the details.
Additionally, \cite{gu1996efficient}  proposed efficient algorithms for computing a  rank-revealing QR factorization~\citep{hong1992rank}.
\cite{stewart1999four}  devised efficient computational algorithms of truncated pivoted {QR} approximations to a sparse matrix.

\section{The CUR Decomposition}

As we have see,  SVD  leads us to a geometrical representation,  and the QR factorization facilitates computations.  They have little concrete meaning. 
This makes it difficult for us to understand and interpret the data in question.

\citet{kuruvilla2002vector} have  claimed:
``it would be interesting to try to find basis vectors for all experiment vectors, using actual experiment vectors and not artificial bases that offer little insight.''
Therefore, it is of great interest to represent a data matrix in terms of a small number of actual columns and/or actual rows of the matrix. Matrix column selection  and  CUR matrix decomposition provide such  techniques.

Column selection yields a so-called CX decomposition, and
the CUR decomposition can be be regarded as a special  CX decomposition. 
The CUR  decomposition problem has been widely discussed in the literature
\citep{goreinov1997pseudoskeleton,goreinov1997maximalvolume,stewart1999four,tyrtyshnikov2000incompletecross,
berry2005algorithm,drineas2005nystrom,bien2010cur},
and it has been shown to be very useful in high dimensional data analysis.

The  \CUR was originally called a skeleton decomposition~\citep{goreinov1997pseudoskeleton}. Let $\A \in \RB^{m\times n}$ be a given matrix of rank $r$. Then there exists a nonsingular  $r{\times} r$ submatrix in $\A$.
Without loss of generality, assume this nonsingular matrix is the first $r\times r$ principal submatrix of $\A$. That is, $\A$ can be partioned into the following form:
\[
\A = \begin{bmatrix} \A_{11} & \A_{12} \\ \A_{21} & \A_{22} \end{bmatrix},
\]
where $\A_{11}$ is a $r\times r$ nonsingular matrix. Consider that  $[\A_{21}, \A_{22}] = \B [\A_{11}, \A_{12}]$ for some $\B \in \RB^{(m-r) \times r}$. It follows from $\A_{21}= \B \A_{11}$ that $\B= \A_{21} \A_{11}^{-1}$. Hence, $\A_{22} = \A_{21} \A^{-1}_{11} \A_{12}$. So  it is obtained that
\[
\A = \begin{bmatrix} \A_{11} \\ \A_{21} \end{bmatrix} \A_{11}^{-1} [\A_{11}, \A_{12}].
\]
In general case,  let  $\A_{I, J}$ be the nonsingular submatrix where
$I = \{i_1, \ldots, i_r\} \subset [m]$ and $J =\{j_1, \ldots, j_r\} \subset [n]$.  Then it also hods that
\[
\A = \C  \A_{I, J}^{-1} \R,
\]
where $\C = \A_{:, J}$ and $\R = \A_{I, :}$ are respectively a subset of columns and a subset of rows, of $\A$.

In practical applications, however,  it is intractable to select $\A_{I, J}$. Alternatively, 
\citet{stewart1999four} 
proposed a  quasi Gram-Schmidt algorithm, obtaining a  sparse column-row (SCA) approximation  of the original matrix $\A$~\citep{berry2005algorithm}. The SCA approximation is of the form $\A \approx \X \T \Y$, where $\X$ and $\Y$ consist of columns and rows of $\A$, and $\T$ minimizes $\|\A- \X \T \Y\|_F^2$. This algorithm is a deterministic peocedure but computationally expensive.   
%SCRA is based on the truncated pivoted QR decomposition via . 
%This algorithm is quite effective but very time expensive. %especially when  $r$ is large.

The terminology of the CUR decomposition has been proposed by 
\cite{drineas2005nystrom,mahoney2008tensor}. They reformulated  the idea based on  random selection.  A \CUR decomposition algorithm seeks to find a subset of $c$ columns of $\A$ to form a matrix $\C \in \RB^{m{\times} c}$,
a subset of $r$ rows to form a matrix $\R \in \RB^{r{\times} n}$, and an intersection matrix $\U \in \RB^{c{\times}r}$ such that $\|\A - \C \U \R\|_\xi$  is small.
Accordingly,  $\tilde{\A} = \C \U \R$ is used to approximate $\A$.  

Since there are $(^n_c)$ possible choices of constructing $\C$ and $(^m_r)$ possible choices of constructing $\R$,
obtaining the best CUR decomposition is  a hard problem.   In Chapter~\ref{ch:lsma}
we will further study the CUR decomposition problem via random approximation.

The CUR decomposition is also an extension of the novel Nystr\"{o}m approximation to a general  matrix. 
The  Nystr\"{o}m  method approximates an SPSD matrix only using a subset of its columns,
so it can alleviate computation and storage costs when the SPSD matrix in question is large in size.
Thus, the \nystrom method and its variants \citep{halko2011finding,gittens2013revisiting,kumar2009ensemble,WangZhangJMLR:2013,wang2014efficient,WangZhangKDD:2014,si2014memory} have been extensively used in the machine learning community.
For example, they have been applied to Gaussian processes \citep{williams2001using},
kernel classification \citep{zhang2008improved,jin2013improved}, spectral clustering \citep{fowlkes2004spectral},
kernel PCA and manifold learning \citep{talwalkar2008large,zhang2008improved,zhang2010clustered}, determinantal processes \citep{affandi2013nystrom}, etc.

%%%%%%%%%%%%%%%%%%%%%%%%%%%%%%%%%%%%%%%%%%%%%
\chapter{Variational Principles}
\label{ch:variation}

Variational principles  correspond to  matrix  perturbation theory~\citep{StewartSunBook:1990}, which is the theoretical foundation to characterize stability  or sensitivity of a matrix computation algorithm.  
Thus,  variational principles are important in analysis for error bounds of matrix approximate algorithms (see Chapters~\ref{ch:lowrank} and \ref{ch:lsma}).

In this chapter we specifically study variational properties for eigenvalues of a symmetric matrix as well as for singular values of a general matrix.
We will see that these results  for eigenvalues and for singular values are almost parallel.
The cornerstones are the novel von Neumann theorem~\citep{Neumann:1937} and  Ky Fan theorem~\citep{Fan:1951}.
We present new proofs for them by using theory of matrix differentials.
Additionally,  we present some majorization  inequalities.
They will be used in the latter chapters, especially in investigating unitarily invariant norms (see Chapter~\ref{ch:uinorm}).

%we first give the generalization of the Courant-Fischer theorem for singular values,  and then   introduce  two novel  theorems: the  %von Neumann theorem and the Ky Fan theorem. Finally,  we present some variational properties based on majorization theory.

Given a matrix $\A\in \RB^{m\times n}$, we always let $\sigma_1(\A) \geq  \cdots \geq \sigma_{p}(\A)$ be the singular values of $\A$ where $p=\min\{m, n\}$. When $\A$ is  symmetric, let $\lambda_1(\A) \geq \cdots \geq \lambda_n(\A)$ be the eigenvalues of $\A$.
These eigenvalues or singular values  are always arranged  in deceasing order.
Note that the eigenvalues are real but could be negative.
Let $\lamb(\M)= (\lambda_1(\M), \ldots, \lambda_n(\M))^T$ denote  the eigenvalues of an $n\times n$ real square matrix $\M$,  and   $\sib(\A)= (\sigma_1(\A), \ldots, \sigma_p(\A))^T$ denote  the singular values of an $m \times n$ real
matrix $\A$.
Sometimes we also write them the $\sigma_i$ or the $\lambda_i$  when they are explicit in the context  for notational simplicity.

\section{Variational Properties for Eigenvalues}

In this section we consider variational properties for eigenvalues of a real symmetric matrix. It is well known that for an arbitrary symmetric matrix, its eigenvalues are all real.  
The following cornerstone theorem was originally established by von \cite{Neumann:1937}.

\begin{theorem}[von Neumann Theorem] \label{thm:von00}
Assume $\M \in \RB^{n\times n}$ and $\N \in \RB^{n\times n}$ are symmetric.  Then
\[
\sum_{i=1}^n \lambda_i(\M) \lambda_i(\N) = \max_{\Q \Q^T=\I_n} \;  \tr(\Q \M \Q^T \N).
\]
Moreover,
\[
\sum_{i=i}^n \lambda_i(\M) \lambda_{n-i+1} (\N) = \min_{\Q \Q^T=\I_n} \;  \tr(\Q \M \Q^T \N).
\]
\end{theorem}

\begin{proof} The second part directly follows from the first part because
\[
 \min_{\Q \Q^T =\I_n} \;  \tr(\Q \M \Q^T \N) = - \max_{\Q \Q^T=\I_n} \;  \tr(\Q \M \Q^T (-\N)).
\]
We now present the proof of the first part.
Make  full EVDs of $\M$ and $\N$ as $\M= \U_M \Lam_M \U_M^T$ and $\N = \U_N \Lam_N \U_N^T$, where  $\Lam_M=\diag(\lambda_1(\M), \ldots, \lambda_n(\M))$ and $\Lam_N=\diag(\lambda_1(\N), \ldots, \lambda_n(\N))$, and $\U_M$
and $\U_N$ are orthonormal.  It is easily seen that
\begin{align*}
 \max_{\Q \Q^T = \I_n} \;  \tr(\Q \M \Q^T \N)  & =  \max_{\Q \Q^T =\I_n} \;  \tr(  (\U_N^T \Q \U_M) \Lam_M  (\U_N^T \Q  \U_M)^T \Lam_N) \\
 &= \max_{\Q \Q^T =\I_n} \;  \tr( \Q \Lam_M \Q^T  \Lam_N).
\end{align*}
Let $\Q=[q_{ij}] = [\bq_1, \ldots, \bq_n]^T$. We now have
\begin{align*}
& \tr( \Q \Lam_M \Q^T  \Lam_N) \\
& = \sum_{i=1}^n \bq_i^T  \Lam_M \bq_{i}  \lambda_i(\N) \\
& = \sum_{i=1}^{n-1}  \sum_{j=1}^i \bq_j^T \Lam_M \bq_j [\lambda_i(\N) - \lambda_{i+1}(\N) ]  + \lambda_n(\N) \sum_{j=1}^n \bq_j^T \Lam_M \bq_j \\
& = \sum_{i=1}^{n-1}  [\lambda_i(\N) - \lambda_{i+1}(\N) ]  \sum_{j=1}^i  \sum_{k=1}^n  q_{jk}^2 \lambda_k(\M)  + \lambda_n(\N) \sum_{j=1}^n \lambda_j(\M).
\end{align*}
Define $\W \triangleq  [q_{ij}^2]$ which is doubly stochastic, and $\u=[u_1, \ldots, u_n]^T$ where $u_j = \sum_{k=1}^n  q_{jk}^2 \lambda_k(\M)$. That is, $\u = \W \lamb(\M)$. By Lemma~\ref{lem:001}, we know that $\u \prec \lamb(\M)$. Accordingly,
 \begin{align*}
 \tr( \Q  \Lam_M \Q^T  \Lam_N)  & \leq  \sum_{i=1}^{n-1}  [\lambda_i(\N) {-} \lambda_{i+1}(\N) ]  \sum_{j=1}^i \lambda_j(\M) + \lambda_n(\N) \sum_{j=1}^{n} \lambda_j(\M) \\
 & = \sum_{i=1}^n \lambda_i(\M) \lambda_i(\N).
\end{align*}
When $\Q=\I_n$, the equality holds. That is, $\U_N^T \Q \U_M = \I_n$ in the original problem. The theorem follows.
\end{proof}

The following theorem is a corollary of  Theorem~\ref{thm:von00}  when  taking
\[ \N = \begin{bmatrix} \I_k & \0 \\  \0 & \0  \end{bmatrix}. \]

\begin{theorem}[von Neumann Theorem] \label{thm:von}
Assume $\M \in \RB^{n\times n}$ is symmetric. Then for $k \in [n]$,
\[
\sum_{i=1}^k \lambda_i = \max_{\Q^T \Q=\I_k} \;  \tr(\Q^T \M \Q),
\]
which is arrived when $\Q$ is the $n\times k$ matrix of the orthonormal vectors associated with $\lambda_1, \ldots, \lambda_k$. Moreover,
\[
\sum_{i=n-k+1}^n \lambda_i = \min_{\Q^T \Q=\I_k} \;  \tr(\Q^T \M \Q).
\]
\end{theorem}

In the appendix we give an other proof based on theory of matrix differentials.
The von Neumann theorem  describes the variational principle of eigenvalues of a symmetric matrix.
Using Theorems~\ref{thm:von}, we have the following variational properties.

\begin{proposition} \label{pro:031} Given two  $n\times n$ real symmetric  matrices $\M$ and ${\N}$,  we have that
%let the $\lambda_i(\M)$, $\lambda_i({\bf N})$, and $\lambda_i(\M + {\bf N})$
%are the eigenvalues of $\M$, ${\bf N}$, and $\M + {\bf N}$, all in descending order. Let $\lamb(\M + {\bf N})= \diag(\lambda_1(\M), \ldots, \lambda_n(\M))$. Then
\begin{enumerate}
\item[(1)] $\lamb(\M + {\N}) \prec \lamb(\M) + \lamb(\N)$ and  $\lamb(\M) - \lamb(\N)  \prec \lamb(\M - {\N}) $.
\item[(2)] $\sum_{i=1}^k \lambda_i(\M + \N) \geq \sum_{i=1}^k \lambda_i(\M) + \sum_{j=n-k+1}^n \lambda_j(\N)$ for $k \in [n]$.
\item[(3)] $(m_{11} , \ldots, m_{nn})  \prec  (\lambda_1(\M),  \ldots,  \lambda_n(\M))$.
%\item[(3)] $\lambda_i(\M + \N) \geq \lambda_i(\M)$ (or $\lambda_i(\N)$) for $i\in [n]$ if $\M$ and $\N$ are PSD.
\end{enumerate}
\end{proposition}

\begin{proof} The proof is based on Theorem~\ref{thm:von}. First, for $k\in [n-1]$,
\begin{align*}
\sum_{i=1}^k \lambda_i (\M + \N)  & = \max_{\Q^T \Q=\I_k} \Big\{ \tr(\Q^T \M \Q) + \tr(\Q^T \N \Q) \Big\} \\
& \leq \max_{\Q^T \Q=\I_k}  \tr(\Q^T \M \Q) +  \max_{\Q^T \Q = \I_k} \tr(\Q^T \N \Q) \\
& = \sum_{i=1}^k \lambda_i (\M) +  \sum_{i=1}^k \lambda_i(\N).
\end{align*}
Note  that $\tr(\M {+} \N) = \tr(\M) + \tr(\N)$, so  $\lamb(\M {+} {\N}) \prec \lamb(\M) + \lamb(\N)$. Hence, $\lamb(\M) - \lamb(\N)  \prec \lamb(\M - {\N}) $.
Second,
\begin{align*}
\sum_{i=1}^k \lambda_i (\M + \N)  & = \max_{\Q^T \Q=\I_k} \Big \{ \tr(\Q^T \M \Q) + \tr(\Q^T \N \Q) \Big \} \\
& \geq \max_{\Q^T \Q=\I_k} \Big\{  \tr(\Q^T \M \Q) +  \min_{\Q^T \Q = \I_k} \tr(\Q^T \N \Q) \Big\} \\
& = \sum_{i=1}^k \lambda_i (\M) +  \sum_{j=n-k+1}^n \lambda_j(\N).
\end{align*}
To prove the third part,  we assume that $m_{11}\geq \cdots \geq m_{nn}$ without loss of generality. Now the result is obtained via
\[
\sum_{i=1}^k \lambda_i(\M) = \max_{\Q^T \Q=\I_k} \tr(\Q^T \M \Q) \geq \tr(\H_k^T \M \H_k) = \sum_{i=1}^k m_{ii},
\]
where $\H_k$ consists of the first $k$ columns of $\I_n$ for all $k \in [n]$.
\end{proof}

Proposition~\ref{pro:031}-(3) is sometimes referred to as Schur's theorem. The second part of the following proposition is an extension of Schur's theorem.

\begin{proposition} \label{pro:interacing}  Let $\M = \begin{bmatrix} \M_{11} & \M_{12} \\ \M_{21} & \M_{22} \end{bmatrix}$ be $n\times n$ real symmetric.  Here $\M_{11}$ is $k\times k$. Then
\[ (1) \quad  \lambda_i(\M) \geq \lambda_i(\M_{11}) \geq \lambda_{n-k+i} (\M)
\mbox { for } \; i=1, \ldots, k; \]
and (2) \;  $(\lamb(\M_{11}), \lamb(\M_{22})) \prec \lamb(\M)$.

Furthermore, for any column-orthonormal matrix $\Q \in \RB^{n\times k}$, we have
\[ (3) \quad  \lambda_i(\M) \geq \lambda_i(\Q^T \M \Q) \geq \lambda_{n-k+i} (\M) \; \mbox{ for } \; i=1, \ldots, k. \]
\end{proposition}
\begin{proof} The first result directly follows from the well known interlacing theorem~\citep{Horn:1985}. As for the third part, we can extend $\Q$ to an orthonormal matrix $\tilde{\Q}=[\Q, \Q^{\bot}]$. Consider that
\[
\tilde{\Q}^T \M \tilde{\Q} = \begin{bmatrix}  \Q^T \M \Q & \Q^T \M \Q^{\bot}  \\ (\Q^{\bot})^T \M \Q &   (\Q^{\bot})^T \M  \Q^{\bot} \end{bmatrix}.
\]
Thus,
\[\lambda_{i}(\M)  = \lambda_{i}(\tilde{\Q}^T \M \tilde{\Q}) \geq  \lambda_{i}(\Q^T \M \Q) \geq \lambda_{n-k+i}(\tilde{\Q}^T \M \tilde{\Q}) = \lambda_{n-k+i}(\M). \]

We now consider the proof of the second part. Let  the EVDs of $\M_{11}$ and $\M_{22}$ be $\M_{11}= \U_1 \Lam_1 \U_1^T$ and
$\M_{22} = \U_2 \Lam_2 \U_2^T$. Then
\[
\begin{bmatrix} \U_1^T & \0 \\ \0 & \U_2^T \end{bmatrix} \begin{bmatrix} \M_{11} & \M_{12} \\ \M_{21} & \M_{22} \end{bmatrix} \begin{bmatrix} \U_1 & \0 \\ \0 & \U_2 \end{bmatrix} =   \begin{bmatrix} \Lam_{1} & \U_1^T \M_{12} \U_2 \\ \U_2^T \M_{21} \U_1 & \Lam_{2} \end{bmatrix}.
\]
Since $\U_1$ and $\U_2$ are orthonormal,  we have that $\lamb(\M_{11}) = \lamb(\Lam_1)$, $\lamb(\M_{22}) =  \lamb(\Lam_2) $, and
\[
\lamb\left(  \begin{bmatrix} \U_1^T & \0 \\ \0 & \U_2^T \end{bmatrix} \begin{bmatrix} \M_{11} & \M_{12} \\ \M_{21} & \M_{22} \end{bmatrix} \begin{bmatrix} \U_1 & \0 \\ \0 & \U_2 \end{bmatrix}  \right)  = \lamb(\M).
\]
Applying Proposition~\ref{pro:031}-(3) completes the proof.
\end{proof}

%\begin{theorem} Given an $n\times n$ real symmetric matrix
%\[
%\M = \begin{bmatrix} \M_{11} & \M_{12} \\ \M_{21} & \M_{22} \end{bmatrix}
%\]
%where $\M_{11}$ and $\M_{22}$ are $p{\times}p$
%and $q{\times}q$, let
%\[ \N(\omega) = \begin{bmatrix} \M_{11} & \omega \M_{12} \\ \omega \M_{21} & \M_{22} \end{bmatrix} \; \mbox{ for } \; \omega \in [0, 1].
%\]
%Then $\lamb(\N (\omega_1)) \prec \lamb(\N(\omega_2))$ if $\omega_1 \leq \omega_2$. Especially, $\lamb(\N(0)) \prec \lamb(\M)$.
%\end{theorem}

\section{Variational Properties for Singular Values}

Theorems~\ref{thm:von00} and \ref{thm:von}  can be extended to a general matrix. 
In this case, we investigate singular values of the matrix instead. 
Theorems~\ref{thm:kyfan0} and \ref{thm:kyfan} correspond to Theorems~\ref{thm:von00} and \ref{thm:von}, respectively. 
%which were studied by J. von Neumann and Ky Fan.

\begin{theorem}[Ky Fan Theorem] \label{thm:kyfan0} Given  two matrices $\A \in \RB^{m\times n}$ and $\B \in \RB^{m\times n}$,
let  $\A$ and $\B$ have full SVDs $\A = \U_A \Si_A \V_A^T$ and $\B = \U_B \Si_B \V_B^T$, respectively.  Let $p=\min\{m, n\}$.
Then
\begin{align*}
\sum_{i=1}^p \sigma_i(\A) \sigma_i(\B)  &= \max_{\X^T \X = \I_m,  \Y^T \Y = \I_n} \; |\tr(\X^T \A \Y \B^T)| \\
&= \max_{\X^T \X = \I_m,  \Y^T \Y = \I_n} \; \tr(\X^T \A \Y \B^T),
\end{align*}
which is achieved at $\X=\U_A \U_B^T$ and $\Y=\V_A \V_B^T$.
\end{theorem}

\begin{proof}
Note that
\[
\tr(\X^T \A \Y \B^T) = \frac{1}{2} \tr\left(\begin{bmatrix} \Y^T & \0 \\ \0 & \X^T \end{bmatrix} \begin{bmatrix} \0 & \A^T \\ \A & \0 \end{bmatrix} \begin{bmatrix} \Y & \0 \\ \0 & \X \end{bmatrix}   \begin{bmatrix} \0 & \B^T \\ \B & \0 \end{bmatrix} \right).
\]
The theorem is directly obtained from Theorems~\ref{thm:von00} and \ref{thm:h-a}.
\end{proof}

\begin{theorem}[Ky Fan Theorem] \label{thm:kyfan} Given an $m\times n$ real matrix $\A$, let $p=\min\{m, n\}$, and let the singular values of $\A$
be $\sigma_1, \ldots, \sigma_p$ which are arranged in descending order, with the corresponding left and right singular vectors
$\u_i$ and $\v_i$.  Then for any $k \in [p]$,
\[
\sum_{i=1}^k \sigma_i =\max_{\X^T \X = \I_k,  \Y^T \Y = \I_k} \; |\tr(\X^T \A \Y)|= \max_{\X^T \X = \I_k,  \Y^T \Y = \I_k} \; \tr(\X^T \A \Y),
\]
which is achieved at $\X=[\u_1, \ldots, \u_k]$ and $\Y=[\v_1, \ldots, \v_k]$.
\end{theorem}
%\begin{proof}
%Note that
%\[
%\tr(\X^T \A \Y) = \frac{1}{2} \tr\left( [\Y^T, \X^T] \begin{bmatrix} \0 & \A^T \\ \A & \0 \end{bmatrix} \begin{bmatrix} \Y \\ \X \end{bmatrix}
%\right).
%\]
%\end{proof}
The theorem can be obtained from Theorems~\ref{thm:von} and \ref{thm:h-a} or from Theorem~\ref{thm:kyfan0}.
In the appendix we give the third  proof.

\begin{proposition} \label{pro:032} Given two   matrices $\A \in \RB^{m\times n}$ and ${\B} \in \RB^{m\times n}$,
let $p=\min\{m, n\}$.  Let   $\hat{\A}$ be obtained by replacing the last $r$ rows and/or columns of $\A$ by zeros.  Then
%we have that
\begin{enumerate}
\item[(1)] $\sib(\A + {\B}) \prec_{w} \sib(\A) + \sib(\B)$.
\item[(2)] $\sigma_{i+j-1}(\A + \B) \leq \sigma_i(\A) + \sigma_j(\B)$ for $i, j\geq 1$ and $i+j-1\leq p$.
\item[(3)]  $\a \prec_{w} \sib(\A)$ where  $\a=(a_{11}, \ldots, a_{pp})^T$.
%\item[(2)] $\sum_{i=1}^k \sigma_i(\A + \B) \geq \sum_{i=1}^k \sigma_i(\A) - \sum_{j=p-k+1}^p \sigma_j(\B)$ where $p=\min(m, n)$.
%\item[(2)] $\sib(\hat{\A})  \prec_{w} \sib(\A) $ %Moreover, $\sum_{i=1}^k \sigma(\hat{\A}) \geq \sum_{j=r+1}$
%\item[(4)] For $i\in [p]$, $\sigma_i(\hat{\A}) \leq \sigma_i(\A)$.
\item[(4)] For $i \in [p-r]$, $\sigma_{r+i}(\A) \leq  \sigma_i(\hat{\A}) \leq \sigma_i(\A)$.
\item[(5)] Let $\PP \in \RB^{m\times r}$ and $\Q \in \RB^{n\times r}$ be column orthonormal matrices where $r \leq p$. Then $\sigma_{r+i} (\A) \leq \sigma_i(\PP^T \A) \leq \sigma_i(\A)$ and $\sigma_{r+i}(\A) \leq \sigma_i( \A \Q) \leq \sigma_i(\A)$  for $i=1, \ldots, p-r$.
\end{enumerate}
\end{proposition}
\begin{proof}
The proof of Proposition~\ref{pro:032}-(1) and (3)  is parallel to that of Proposition~\ref{pro:031}-(1) and (3).
Part-(2) is Weyl's monotonicity theorem. It can be proven by the Courant-Fischer theorem (see Theorem~\ref{thm:CF-svd}).
Consider that $\sigma_i(\A) = \sqrt{\lambda_i(\A^T \A)}= \sqrt{\lambda_i(\A \A^T)}$ and $\sigma_i(\hat{\A}) = \sqrt{ \lambda_i(\hat{\A}^T \hat{\A})} = \sqrt{ \lambda_i(\hat{\A} \hat{\A}^T )}$. Part (4)
follows from  Proposition~\ref{pro:interacing}-(1).  Part (5)
follows  then from Proposition~\ref{pro:interacing}-(3).
\end{proof}
%Assume that $\hat{\A}$ is obtained by replacing the last $r$ rows of $\A$ by zeros, and $\A_1$ consists of the first $m{-}r$ rows of %$\A$.
%When $\hat{\A}$ is defined via the other two approaches, the results can be obtained similarly.

%Obviously, $\hat{\A}$ and $\A_1$ have the same nonzero singular values. Moreover, they have at most $m{-}r$ nonzero singular values.
%Thus, we only need to consider the case  $i \leq m{-}r$. In terms of Theorem~\ref{thm:CF-svd},
%\begin{align*}
%\sigma_i(\A) & = \min_{\v_1, \ldots, \v_{i-1} \in \RB^n } \max_{\begin{array}{c} \v \in \RB^{n}, \|\v\|_2=1 \\ \v^T[\v_1, \ldots, \v_{i-1}] =\0 \end{array}} \; \|\A \v\|_2 \\
%& =  \min_{\v_1, \ldots, \v_{i-1} \in \RB^n } \max_{\begin{array}{c} \v \in \RB^{n}, \|\v\|_2=1 \\ \v^T[\v_1, \ldots, \v_{i-1}] =\0 \end{array}} \; \|\A \v\|_2
%\end{align*}
%
%In this case, we have
%\[
%\sum_{i=1}^k \sigma_i(\A) = \max_{\begin{array}{ll} \X^T \X = \I_k, \\ \Y^T \Y = \I_k \end{array}} \tr(\X^T \A \Y) \geq
%\tr([\Q^T, \0] \A \Z) = \tr(\Q^T \A_1 \Z).
%\]
%for   an arbitrary $(m{-}r)\times k$ column orthonormal matrix $\Q$  and  an arbitrary $n \times k$ column orthonormal matrix $\Z$. Hence,
%\[
%\sum_{i=1}^k \sigma_i(\A) \geq \max_{\begin{array}{ll} \Q^T \Q = \I_k, \\ \Z^T \Z = \I_k \end{array}} \tr(\Q^T \A_1 \Z) = \sum_{i=1}^k \sigma_i(\A_1)= \sum_{i=1}^k \sigma_i(\hat{\A}).
%\]

\begin{theorem} \label{thm:d-b-sv}  Given two   matrices $\A \in \RB^{m\times n}$ and ${\B} \in \RB^{m\times n}$,
let  $s_i(\A-\B) = |\sigma_i(\A) - \sigma_i(\B) |$ for $i\in [p]$ where $p=\min\{m, n\}$. Then
\[
\sum_{i=1}^k s_i^{\downarrow} (\A-\B) \leq \sum_{i=1}^k \sigma_i(\A - \B) \; \mbox{ for }  k=1, \ldots, p.
\]
\end{theorem}

\begin{proof} Consider the following two $(m+n)\times (m+n)$ symmetric matrices:
\[
\tilde{\A}= \begin{bmatrix} \0 & \A \\ \A^T & \0 \end{bmatrix} \; \mbox{ and } \;  \tilde{\B}= \begin{bmatrix} \0 & \B \\ \B^T & \0 \end{bmatrix}.
\]
By Theorem~\ref{thm:h-a}, the eigenvalues of $\tilde{\A}$ are $\pm \sigma_1(\A), \ldots, \pm \sigma_p(\A)$, together with $m+n-2 p$ zeros; and similarly for $\tilde{\B}$ as well as for  $\tilde{\A}-\tilde{\B}$. Thus, the $p$ largest entries of  $\lamb(\tilde{\A}- \tilde{\B})$ are $\sigma_1(\A-\B), \ldots, \sigma_p(\A-\B) $. Note that both $\sigma_i(\A)- \sigma_i(\B)$  and $\sigma_i(\B)-\sigma_i(\A)$ are the entries of  $\lamb(\tilde{\A})- \lamb(\tilde{\B})$, so the $p$ largest entries of  $\lamb(\tilde{\A})- \lamb(\tilde{\B})$ comprise the set  $\{s_1(\A-\B), \ldots, s_p(\A-\B)\}$. Proposition~\ref{pro:031} shows that $\lamb(\tilde{\A}-\tilde{\B}) \prec \lamb(\tilde{\A})-\lamb(\tilde{\B})$. This implies the result of the theorem.
\end{proof}

\begin{theorem} \label{thm:product} Let $\A \in \RB^{m\times n}$ and $\B \in \RB^{n\times p}$ be given, and let $q=\min\{m, n, p\}$. Then
for $k=1, \ldots, q$,
\[
\prod_{i=1}^k \sigma_i(\A \B) \leq \prod_{i=1}^k \sigma_i(\A) \sigma_i(\B).
\]
If $n=p=m$, then equality holds for $k=n$. And
\[
\sum_{i=1}^k \sigma_i(\A \B) \leq \sum_{i=1}^k \sigma_i(\A) \sigma_i(\B) \leq \Big(\sum_{i=1}^k \sigma_i(\A)\Big) \Big(\sum_{i=1}^k \sigma_i(\B) \Big).
\]
\end{theorem}

\begin{proof} Let $\A\B = \U \Si \V^T$ be a full SVD of $\A \B$, and for $k\leq q$ let $\U_k$ and $\V_k$ be the first $k$ columns of $\U$ and
$\V$, respectively. Now take a polar decomposition of $\B \V_k$ as $\B \V_k=\Q \S$. Since $\S^2 = \V_k^T \B^T \B \V_k$ and by Proposition~\ref{pro:032}-(4), we obtain
\[
\det(\S^2) = \det(\V_k^T \B^T \B \V_k ) \leq \prod_{i=1}^k \sigma_i^2(\B)
\]
We  further have that
\begin{align*}
\prod_{i=1}^k \sigma_i(\A \B) & = |\det(\U_k^T \A \B \V_k) | = |\det(\U_k^T \A \Q)  \det(\S)|  \\
& \leq   \prod_{i=1}^k \sigma_i(\A) \sigma_i(\B).
\end{align*}
The above inequality again follows from Proposition~\ref{pro:032}-(4), When $n=p=m$, then
\[
\prod_{i=1}^n \sigma_i(\A \B) = |\det(\A \B)| = |\det(\A)| \times |\det(\B)| =  \prod_{i=1}^n \sigma_i(\A) \sigma_i(\B).
\]
The second part follows from the first part and Lemma~\ref{lem:shur}.
\end{proof}

\section{Appendix: Application of Matrix Differentials}

Here we present alternative proofs for Theorem~\ref{thm:von} and  Theorem~\ref{thm:kyfan}, which are based on matrix differentials.  It aims at further illustrating how to use matrix differentials.

\begin{proof}[The Second Proof of Theorem~\ref{thm:von}]
To solve the problem,  we define the  Lagrangian function:
\[
L(\Q, \C) = \tr(\Q^T \M \Q) - \tr(\C(\Q^T \Q - \I_k)),
\]
where $\C$ is a $k\times k$ symmetric matrix of Lagrangian multipliers.   Since
\[
d L = \tr(d \Q^T \M \Q + \Q^T \M d \Q) - \tr(\C (d \Q^T \Q + \Q^T d \Q)),
\]
this shows that $\frac{d L}{d \Q} = 2 \M \Q - 2 \Q \C$. The KKT condition is now
\[
\M \Q - \Q \C= \0.
\]
Clearly, if  $\hat{\C} \triangleq \diag(\lambda_1, \ldots, \lambda_k)$ and $\hat{\Q}$ consists of the corresponding orthonormal eigenvectors,
they are a solution of the above equation. In this setting, we see that $\tr(\hat{\Q}^T \M \hat{\Q}) = \sum_{i=1}^k \lambda_i$.

Thus, we only need to prove that $\hat{\Q}$ is indeed the maximizer of the original problem. We now compute the Hessian matrix of
$L$ w.r.t.\ $\Q$ at $\Q = \hat{\Q}$ and $\C = \hat{\C}$. Since $\vect(\M \Q - \Q \C) = (\I_k \otimes \M - \C \otimes \I_n) \vect(\Q)$,
the Hessian matrix is given as
\[
 \H = 2  (\I_k \otimes \M - \hat{\C} \otimes \I_n).
\]
For any $\X \in \RB^{n \times k}$ such that $\X^T \hat{\Q} =\0$,
it  suffices for our purpose to prove $\x^T \H \x/2\leq 0$ where $\x = \vect(\X)$. Take the full EVD of $\M$ as $\M = \U \Lam \U^T$, where $\Lam = \diag(\lambda_1, \ldots, \lambda_n)$ and $\U =[\hat{\Q}, \hat{\Q}^{\bot}]$ such that $\U^T \U = \I_n$. Denote $\Lam_2=\diag(\lambda_{k+1}, \ldots, \lambda_n)$ and $\Y = (\hat{\Q}^{\bot})^T \X  =[\y_1, \ldots, \y_k]$. 
Then,
\begin{align*}
\frac{1}{2} \x^T \H \x & = \tr(\X^T \M \X) - \tr(\X \hat{\C} \X^T) \\
& = \tr(\X^T \hat{\Q}^{\bot} \Lam_2 (\hat{\Q}^{\bot})^T \X) - \tr (\hat{\C} \X^T (\hat{\Q} \hat{\Q}^T + \hat{\Q}^{\bot} (\hat{\Q}^{\bot})^T) \X) \\
& = \tr (\Y^T \Lam_2 \Y) - \tr(\hat{\C} \Y^T \Y) \\
& = \sum_{i=1}^k \y_i^T \Lam_2 \y_i - \sum_{i=1}^k \lambda_i \y_i^T \y_i \\
& = \sum_{i=1}^k \y_i^T (\Lam_2 - \lambda_i \I_{n-k}) \y_i \leq 0.
\end{align*}
\end{proof}

\begin{proof}[The Third Proof of Theorem~\ref{thm:kyfan}]
To solve the constrained problem in the theorem,  we now define the  Lagrangian function:
\[
L(\X, \Y, \C_1,  \C_2) =  \tr(\X^T \A \Y) -  \frac{1}{2} \tr(\C_1(\X^T \X - \I_k)) -  \frac{1}{2} \tr(\C_2(\Y^T \Y - \I_k)) ,
\]
where $\C_1$ and $\C_2$ are  two $k\times k$ symmetric matrix of Lagrange multipliers.   Since
\begin{align*}
d L & = \tr(d \X^T \A \Y) -  \frac{1}{2} \tr(\C_1 (d \X^T \X + \X^T d \X)), \\
d L & =  \tr( \X^T \A d\Y) -  \frac{1}{2} \tr(\C_2 (d \Y^T \Y + \Y^T d \Y)),
\end{align*}
which yield that $\frac{d L}{d \X} =   \A \Y -  \X \C_1$  and $\frac{d L}{d \Y} =   \X \A^T -  \Y  \C_2$. The KKT condition is now
\[
\A \Y - \X \C_1= \0 \; \mbox{ and } \;  \A^T \X - \Y \C_2=\0.
\]
It then follows from $\X^T \X = \I_k$ and $\Y^T \Y = \I_k$ that $\C_1= \C_2$. We denote   $\C \triangleq \C_1 = \C_2$. So,
\begin{align*}
\A \Y - \X \C &= \0, \\
\A^T \X - \Y \C & =\0.
\end{align*}
That is,
\[
\begin{bmatrix} \0 & \A \\ \A^T & \0 \end{bmatrix} \begin{bmatrix} \X \\ \Y \end{bmatrix} = \begin{bmatrix} \X \\ \Y \end{bmatrix} \C.
\]
Clearly, if  $\hat{\C} \triangleq \Si_k= \diag(\lambda_1, \ldots, \lambda_k)$,  $\hat{\X}\triangleq \U_k= [\u_1, \ldots, \u_k]$, and $\hat{\Y}\triangleq \V_k = [\v_1, \ldots, \v_k]$, then
they are a solution of the above equation. In this setting, we see that $\tr(\hat{\X}^T \A \hat{\Y}) = \sum_{i=1}^k \sigma_i$.

Thus, we only need to prove that $(\hat{\X}, \hat{\Y})$ is  the maximizer of the original problem. We now compute the Hessian matrix of
$L$ w.r.t.\ $(\X, \Y)$ at $(\X, \Y) = (\hat{\X}, \hat{\Y})$,  and $\C = \hat{\C}$.
The Hessian matrix is given as
\[
\H \triangleq  \begin{bmatrix} \frac{\partial^2L} {\partial \vect(\hat{\X}) \partial \vect(\hat{\X})^T}  & \frac{\partial^2L}{\partial \vect(\hat{\X}) \partial \vect(\hat{\Y})^T}   \\   \frac{\partial^2 L}{\partial \vect(\hat{\Y}) \partial \vect(\hat{\X})^T}  & \frac{\partial^2L} {\partial \vect(\hat{\Y}) \partial \vect(\hat{\Y})^T}      \end{bmatrix} =   \begin{bmatrix}  - \Si_k \otimes \I_m & \I_k \otimes \A   \\   \I_k \otimes \A^T  & -\Si_k \otimes \I_n      \end{bmatrix},
\]
because
 $\vect(\A \Y - \X \C) = (\I_k \otimes \A)\vect(\Y)  -( \C^T \otimes \I_m) \vect(\X)$ and  $\vect(\A^T \X - \Y \C) = (\I_k \otimes \A^T)\vect(\X)  -( \C^T \otimes \I_n) \vect(\Y)$.

Note that
\[
\begin{bmatrix} \hat{\X}^T & \0 \\ \0 & \hat{\Y}^T \end{bmatrix}  \begin{bmatrix} \hat{\X} & \0 \\ \0 & \hat{\Y}  \end{bmatrix} = \I_{2k}.
\]
Thus,
for any $\Z_1  \in \RB^{m \times k}$ and $\Z_2 \in \RB^{n\times k}$ such that $\Z_1^T \hat{\X}  = \0$ and $\Z_2^T \hat{\Y}  = \0$,
it  suffices for our purpose to prove $\z^T \H \z \leq 0$ where $\z^T = (\vect(\Z_1)^T,  \vect(\Z_2)^T)$.  Compute
\begin{align*}
\z^T \H \z & = [\vect(\Z_1)^T,  \vect(\Z_2)^T ]   \begin{bmatrix}  - \Si_k \otimes \I_m & \I_k \otimes \A   \\   \I_k \otimes \A^T  & -\Si_k \otimes \I_n      \end{bmatrix} \begin{bmatrix} \vect(\Z_1) \\ \vect(\Z_2) \end{bmatrix} \\
& = \vect(\Z_2)^T(\I_k{\otimes} \A^T) \vect(\Z_1) + \vect(\Z_1)^T(\I_k{\otimes} \A) \vect(\Z_2)  \\
& \quad -   \vect(\Z_1)^T(\Si_k{\otimes} \I_m) \vect(\Z_1)- \vect(\Z_2)^T(\Si_k{\otimes} \I_n) \vect(\Z_2) \\
& = - \tr(\Z_1^T \Z_1 \Si_k) - \tr(\Z_2^T \Z_2 \Si_k) + 2 \tr(\Z_1^T \A \Z_2)  \triangleq \Delta.
\end{align*}

Take a thin SVD of $\A$ as $\A = \U \Si \V^T$, where $\Si  =\Si_k \oplus \Si_{-k} $,  $\U =[\U_k, {\U}_{-k}]$, and $\V = [\V_k, \V_{-k}]$. Denote  $\R_1 = \U_{-k}^T \Z_1 $ and and $\R_2 = \V_{-k}^T \Z_2 $.   Then $\tr(\Z_1^T \A \Z_2) = \tr(\Z_1^T \U_{-k} \Si_{-k} \V_{-k}^T)$. And hence,
\begin{align*}
- \Delta  & = \tr(\Z_1^T \Si_k \Z_1) + \tr(\Z_2^T \Si_k \Z_2) - 2 \tr(\Z_1^T \U_{-k} \Si_{-k} \V_{-k}^T \Z_2) \\
& \geq  \tr(\Z_1^T \U_{-k}  \U_{-k}^T \Z_1  \Si_k) +  \tr(\Z_2^T \U_{-k}  \U_{-k}^T \Z_2  \Si_k)  \\
& \quad  - 2  \tr(\Z_1^T \U_{-k} \Si_{-k} \V_{-k}^T \Z_2)  \\
& = \tr (\R_1^T \R_1 \Si_k) + \tr(\R_2^T \R_2 \Si_k) -2  \tr(\R_1^T \Si_{-k}  \R_2) \\
& \geq  \tr(\R_1^T \Si_{-k}  \R_1) + \tr(\R_2^T \Si_{-k}  \R_2) - 2  \tr(\R_1^T \Si_{-k}  \R_2)   \\
& = \tr[ (\R_1 -\R_2)^T \Si_{-k} (\R_1 - \R_2) ] \geq 0.
\end{align*}
The last inequality uses the fact that $\tr (\R_1^T \R_1 \Si_k) \geq \tr(\R_1^T \Si_{-k} \R_1)$ and $\tr (\R_2^T \R_2 \Si_k) \geq \tr(\R_2^T \Si_{-k} \R_2)$.
\end{proof}

%%%%%%%%%%%%%%%%%%%%%%%%%%%%%%%%%%%%%%%
%%%%%%%%%%%%%%%%%%%%%%%%%%%%%%%%%%%%%%%
\chapter{Unitarily Invariant Norms}
\label{ch:uinorm}

In this chapter we   study  unitarily invariant norms of a matrix, which can be defined via  singular values of  the matrix.
Unitarily invariant norms were contributed by  J. von Neumann, Robert Schatten, and Ky Fan. J.  von Neumann established an equivalent relationship between unitarily invariant norms and symmetric gauge functions.
There are two popular classes of unitarily invariant norms: the Ky Fan norms and  Schatten $p$-norms.

Parallel with the vector $p$-norms, the Schatten $p$-norms are defined on singular values of a matrix.
Their special cases include the spectral norm, Frobenius norm, and nuclear norm. They have wide applications in modern data analysis and computation. For example, the Frobenius norm is used to measure approximation errors in regression and reconstruction problems because it  essentially equivalent to the $\ell_2$-norm of a vector.
The spectral norm is typically used to describe convergence and convergence rate of an iteration procedure.
The nuclear norm provides an effective approach to matrix low rank modeling.

We first briefly review  matrix norms, and then present the notion of symmetric gauge functions.
%which help to establish the connection of  unitarily invariant norms with singular value decomposition.
Symmetric gauge functions  facilitate us to study unitarily invariant norms. First, it  transforms a unitarily invariant norm on matrices to a norm on vectors equivalently. Second, it can incorporate majorization theory.
Accordingly, we give some important properties of unitarily invariant norms.

\section{Matrix Norms}

A function $f: \RB^{m\times n} \to \RB$ is said to be a matrix norm if the following conditions are satisfied:
\begin{enumerate}
\item[(1)] $f(\A )>0$ for all nonzero matrix $\A \in \RB^{m\times n}$;
\item[(2)] $f(\alpha \A) = |\alpha| f(\A)$ for any  $\alpha \in \RB$ and any  $\A \in \RB^{m\times n}$;
\item[(3)] $f(\A + \B) \leq f(\A) + f(\B)$ for any $\A$ and $\B \in \RB^{m\times n}$.
\end{enumerate}
We  denote the norm of a matrix $\A$ by $\|\A\|$. Furthermore, if
\begin{enumerate}
\item[(4)] $\|\A \B \|\leq \|\A\| \|\B\|$
where $\A \in \RB^{m\times n}$ and $\B \in \RB^{n\times p}$,
\end{enumerate}
the matrix norm is said to be consistent.  In some literature, when one refers to a matrix norm on $\R^{n\times n}$. it is required to be consistent. Here we do not make this requirement.

There is an equivalence between any two norms. Let $\|\cdot\|_{\alpha}$ and $\|\cdot\|_{\beta}$ be two norms on $\RB^{m\times n}$.
Then there exist positive numbers $\alpha_1$ and $\alpha_2$ such that for all $\A \in \RB^{m\times n}$,
\[
\alpha_1 \|\A\|_{\alpha}  \leq \|\A \|_{\beta} \leq \alpha_2 \|\A\|_{\alpha}.
\]
Conditions~(2) and (3) tell us that the norm is convex. Moreover, it is continuous because
\[
|\|\A\| - \|\B\| | \leq \|\A - \B\| \leq \alpha \|\A - \B\|_F, \; \mbox{ where } \alpha>0.
\]
A norm always companies with its dual.  The dual is a norm. Moreover,  the dual of the dual  norm is the original norm.

\begin{definition} \label{def:dual} Let $\|\cdot\|$ be a given norm on $\RB^{m\times n}$. Its dual  (denoted $\|\cdot\|^*$)
is defined as
\[
\|\A\|^* = \max \big\{\tr(\A \B^T):  \;  \B\in \RB^{m\times n},  \|\B\| =1 \big\}.
\]
\end{definition}

\begin{proposition} \label{pro:dualnorm}  The dual  $\|\cdot \|^*$ has the following properties:
\begin{enumerate}
\item[(1)]
The dual is a norm.
\item[(2)]  $(\|\A\|^*)^* = \|\A\|$.
\item[(3)] $\tr(\A \B^T)  \leq|\tr(\A^T \B)| \leq \|\A\| \|\B\|^*$ (or  $\|\A\|^*  \|\B\|$).
\end{enumerate}
\end{proposition}
%the dual of dual is the original norm.

There are two approaches for definition of a matrix norm. In the first approach, the  norm of  matrix $\A$ is defined via its vectorization $\vect(\A)$;
that is, $\|\A\|=\|\vect(\A)\|$, which obviously satisfies  Conditions~(1)-(3). We  refer to this class of the matrix norms
as \emph{matrix vectorization norms} for ease of exposition. Note that the Frobenius norm is a matrix vectorization norm because $\|\A\|_F= \|\vect(\A)\|_2$.
However, this class of matrix norms are not always consistent.
For example, let
\[
\A =\B = \begin{bmatrix} 1 & 1 \\ 1& 1 \end{bmatrix}.
\]
Since $\A \B =\begin{bmatrix} 2 & 2 \\ 2 & 2 \end{bmatrix}$ and
\[2=\|\vect(\A \B)\|_{\infty} >\|\vect(\A)\|_{\infty} \|\vect(\B)\|_{\infty} =1, \]
this implies that the corresponding matrix norm is not  consistent. 

In the second approach, the matrix norm is defined by
\[
\|\A \| = \max_{\|\x \|=1} \|\A \x\|,
\]
which is also called the \emph{induced} or \emph{operator} norm.
\begin{theorem} \label{thm:operator} The operator  norm on $\RB^{m\times n}$ is a consistent matrix norm.
\end{theorem}
\begin{proof} Given a  matrix $\A \in \RB^{m\times n}$,   the result is trivial If $\A=\0$.  Assume that $\A\neq \0$. Then there exists a
nonzero vector $\z \in \RB^n$ for which $\A \z \neq \0$. So we have $\|\A \z\| >0$ and $\|\z\|>0$. Hence,
\[
\|\A\| = \max_{\x \neq \0}   \frac{\|\A \x\|}{\|\x\|} \geq \frac{\|\A \z\|}{\|\z\|} >0.
\]
Conditions (2)-(3)  are directly obtained from the definition of  the vector norm. As for Condition~(4), it can be established by
\[
\|\A \B  \x\| \leq \|\A \| \|\B \x \| \leq \|\A \| \|\B \| \|\x\|
\]
for any $\x \neq \0$. Thus,
\[
\|\A \B \|  = \max_{\x \neq \0} \frac{ \|\A \B  \x\| }{\|\x \| } \leq \|\A\| \|\B\|.
\]
% $\|\alpha \A\| = \max_{\x \neq \0}   \frac{\|\alpha \A \x\|}{\|\x\|} \geq \frac{\|\A \z\|}{\|\z\|}$
\end{proof}
As we have shown, $\|\A\|_2 = \max_{\|\x \|_2=1} \|\A \x\|_2 = \sigma_1(\A)$.  It is thus called the spectral norm.

Note that $\|\U \A \V\|_2=\|\A\|_2$ and $\|\U \A \V\|_F = \|\A\|_F$ for any $m\times m$ orthonormal matrix $\U$ and
any $n \times n$ orthonormal matrix $\V$. In other words, they are unitarily  invariant.

\begin{definition} A matrix norm  is said to be \emph{unitarily  invariant} if $\|\U \A \V\|=\| \A \|$ for any unitary matrices $\U$ and $\V$.
\end{definition}

In this tutorial, we only consider real matrices. Thus, a unitarily invariant norm should be termed as ``orthogonally invariant norm.''
However, we still follow the term of the unitarily invariant norm and denote it by $\nm \cdot \nm$.

\begin{theorem} \label{thm:dual} Let $\|\cdot\|$ be a given norm on $\RB^{m\times n}$. Then it is  unitarily invariant if and only if its dual is unitarily invariant.
\end{theorem}
\begin{proof} Suppose $\|\cdot\|$ is unitarily invariant, and let $\U \in \RB^{m\times m}$ and $\V \in \RB^{n\times n}$ be orthonormal. Then
\begin{align*}
\|\U \A \V \|^*  & = \max \big\{ \tr(\U \A \V \B^T):  \;\B \in \RB^{m\times n},  \|\B\| =1   \big \} \\
& = \max \big\{ \tr(\A (\U^T \B \V^T)^T):  \;\B \in \RB^{m\times n},  \|\B\| =1   \big \} \\
& =  \max \big\{ \tr(\A \C^T):  \; \C \in \RB^{m\times n},  \|\U \C \V \| =1   \big \}   \\
& =  \max \big\{ \tr(\A \C^T):  \; \C \in \RB^{m\times n},  \|\C  \| =1   \big \} = \| \A\|^*.
\end{align*}
The converse follows from the fact that $(\|\A\|^*)^* = \|\A\|$.
\end{proof}

We find that $\|\A\|_2 = \|\sib(\A) \|_{\infty}$ and $\|\A\|_F = \|\sib(\A) \|_2$; that is, they correspond the norms on the vector $\sib(\A)$ of the singular values of $\A$. This sheds light on the relationship of  a unitarily invariant norm  of a matrix with its singular values.

\section{Symmetric Gauge Functions}

In order to investigate the  unitarily  invariant norm, we first present the notion of symmetric gauge functions.

\begin{definition} \label{def:sgf} A real function $\phi: \RB^n \to \RB$ is called a symmetric gauge function if it satisfies the following four conditions:
\begin{enumerate}
\item[(1)] $\phi(\u) >0$ for all nonzero $\u \in \RB^n$.
\item[(2)] $\phi(\alpha \u) = |\alpha| \phi(\u)$ for any constant $\alpha \in \RB$.
\item[(3)] $\phi(\u + \v)\leq \phi(\u) + \phi(\v)$ for all $\u, \v \in \RB^n$.
\item[(4)] $\phi( \D \u_{\pi}) = \phi(\u)$ where $\u_{\pi} = (u_{\pi_1}, \ldots, u_{\pi_n})$ with $\pi$ as a permutation of $[n]$ and $\D$ is an $n\times n$ diagonal matrix with $\pm 1$ diagonal elements.
\end{enumerate}
Furthermore, the gauge  function is called normalized if it satisfies the condition:
\begin{enumerate}
\item[(5)] $\phi(1, 0, \ldots, 0)=1$.
\end{enumerate}
\end{definition}

Conditions (1)-(3) show that that the gauge function is a  vector norm. Thus, it is convex and continuous.  Condition~(4) says that the gauge function is symmetric.

\begin{lemma} \citep{Schatten}  \label{lem:schatten} Let $\u, \v \in \RB^n$. If  $|\u|  \leq |\v|$,  then   $\phi(\u) \leq \phi(\v)$  for every symmetric gauge function $\phi$.
\end{lemma}
\begin{proof} In terms of Condition~(4), we can directly assume that $\u\geq \0$ and $\v\geq \0$.   Currently,
the argument is equivalent to
\[
\phi(\omega_1 v_1, \ldots, \omega_n v_n) \leq \phi(v_1, \ldots, v_n)
\]
for $\omega_i \in [0, 1]$.  Thus,
by induction, it suffices to prove
\[
\phi(v_1, \ldots, v_{n{-}1}, \omega v_n) \leq \phi(v_1, \ldots, v_n)
\]
where  $\omega \in [0, 1]$ for every symmetric gauge function $\phi$.
It follows from the following direct computation:
\begin{align*}
& \phi(v_1, \ldots, v_{n-1}, \omega v_n)  \\
 & =  \phi\Big(\frac{1{+}\omega}{2} v_1 {+} \frac{1{-}\omega}{2} v_1, \ldots,  \frac{1{+}\omega}{2} v_{n{-}1} {+} \frac{1 {-} \omega}{2} v_{n{-}1},  \frac{1{+}\omega}{2} v_n - \frac{1{-}\omega}{2} v_n    \Big) \\
& \leq \frac{1+\omega}{2} \phi(v_1, \ldots,  v_{n-1},  v_n) + \frac{1- \omega}{2} \phi(v_1, \ldots,  v_{n-1},  -v_n) \\
& = \phi(v_1, \ldots, v_{n-1},  v_n).
\end{align*}
\end{proof}

\begin{theorem} \citep{Fan:1951}   \label{thm:sg-m} Given two nonnegative vectors $\u, \v \in \RB_{+}^n$, then  $\u  \prec_w  \v$ if and only if $\phi(\u) \leq \phi(\v)$  for every symmetric gauge function $\phi$.
\end{theorem}

\begin{proof} The necessity is obtained  by setting a set of special symmetric gauge functions $\phi_k$ for $k\in [n]$. Specifically, they are defined as
\[
\phi_k(\x) = \max_{1\leq i_1 \leq \cdots \leq i_k \leq n} \;  \sum_{l=1}^k |x_{i_l}|.
\]
where $\x = (x_1, \ldots, x_n)$.

It remains to prove the sufficiency. Without loss of generality, we assume that $u_1\geq \cdots \geq u_n$ and $v_1\geq \cdots \geq v_n$.  Let $\z=(z_1, \ldots, z_n)^T$ where $z_i=v_i$ for $i \in [n-1]$ and $z_n= v_n - \sum_{i=1}^{n}(v_i-u_i)$. Obviously, $\z \leq \v$. And it  follows from $\u \prec_w \v$ that  $\u \prec  \z$. In terms of the theorem of Hardy, Littlewood, and P\'{o}lya (see Lemma~\ref{lem:001}), there exists a doubly stochastic matrix (say $\W$) such that $\u = \W \z$. Since $\W(\v- \z)\geq \0$, we have $\u \leq \W \v$. Thus, by Lemma~\ref{lem:schatten}, $\phi(\u) \leq \phi(\W \v)$
for every symmetric gauge function. Consider that a doubly stochastic matrix can be expressed a convex combination of a set of permutation matrices (see Lemma~\ref{lem:dsm}). We write $\W = \sum_{j=1} \alpha_j \PP_j$ where $\alpha_j\geq 0$ and $\sum_{j} \alpha_j =1$, and the $\PP_j$ are permutation matrices. Accordingly,
\[
\phi(\u) \leq \phi(\sum_{j} \alpha_j \PP_j \v ) \leq \sum_{j} \alpha_j \phi(\PP_j \v) = \sum_{j} \alpha_j \phi(\v) = \phi(\v).
\]
\end{proof}

It is worth noting that the proof of Theorem~\ref{thm:sg-m} implies that if $\phi_k(\u) \leq \phi_k (\v)$ for $k \in [n]$, then $\phi(\u) \leq \phi(\v)$ for every symmetric gauge function $\phi$. In other words, an infinite family of norm inequalities follows from a finite one.

\begin{definition} \label{def:polar} The dual of a symmetric gauge function $\phi$ on $\RB^n$ is defined as
\[
\phi^{*}(\u) \triangleq \max  \big\{\u^T \v: \v \in \RB^n, \phi(\v)=1 \big\}.
\]
\end{definition}

\begin{proposition} \label{pro:polar} Let $\phi^*$ be the dual of the symmetric gauge function $\phi$. Then $\phi^*$ is also a symmetric gauge function. Moreover, $(\phi^*)^* = \phi$.
\end{proposition}
\begin{proof}  For a  nonzero vector $\u \in \RB^n$,  then $\phi(\u)>0$. Hence, 
\[
\max_{\phi(\v)=1}  \u^T \v \geq \frac{\u^T \u}{\phi(\u)} >0.
\]
It is also seen that
\[
\phi^*( \u {+} \v) = \max_{\phi (\z) =1 } (\u {+} \v)^T \z \leq  \max_{\phi (\z) =1 } \u^T \z + \max_{\phi (\z) =1 }  \v^T \z \leq \phi^*(\u) + \phi^*(\v).
\]
As for the symmetry of $\phi^*$ can be directly obtained from that of $\phi$. Finally,  note that $\phi^*$ is a norm on $\RB^n$. Thus,
$(\phi^*)^* = \phi$.
\end{proof}

\section{Unitarily Invariant Norms via SGFs}

There is a one-to-one correspondence between a unitarily invariant norm and a symmetric gauge function (SGF).
\begin{theorem} \label{thm:ui-sg}
If $\nm \cdot\nm$ is a given unitarily invariant norm on $\RB^{m\times n}$, then there is a symmetric gauge function $\phi$ on $\RB^{q}$ where $q=\min\{m,  n\}$ such that $\nm\A \nm=\phi(\sib(\A))$ for all $\A \in \RB^{m\times n}$.

Conversely, if $\phi$ is a symmetric gauge function on $\RB^q$, then ${\nm} \A \nm \triangleq \phi(\sib(\A))$ is a unitarily invariant norm on $\RB^{m\times n}$.
\end{theorem}
\begin{proof} Given a unitarily invariant norm $\nm \cdot \nm$ on $\RB^{m\times n}$  and a vector $\x \in \RB^{q}$, define $\phi(\x) \triangleq \nm \X \nm$ where $\X=[x_{ij}] \in \RB^{m\times n}$ satisfying that $x_{ii}= x_i$ for $i \in [q]$ and all other elements are zero. That $\phi$ is a norm on $\RB^q$ follows from the fact that $\nm \cdot \nm$ is a norm. The unitary invariance of $\nm \cdot \nm$ then implies that $\phi$ satisfies
the symmetry. Now let $\A= \U \Si \V^T$ be the full SVD of $\A$. Then $\nm \A \nm = \nm \U \Si \V^T \nm = \nm \Si \nm = \phi(\sib(\A))$.

Conversely, if $\phi$ is a symmetric gauge function, for any $\A \in \RB^{m\times n}$ define $\nm \A \nm = \phi(\sib(\A))$.
We now prove that $\nm \cdot \nm$ is a unitarily invariant norm. First,  that $\nm \A \nm>0$ for $\A\neq \0$ and $\nm \alpha \A \nm = |\alpha| \nm \A \nm$ for any constant $\alpha$ follows the fact that $\phi$ is a norm. The unitary invariance of $\nm \cdot \nm$ follows from that for any orthonormal matrices $\U$ ($m\times m$) and $\V$ ($n\times n$),  $\U \A \V$ and $\A$ have the same singular values. Finally,
\begin{align*}
\nm \A + \B\nm  &= \phi(\sib(\A+ \B)) \leq \phi(\sib(\A) + \sib(\B))  \\
& \leq \phi(\sib(\A)) + \phi(\sib(\B)) \\
& = \nm \A \nm + \nm \B\nm.
\end{align*}
Here the first inequality follows Proposition~\ref{pro:032} and Theorem~\ref{thm:sg-m}.
\end{proof}

The following theorem implies that there is also a one-one correspondence  between
the dual of  a symmetric gauge function and a dual unitarily invariant norm.

\begin{theorem} \label{thm:dual-pri} Let  $\phi^*$ be the dual of  symmetric gauge function $\phi$. Then $\nm \A\nm = \phi(\sib(\A))$ if and only if $\nm \A \nm^* = \phi^*(\sib(\A))$.
\end{theorem}
\begin{proof}  Assume that  $\nm \A\nm = \phi(\sib(\A))$.   Then
\begin{align*}
\nm \A \nm^* &= \max \big\{\tr(\A^T \B): \B \in \RB^{m\times n}, \nm \B \nm =1 \big \} \\
                       & = \max  \big \{ \tr(\Si_A^T \U_A^T \B \V_A):   \phi(\sib(\B))  =1  \big\},
\end{align*}
where $\A  = \U_A  \Si_A  \V_A^T$ is a full SVD of $\A$.  By Theorem~\ref{thm:kyfan0},  we have
\[
 \tr(\V_A^T \B^T \U_A  \Si_A ) \leq \max_{\U^T \U=\I_m,  \V^T \V=\I_n} \tr(\V^T \B^T \U \Si_A) = \sum_{i=1}^{q} \sigma_i(\A) \sigma_i(\B).
 \]
When letting $\B  = \U_A  \Si_B \V_A^T$ as a full SVD of $\B$, we can obtain that
\[
\nm \A \nm^* = \max  \big \{ \tr(\Si_A^T  \Si_B ),  \phi(\sib(\B))  =1  \big\} = \phi^*(\sib(\A)).
\]

Conversely, the result follows from the fact that $(\phi^*)^* =\phi$.
\end{proof}

Given a matrix $\A \in \RB^{m\times n}$, let it have a full SVD:  $\A = \U \Si \V^T$. Then $\nm \A \nm = \nm \Si \nm$.
As we have seen, for $\x \in \RB^n$ the function
\[
\phi(\x) \triangleq \max_{1 \leq i_1 \leq \cdots \leq i_k\leq n} \sum_{l=1}^k |x_{i_l}|
\]
is a symmetric gauge function. Thus, $\sum_{i=1}^k \sigma_i(\A)$ defines also a class of unitarily invariant norms which are the so-called \emph{Ky Fan $k$-norms}.

Clearly, the vector $p$-norm $\|\cdot\|_p$ for $p\geq 1$ is a symmetric gauge function. Thus,
Theorem~\ref{thm:ui-sg} shows that $\nm \A \nm_p \triangleq \|\sib(\A)\|_p$ for $p\geq 1$ are a class of unitarily invariant norms. They are well known  as the \emph{Schatten $p$-norms}.
Thus, $\|\A\|_F  = \| \sib(\A)\|_2 =\nm \A \nm_2$ and $\| \A \|_2  = \|\sib(\A)\|_{\infty} =\nm\A\nm_{\infty}$.

When $p=1$, $\|\A\|_{*} \triangleq \nm \A \nm_1 =  \|\sib(\A)\|_1 = \sum_{i=1}^{\min\{m, n\}} \sigma_i(\A)$ is called the \emph{nuclear norm} or \emph{trace norm}, which has been widely used in many machine learning problems such as matrix completion, matrix data classification, multi-task learning, etc. ~\citep{srebro2004maximum,cai2010singular,mazumder2010spectral,liu2009tensor,luoluo:2015,kang2011learning,pong2010trace,zhou2013regularized}.
Parallel with the $\ell_1$-norm which is used as convex relaxation of the $\ell_0$-norm~\citep{TibshiraniLASSO:1996},
the nuclear norm is a convex alternative of the matrix rank.
Since the nuclear norm is the  best convex approximation of the matrix rank over the unit  ball of matrices, this makes it more tractable to solve the resulting optimization problem (see Example~\ref{exm:svt} below).

\section{Properties of Unitarily Invariant Norms}

Theorem~\ref{thm:ui-sg} opens an approach for exploring unitarily invariant norms by
using symmetric gauge functions and majorization theory.  We will see that this makes things  more tractable.

\begin{theorem} \label{thm:uin-consist} Let $\nm \cdot \nm$ be a unitarily invariant norm on $\RB^{n\times n}$. Then it is consistent.
\end{theorem}
Theorem~\ref{thm:uin-consist} follows immediately  from Theorem~\ref{thm:product}. However, when the norm is defined on $\RB^{m\times n}$,
Theorem~\ref{thm:product} can not help to establish the consistency of the corresponding unitarily invariant norm.

As an immediate corollary of Theorem~\ref{thm:ui-sg}, we have the following result, which shows that unitarily invariant norms are monotone.

\begin{theorem} Let $\nm \cdot \nm$ be a given unitarily invariant norm on $\RB^{m\times n}$. Then $\nm \A\nm \leq \nm \B \nm$ if and only if $\sib(\A) \prec_w \sib(\B)$.
\end{theorem}

\begin{proposition} Given a matrix $\A\in \RB^{m\times n}$, let $[\A]_r$  be obtained by replacing the last $r$ rows and $r$ columns of $\A$ with zeros, and $\langle \A\rangle_r$  by replacing the last $r$ rows or columns of $\A$ with zeros. Let  $q=\min\{m, n\}$.
%\[ \Si_0 = \diag(\sigma_{r+1}(\A), \ldots, \sigma_p(\A), 0, \ldots, 0) \]
%be be an $m\times n$ diagonal matrix  where $p=\min\{m, n\}$.
Then for any $r \in [q]$,
\[
\nm [\A]_r \nm \leq \nm \langle\A\rangle_r \nm \leq \nm \A\nm.
\]
%\begin{enumerate}
%\item[(1)] $\nm [\A]_r \nm \leq \nm \langle\A\rangle_r \nm \leq \nm \A\nm$.
%\item[(2)] $\nm  \Si_0  \nm \leq  \nm \langle \A \Q\rangle_r \nm$
%if  $\langle \A \Q\rangle_r$ is obtained by replacing the last $r$  columns of $\A \Q$ with zeros; or
%$\nm  \Si_0  \nm \leq  \nm \langle \PP \A \rangle_r \nm$
%if $\langle \PP \A \Q\rangle_r$ is obtained by replacing the last $r$  rowss of $\PP \A$ with zeros.
%Here  $\Q \in \RB^{n\times n}$ and $\PP \in \RB^{m\times m}$ are any  orthonormal matrices.
%%\item[(3)] $\nm \PP^T \A \Q \nm \leq \nm \A \nm$ for any column orthonormal matrices $\PP \in \RB^{m\times p}$ and $%\Q \in \RB^{n\times q}$ where $p\leq m$ and $q\leq n$.
%\end{enumerate}
\end{proposition}

\begin{proof} Part~(1) directly follows from Proposition~\ref{pro:032} which shows that $\sib([\A]_r) \prec_w \sib( \langle\A\rangle_r)  \prec_w \sib(\A)$.
%Note that $\langle \A \Q\rangle_r = \A \langle \Q \rangle_r =[\A \Q_{n-r}, \0]$ where $\Q_{n-r}$ is the first $n-r$ columns of $\Q$
%Proposition~\ref{pro:032} also shows $\sigma_i(\A  \Q _{n-r}) \geq \sigma_{r+i} (\A)$ for $i=1, \ldots, n-r$.
%This implies that  $\sib(\Si_0) \prec_w \sib( \langle \A \Q\rangle_r)$. Thus, the result is obtained.
%We extend $\PP$ and $\Q$ to orthonormal matrices $\bar{\PP}=[\PP, \PP^{\bot}]$ and $\bar{\Q}=[\Q, \Q^{\bot}]$. Note that
%\[
%\bar{\PP}^T \A \bar{\Q} = \begin{bmatrix} \PP^T \A \Q & \PP^T \A \Q^{\bot} \\ (\PP^{\bot})^T \A \Q & \PP^{\bot})^T \A \PP^{\bot} \end{bmatrix}.
%\]
%Hence,
%\[
%\nm \A \nm = \nm \bar{\PP}^T \A \bar{\Q} \nm \geq \nm  \PP^T \A \Q \nm.
%\]
 \end{proof}

\begin{proposition} \label{pro:d-b-sv}  Given two   matrices $\A \in \RB^{m\times n}$ and ${\B} \in \RB^{m\times n}$,
we have that
\[
\nm \diag (\sib(\A) -\sib(\B)) \nm \leq \nm \A - \B  \nm.
\]
Furthermore, if both $\A$ and $\B$ are symmetric matrixes in $\RB^{m\times m}$, then 
\[
 \nm \diag (\sib(\A) -\sib(\B))  \leq \nm  \diag (\lamb(\A) -\lamb(\B))  \nm  \nm \leq \nm \A - \B  \nm.
\]
\end{proposition}
\begin{proof}
The first part of the  proposition is immediately obtained from Theorem~\ref{thm:d-b-sv}. 
As for the second part,  Proposition~\ref{pro:031}-(i) says that  $\lamb(\A) - \lamb(\B) \prec \lamb(\A-\B)$.   It then follows from Lemmas~\ref{lem:001} and \ref{lem:dsm} that $\lamb(\A) - \lamb(\B) = \sum_{j}  \alpha_j \PP_j \lamb(\A-\B)$ where the $\alpha_j\geq 0$ and $\sum_{j} \alpha_j=1$, and the $\PP_j$ are some permutation matrices. 
Accordingly, for every symmetric gauge function $\phi$ on $\RB^m$, we have that
\begin{align*}
\phi( \lamb(\A) - \lamb(\B)) &= \phi(  \sum_{j}  \alpha_j \PP_j \lamb(\A-\B) ) \leq  \sum_{j}  \alpha_j \phi(\PP_j \lamb(\A-\B)) \\
& =   \sum_{j}  \alpha_j \phi (\lamb(\A-\B)) = \phi( \lamb(\A-\B)),
\end{align*}
which implies that $\nm  \diag (\lamb(\A) -\lamb(\B))  \nm  \nm \leq \nm \A - \B  \nm$. 
Additionally, consider that for a symmetric matrix $\M$, it holds that $\sigma_i(\M) = |\lambda_i(\M)|$. 
Hence, we have that
\[
|\lambda_i(\A) - \lambda_i(\B)|  \geq \big| |\lambda_i(\A)| - |\lambda_i(\B)| \big| = |\sigma_i(\A) - \sigma_i(\B)|.
\]
This concludes the proof. 
\end{proof}

As a direct corollary of  Proposition~\ref{thm:d-b-sv}, we have  that
\[
|\sigma_i(\A)- \sigma_i(\B)| \leq \| \A - \B \|_2, \mbox{ for } \; i= 1, \ldots, q, 
\]
where $q=\min\{m, n\}$,  and 
\[
\sqrt{\sum_{i=1}^q  (\sigma_i(\A)- \sigma_i(\B))^2 } \leq \|\A -\B\|_F.
\]
When $\A$ and $\B$ are both symmetric, we also have that
\[
|\lambda_i(\A)- \lambda_i(\B)| \leq \| \A - \B \|_2, \mbox{ for } \; i= 1, \ldots, m, 
\]
\[
\sqrt{\sum_{i=1}^m  (\lambda_i(\A)- \lambda_i(\B))^2 } \leq \|\A -\B\|_F.
\]
The latter result is well known as the Hoffman-Wielandt theorem. Note that 
the Hoffman-Wielandt theorem still hods when $\A$ and $\B$ are normal~\citep{StewartSunBook:1990}.

\begin{theorem} \label{thm:uin-sp-nuclear} Let $\nm \cdot \nm$ be an arbitrary unitarily invariant norm on $\RB^{m\times n}$, and $\E_{11} \in \RB^{m\times n}$ have the entry 1 in the  $(1, 1)$th position and zeroes elsewhere. Then
\begin{enumerate}
\item[(a)]  $\nm \A \nm = \nm \A^T \nm$.
\item[(b)] $ \sigma_1(\A) \nm \E_{11} \nm  \leq  \nm \A  \nm \leq  \|\A\|_{*} \nm \E_{11} \nm$.
\item[(c)] If the symmetric gauge function $\phi$ corresponding to the norm $\nm \cdot \nm$ is normalized (i.e.,  $\phi(1, 0, 0, \ldots,0)=1$), then
\[ \|\A\|_2  \leq  \nm \A  \nm \leq  \|\A\|_{*} . \]
\end{enumerate}
\end{theorem}
\begin{proof}  Part (a) is due to that $\phi(\sib(\A))= \phi(\sib(\A^T))$.

If  $\phi(1, 0,  \ldots, 0)=1$, then $\nm \E_{11} \nm =1$. Thus,  we can have Part (c) from Part (b). Assume $\A$ is nonzero. Otherwise, the result is trivial.   Let $q = \min\{m, n\}$. First,
\begin{align*}
\nm \A \nm & = \phi(\sigma_1 (\A), \ldots, \sigma_q(\A)) = \sigma_1(\A)  \phi(1, \sigma_2(\A)/\sigma_1(\A), \ldots, \sigma_q (\A)/\sigma_1(\A) ) \\
& \geq \sigma_1(\A) \phi(1, 0, \ldots, 0) = \sigma_1(\A) \nm \E_{11} \nm.
\end{align*}
Since  $ \big(\sigma_1(\A)/\sum_{i=1}^q \sigma_i(\A), \ldots,  \sigma_q(\A)/\sum_{i=1}^q \sigma_i(\A) \big) \prec (1, 0, \ldots, 0)$,
we have
\begin{align*}
\nm \A \nm  &  = ( \sum_{i=1}^q \sigma_i(\A) ) \phi \big(\sigma_1(\A)/\sum_{i=1}^q \sigma_i(\A), \ldots,  \sigma_q(\A)/\sum_{i=1}^q \sigma_i(\A) \big)   \\
 &  \leq \|\A \|_* \phi(1, 0, \ldots, 0) = \|\A\|_* \nm \E_{11} \nm.
\end{align*}
\end{proof}

Note that a norm $\|\cdot\|$ on $\RB^{m\times n}$ is said to be \emph{self adjoint} if $\|\A\|=\|\A^T\|$ for any $\A \in \RB^{m\times n}$. Thus, Theorem~\ref{thm:uin-sp-nuclear}-(a) shows that the unitarily invariant norm is self-adjoint.  

It is worth mentioning that  $\nm \E_{ij} \nm = \nm \E_{11} \nm$ where $\E_{ij} \in \RB^{m\times n}$ has entry 1 in the $(i,j)$th position and zeros elsewhere. Moreover, the Schatten $p$-norms satisfy $\nm \E_{11} \nm_p =1$.
Theorem~\ref{thm:uin-sp-nuclear} says that  for any unitarily invariant norm $\nm \cdot \nm$ such that  $\nm \E_{11} \nm =1$,
\[
1  \leq \frac{\nm \A \nm} { \|\A\|}_2   \leq  \frac{\|\A\|_*} { \|\A\|}_2 \leq \rk(\A).
\]

Recall that  $\frac{\sum_{i=1}^q \sigma_i^2(\A)}{\sigma_1^2(\A)} =  \frac{\| \A \|^2_F} { \|\A\|^2_2}$  and  $\frac{ \sum_{i=1}^q \sigma_i(\A)}{\sigma_1(\A)} =\frac{\| \A \|_{*}} { \|\A\|_2}$, so called \emph{stable rank} and \emph{nuclear rank} (see Definition~\ref{def:grank}).  
They have been found usefulness in the analysis of matrix multiplication approximation~\citep{magen2011low,CohenNelsonWoodrull,kyrillidis2014approximate}.

\begin{theorem} Let $\M \in \RB^{m\times m}$, $\N \in \RB^{n\times n}$, and $\A \in \RB^{m\times n}$ such that the block matrix
\[
\begin{bmatrix} \M & \A \\ \A^T & \N \end{bmatrix}
\]
is SPSD. Then
\[
\nm \M \nm + \nm \N \nm  \geq 2 \nm \A \nm.
\]
\end{theorem}

\begin{proof} Without loss of generality, we assume $m\geq n$. Let $\A= \U \Si \V^T$ be a thin SVD of $\A$. Consider that
\[
[\U^T, - \V^T] \begin{bmatrix} \M & \A \\ \A^T & \N \end{bmatrix} \begin{bmatrix} \U  \\ -\V  \end{bmatrix} = \U^T \M \U + \V^T \N \V - \U^T \A \V - \V^T \A^T \U
\]
is PSD. Hence, $\nm \U^T \M \U + \V^T \N \V \nm \geq 2 \nm \Si \nm$. That is,
\[
\nm \V^T \U^T \M \U \V +  \N  \nm \geq 2 \nm \A \nm.
\]
Note that
\[\nm \V^T \U^T \M \U \V +  \N  \nm \leq \nm \V^T \U^T \M \U \V \nm + \nm \N \nm \leq \nm \M  \nm + \nm \N \nm. \]
\end{proof}

\begin{proposition}  \label{pro:uin-sum} Given a matrix $\A \in \RB^{m\times n}$, then the following holds
\[
\nm \A\nm = \min_{\X, \Y: \X \Y^T = \A}  \; \frac{1}{2} \Big\{ \nm \X \X^T \nm + \nm \Y \Y^T \nm  \Big \}.
\]
If $\rk(\A) = r \leq \min\{m, n\}$, then the minimum above is attained at a rank decomposition $\A=\hat{\X} \hat{\Y}^T$ where $\hat{\X} = \U_r \Si_r^{1/2}$ and $\hat{\Y} = \V_r \Si_r^{1/2}$, and $\A = \U_r \Si_r \V_r^T$ is a condensed SVD of $\A$.
\end{proposition}
\begin{proof} Let $\A =  {\X}  {\Y}^T$ be any  decomposition of $\A$. Then
\[
\begin{bmatrix} {\X} \\  {\Y} \end{bmatrix} [{\X}^T,  {\X}^T] = \begin{bmatrix} {\X} {\X}^T & {\X} {\Y}^T  \\ {\Y} {\X}^T & {\Y}  {\Y}^T  \end{bmatrix}
\]
is SPSD. Thus,
\[
\frac{1}{2} \Big[ \nm {\X} {\X}^T \nm + \nm {\Y} {\Y}^T \nm  \Big] \geq \nm \A \nm.
\]
When $\X \triangleq \hat{\X} = \U_r \Si_r^{1/2}$ and $\Y \triangleq \hat{\Y} = \V_r \Si_r^{1/2}$,
it holds that $\nm \A \nm = \frac{1}{2} \big[ \nm \hat{\X} \hat{\X}^T \nm + \nm \hat{\Y} \hat{\Y}^T \nm \big]$.
\end{proof}

Since $\frac{1}{2} \Big[ \nm {\X} {\X}^T \nm + \nm {\Y} {\Y}^T \nm  \Big] \geq \sqrt{ \nm {\X} {\X}^T \nm}  \sqrt{ \nm {\Y}  {\Y}^T \nm} $,
\[
\nm \A\nm \geq   \min_{\X, \Y: \X \Y^T = \A}  \;   \sqrt{ \nm {\X} {\X}^T \nm}  \sqrt{ \nm {\Y} {\Y}^T \nm}.
\]
When taking $\hat{\X} = \U_r \Si_r^{1/2} \V_r^T $ and $\hat{\Y} = \V_r \Si_r^{1/2} \V_r^T$,  one has
\[
\nm \A\nm =  \sqrt{ \nm \hat{\X} \hat{\X}^T \nm}  \sqrt {\nm \hat{\Y} \hat{\Y}^T \nm}.
\]
This thus leads us to the following proposition.
\begin{proposition}  \label{pro:chay-s-uin} Given a matrix $\A \in \RB^{m\times n}$, then the following holds
\[
\nm \A\nm = \min_{\X, \Y: \X \Y^T = \A}  \;  \sqrt{  \nm \X \X^T \nm}  \sqrt{ \nm \Y \Y^T \nm }.
\]
Accordingly, the  following inequality hods:
\begin{equation} \label{eqn:cs-uin}
\nm \X \Y^T \nm \leq \nm \X \X^T \nm^{1/2} \nm \Y \Y^T \nm^{1/2}.
\end{equation}
\end{proposition}
This is a form of the Cauchy-Schwarz inequality under the unitarily invariant norms.

As a corollary of Proposition~\ref{pro:uin-sum}, the following proposition immediately follows. Moreover, this proposition was widely used in matrix completion problems, because an optimization problem regularized by the Frobenius norm   is  solved more easily than that regularized by the nuclear norm~\citep{hastie2014matrix}.

\begin{proposition} \citep{srebro2004maximum,mazumder2010spectral} Given a matrix $\A \in \RB^{m\times n}$, then the following holds
\[
\|\A\|_{*} = \min_{\X, \Y: \X \Y^T = \A}  \; \frac{1}{2} \Big\{ \| \X\|_F^2 + \|\Y\|_F^2  \Big \}.
\]
If $\rk(\A) = k \leq \min\{m, n\}$, the  minimum above is attained at some rank decomposition.
\end{proposition}

%We are especially concerned with the Frobenius norm and  spectral norm. %Clearly, they satisfy
%\[
%\|\A\|_{\xi}^2 = \|\A^T \A\|_{\xi} = \|\A \A^T \|_{\xi},
%\]
%where either ``$\xi= F$'' or ``$\xi= 2$.''
The following theorem shows that the Frobenius norm has  a so-called matrix-Pythagoras' property.
However, for other Schatten norms, there needs a strong condition to make the property hold.

\begin{theorem} \label{thm:pytho} Let $\A, \B \in \RB^{m\times n}$. If $\A \B^T=\0$ or $\A^T \B =\0$, then
\[
\|\A + \B\|_F^2 = \|\A\|_F^2 + \|\B\|_F^2,
\]
\[
\max\{\|\A\|_2^2,  \|\B\|_2^2 \} \leq  \|\A + \B\|_2^2 \leq \|\A\|_2^2 + \|\B\|_2^2.
\]
%and
%\[
%\nm \A \nm_p^p + \nm \B \nm_p^p \leq \nm \A+ \B\nm_p^p
%\]
%for $ 1 \leq p < \infty$ and $p\neq 2$.

If both $\A \B^T=\0$ and $\A^T \B =\0$ are satisfied, then
\[
\nm \A +  \B \nm_{p}^p  = \nm  \A\nm_{p}^p  + \nm  \B \nm_{p}^p
\]
for $1 \leq p < \infty$ and $\|\A + \B\|_2 = \max\{ \|\A\|_2,  \|\B\|_2\}$.
\end{theorem}
\begin{proof}  Since $(\A + \B)^T (\A + \B)=\A^T \A + \B^T \B$ when $\A^T \B =\0$,  the Pythagorean property for the Frobenius norm is obvious.
As for the spectral norm, it is easily seen that
\begin{align*}
\|\A + \B\|_2^2 & = \max_{\|\x\|_2=1} \x^T(\A + \B)^T (\A + \B) \x  \\
 & = \max_{\|\x\|_2=1} \x^T(\A^T \A + \B^T  \B) \x  \\
 & \leq \max_{\|\x\|_2=1} \x^T \A^T \A \x +  \max_{\|\x\|_2=1} \x^T \B^T  \B \x \\
 & = \|\A\|_2^2 + \|\B\|_2^2.
\end{align*}
 Let the condensed SVDs of $\A$ and $\B$ be $\A= \U_A \Si_A \V_A^T$ and $\B = \U_B \Si_B \V_B^T$. If $\A^T \B =\0$ and $\A \B^T =\0$, then $\V_A^T \V_B = \0$ and $\U_A^T \U_B = \0$. Note that
\[
\A + \B = [\U_A, \U_B] \begin{bmatrix} \Si_A & \0 \\ \0 & \Si_B \end{bmatrix} \begin{bmatrix} \V_A^T \\ \V_B^T \end{bmatrix}
\]
is the condensed SVD of $\A+\B$.
So the nonzero singular values of  $\A+ \B$ consist of  those of $ \A$ and of $ \B$. The theorem  accordingly follows.
%Hence,
%\[
%(\A + \B) (\A + \B)^T = \A \A^T + \B \B^T = [\U_A, \U_B] \begin{bmatrix} \Si_A^2 & \0 \\ \0 & \Si_B^2 \end{bmatrix} \begin{bmatrix} \U_A^T \\ \U_B^T \end{bmatrix},
%\]
%from which we have
%\[
%\| \A \A^T + \B \B^T \| =  \left\| \begin{bmatrix} \Si_A^2 & \0 \\ \0 & \Si_B^2 \end{bmatrix}  \right \|
%\]
%for every unitarily invariant norm $\|\cdot \|$. The theorem  accordingly follows.
\end{proof}

%The above proof also implies that  if $\A \B^T =\0$, then
%\[
%\|(\A+\B) (\A+\B)^T\| \leq \|\A \A^T\| + \|\B \B^T\|
%\]
%holds for every matrix operator norm.
%However,
%this property does not always hold for every unitarily invariant norms.  Let the condensed SVDs of $\A$ and $\B$ be $\A= \U_A \Si_A \V_A^T$ and $\B = \U_B \Si_B \V_B^T$. Assume $\A^T \B =\0$. Then $\V_A^T \V_B =\0$. Note that
%\[
%\A + \B = [\U_A, \U_B] \begin{bmatrix} \Si_A & \0 \\ \0 & \Si_B \end{bmatrix} \begin{bmatrix} \V_A^T \\ \V_B^T \end{bmatrix}.
%\]
%Thus,
%\[
%(\A + \B) (\A + \B)^T = [\U_A, \U_B] \begin{bmatrix} \Si_A^2 & \Si_A \V_A^T \V_B \Si_B \\ \Si_B \V_B^T \V_A \Si_A & \Si_B^2 \end{bmatrix} \begin{bmatrix} \U_A^T \\ \U_B^T \end{bmatrix}
%\]
%Now consider that
%\[
%\nm \A+\B\nm =
%\]

Let us end this chapter by showing a relationship among the matrix operator, matrix vectorization, and unitarily invariant norms.

\begin{theorem} \label{thm:uin-2} Let $f$ be a  matrix norm on $\RB^{m\times n}$.
\begin{itemize}
\item[(a)] The norm $f$ is both unitarily invariant and operator norm if and only if  $f(\A) = \|\A\|_2$ for any $\A \in \RB^{m\times n}$. In other words, the spectral norm is only one operator norm that satisfies the self-adjoint property.  %$\nm \A \nm = \nm \A^T \nm$.
\item[(b)]  Given a matrix $\A \in \RB^{m\times n}$,  $f(\A)  \triangleq \|\vect(\A)\|$ is  unitarily invariant  if and only if it is the  norm $\gamma \|\A\|_F$ for some $\gamma> 0$.
\end{itemize}
\end{theorem}

\begin{proof}
The proof of Part (a) can be found in Corollary 5.6.35 of \cite{Horn:1985}. As for Part (b),  it is obvious that
the Frobenius norm is both unitarily invariant and  vectorization norm. Conversely, given any   $\A \in \RB^{m\times n}$,
the vectorization norm is defined as $\|\a \|$ where $\a= \vect(\A)$. Recall that the vector $\a$ can be regarded as an $mn \times 1$ matrix.  
Let $\a= \U_a \Si_a \v_a^T$ be the full SVD of $\a$. Then
it is easily seen that $\Si_a=(\|\A\|_F, 0, \ldots, 0)^T $. Moreover, we can set $\v_a=1$.  For any orthonormal matrices $\U \in \RB^{m\times m}$ and  $\V \in \RB^{n\times n}$, we have that $f(\U \A \V^T) = \|\vect(\U \A \V^T) \| = \|(\V \otimes \U) \vect(\A)\| = \| \a\|$ due to the unitary invariance. Moreover, we have that $\|\a\| = \| \Si_a \| = \|\A\|_F   \|(1, 0, \ldots, 0)\|$. Letting $\gamma =\|(1, 0, \ldots, 0)\|>0$, we complete the proof.
Notice that if the norm is normalized, then $\gamma=1$.
%First note that $\sib(\A)= \sib(\A^T)$. Thus, it holds that $\nm \A \nm = \nm \A^T \nm$ for any unitarily invariant norm.
%The  Thus, Part (a) follows.
\end{proof}

\chapter{Subdifferentials of Unitarily Invariant  Norms}
\label{ch:subdiff}

In the previous chapters, we have used matrix differential calculus. Let $f: \RB^{m\times n} \to \R$. We have discussed the gradient and Hessian of $f$ w.r.t.\ $\X \in \RB^{m\times n}$. Especially,
the function $f: \RB^{m\times n} \to \RB$ is defined as a trace function. Such a function is differentiable.
In this chapter
we consider  $f$ to be a unitarily invariant norm.

Norm functions   are not necessarily differentiable.  For example,   the spectral norm and nuclear norm are not differentiable.  But  norm functions are convex and continuous, so we can resort to theory of  subdifferentials~\citep{Rockafellar:1970,BorweinLewis}.
Indeed, the subdifferentials of  unitarily invariant norms have been studied by  \cite{watson:1992} and \cite{lewis2003mathematics}.

Using the properties of  unitarily invariant norms and the SVD theory, we present directional derivatives and subdifferentials of
unitarily invariant norms.  As two special cases, we report the subdifferentials of the spectral norm and nuclear norm.
These two norms have been widely used in machine learning such as matrix low rank approximation. We illustrate applications of the subdifferentials in  optimization problems regularized by either the  spectral norm or the nuclear norm.
We also study the use of the subdifferentials of unitarily invariant norms in solving least squares estimation problems, whose loss function is defined as any unitarily invariant norm.

\section{Subdifferentials}

Let $\|\cdot\|$ be a given norm on $\R^{m\times n}$, and $\A$ be a given matrix in $\RB^{m\times n}$. The subdifferential, a set of subgradients, of $\|\A\|$ is
defined as
\[
\Big \{\G \in \RB^{m\times n}: \|\B\| \geq \|\A\| + \tr((\B-\A)^T \G) \mbox{ for all  } \B \in \RB^{m\times n} \Big \},
\]
and  denoted by $\partial \|\A\|$.
When the norm $\| \cdot \|$ is differentiable, the subgradient degenerates to the gradient. That is, the subdifferential is a singleton.
For example, when taking the squared Frobenius norm $\|\A\|_F^2 = {\tr(\A^T \A)}$, $\partial \|\A\|_F^2 = \{ 2 \A \}$.

\begin{lemma} \label{lem:subdiff} Let $\A \in \RB^{m\times n}$ be a given matrix.  Then  $\G \in \partial \| \A\|$ if and only if
$\|\A\| = \tr(\G^T \A)$ and $\|\G\|^* \leq 1$.
\end{lemma}

\begin{proof} The sufficiency is immediate. Now assume that $\G \in \partial \| \A\|$. Then taking $\B = 2 \A$ yields $\|\G\| \geq \tr(\A^T \G)$
and taking $\B = \frac{1}{2} \A$ yields $\frac{1}{2} \|\A\| \leq \frac{1}{2} \tr(\A^T \G)$, which implies that $\|\A\| = \tr(\A^T \G)$.
Subsequently, $\|\B\|\geq \tr(\G^T \B)$ for all matrices $\B$. Thus, the dual norm satisfies
\[
\|\G\|^* = \max \{\tr(\G^T \B): \|\B\|=1  \}\leq 1.
\]
\end{proof}

We especially  consider the subdifferential of unitarily invariant norms.
Given a unitarily invariant norm $\nm \cdot \nm$ on $\RB^{m\times n}$, let $p=\min\{m, n\}$. Theorem~\ref{thm:ui-sg} shows there exists a symmetric gauge function $\phi: \RB^{p} \to \RB$ associated with the norm $\nm \cdot \nm$.
Thus, this encourages us to define the subdifferential of unitarily invariant norms via the  subdifferential of symmetric gauge functions.

The subdifferential of the symmetric gauge function $\phi$ at $\x \in \RB^p$ is
\[
\partial \phi(\x) \triangleq \{\z \in \RB^p: \phi(\y)\geq \phi(\x) + (\y-\x)^T \z \; \mbox{ for all } \y \in \RB^p \}.
\]
In terms of Lemma~\ref{lem:subdiff}, that $\z \in \partial \phi(\x)$ is equivalent to that $\phi(\x)= \x^T \z$ and $\phi^*(\z)\leq 1$.
Here $\phi^*$  is the dual of $\phi$ (see Definition~\ref{def:polar})
%defined as
%\[
%\phi^*(\x) = \max_{\phi(\y) =1} \x^T \y,
%\]
which is  a symmetric gauge function for the dual norm $\nm \cdot \nm^*$. That is, $\phi^{*}(\sib(\A)) = \nm \A \nm^{*}$ (see Theorem~\ref{thm:dual-pri}).

Let us return to the subdifferential of unitarily invariant norms. The following lemma gives the directional derivative of $\nm \A\nm$.
\begin{lemma}  \label{lem:dd}  Let $\nm \cdot \nm$ be a given unitarily invariant norm on $\RB^{m\times n}$, and $\phi$ be the corresponding symmetric gauge function.  Then the directional derivative of the norm at $\A \in \RB^{m\times n}$ in a direction $\R \in \RB^{m\times n}$ is
\[
 \lim_{t \downarrow 0}  \frac{\nm \A {+} t \R\nm - \nm \A\nm }{t} =  \max_{\d \in \partial \phi(\sib(\A))}  \; \sum_{i=1}^p d_i \u_i^T \R \v_i = \max_{\G \in \partial \nm \A \nm} \; \tr(\R^T \G) .
\]
Here $p=\min\{m, n\}$,  $\U=[\u_1, \ldots, \u_m]$, $\V=[\v_1, \ldots, \v_n]$,  $\Si=\diag (\sib(\A))$, and $\A = \U \Si \V^T$ is a full SVD of $\A$.
\end{lemma}

\begin{proof}  By Lemma~\ref{lem:subg-dd}, we immediately have
\[
 \lim_{t \downarrow 0}  \frac{\nm \A + t \R\nm - \nm \A\nm }{t} =   \max_{\G \in \partial \nm \A \nm} \; \tr(\R^T \G).
 \]
We now prove the first equality.
Let $\z = (\u_1^T \R \v_1, \ldots,  \u_p^T \R \v_p)^T $.  Consider that
\[
\nm \A + t \R\nm = \nm \Si+ t \U^T \R \V \nm = \phi(\sib(\Si+ t \U^T \R \V))\geq \phi(\sib(\A) + t \z)
\]
because $\sib(\A) + t \z  \prec_w  \sib(\Si+ t  \U^T \R \V) $  by Proposition~\ref{pro:032}. Accordingly, we have that
 \[
 \lim_{t \downarrow 0}  \frac{\nm \A {+} t \R\nm - \nm \A\nm }{t} \geq  \lim_{t \downarrow 0}  \frac{\phi(\sib(\A) {+} t \z) - \phi(\sib(\A)) }{t} = \max_{\d \in \partial \phi(\sib(\A))} \d^T \z.
   \]
The above equality follows from  Lemma~\ref{lem:subg-dd}, when applied to the symmetric gauge function $\phi$.

On the other hand,  let $\sib(t) \triangleq \sib(\A {+} t \R) = \sib(\Si {+} t  \U^T \R \V)$.
Now  we  have
\begin{align*}
  \frac{ \nm \A  \nm -  \nm \A {+} t \R\nm  }{t}  & =   \frac{ \nm \A{+} t \R {-} t \R   \nm -  \nm \A {+} t \R\nm }{t}  \\
 & =  \frac{\phi(\sib(\Si {+} t  \U^T \R \V {-}  t  \U^T \R \V)) - \phi(\sib(t) ) }{t} \\
 & \geq     \frac{\phi(\sib(t) - t \z) - \phi(\sib(t) ) }{t} \\
 & \geq -  \d(t)^T  \z \;  \;  [\mbox{where } \d(t) \in \partial \phi(\sib(t))].
 \end{align*}
 The above first inequality follows from  $\sib(t) - t \z  \prec_w  \sib(\A)$.  %$\sib(\A) - t \z  \prec_w  \sib(\Si- t  \U^T \R \V)$.
 The  second inequality is based on the property of  the subgradient  of $\phi$ at $\sib(t)$. Note that $\phi$ is a continuous function.   By the definition of $\partial \phi(\sib(t))$, it is directly verified that $\liml_{t\to 0+} \d(t) \to {\d}_0 \in \partial \phi(\sib(\A))$. Thus,
 \[
  \lim_{t \downarrow 0}  \frac{\nm \A {+} t \R\nm - \nm \A\nm }{t} \leq \lim_{t\downarrow 0} \d(t)^T \z = {\d}_0^T \z \leq
\max_{\d \in \partial \phi(\sib(\A))}  \d^T \z.  \]
 This implies that the first equality also holds.
\end{proof}

\begin{theorem} \label{thm:subdiff-uin} Let  $\A \in \RB^{m\times n}$ have a full SVD $\A = \U \Si \V^T$, and let $\sib=\dg(\Si)$. Then
\[
\partial \nm \A \nm = \mathrm{conv}\Big\{\U \D \V^T: \d \in \partial \phi(\sib), \D = \diag(\d)  \Big\}.
\]
where $\phi$ is a symmetric gauge function corresponding to the norm $\nm \cdot \nm$.
\end{theorem}
Here the notation ``$\mathrm{conv}\{\cdot\}$'' represents the convex hull of a set, which is closed and convex. If $\G \in \partial \nm \A \nm$, Theorem~\ref{thm:subdiff-uin} says that
$\G$ can be expressed as
\[
\G = \sum_{i} \alpha_i \U^{(i)} \D^{(i)} (\V^{(i)})^T,
\]
where $\alpha_i\geq 0$, $\sum_{i} \alpha_i =1$, $\A= \U^{(i)} \Si (\V^{(i)})^T$ is a full SVD,  $\d_i \in \phi(\sib)$, and $\D^{(i)} = \diag(\d_i)$.
According to Corollary~\ref{cor:svd-uni}, we can rewrite $\G$ as
\begin{equation}  \label{eqn:G}
\G = \sum_{i} \alpha_i \U \Q^{(i)} \D^{(i)} (\PP^{(i)})^T \V^T,
\end{equation}
where $\PP^{(i)}$ and $\Q^{(i)}$ are defined as $\PP$ and $\Q$ in Corollary~\ref{cor:svd-uni};  i.e., they satisfy that $\Q^{(i)} \Si (\PP^{(i)})^T  = \Si$ and $(\Q^{(i)})^T \Si \PP^{(i)} = \Si$.
\begin{proof} First of all, we denote the convex hull on the right-hand side by $\GM(\A)$. Assume that $\G \in \GM(\A)$.
We now prove $\G \in \partial \nm \A \nm$. Based on Lemma~\ref{lem:subdiff}, we try to show that $\nm \A \nm = \tr(\A^T \G)$ and $\nm \G \nm^{*} \leq 1$.
In terms of the above discussion, we can express $\G$ as in \eqref{eqn:G}.  Thus,
 \begin{align*}
 \tr(\A^T \G) & = \sum_{i=1} \alpha_i  \tr(\A^T\U \Q^{(i)} \D^{(i)} (\PP^{(i)})^T \V^T)  \\
  &=  \sum_{i=1} \alpha_i  \tr((\PP^{(i)})^T \Si^T \Q^{(i)} \D^{(i)} )  =  \sum_{i=1} \alpha_i  \tr( \Si^T \D^{(i)} ) \\
  & = \sum_{i=1} \alpha_i \d_i^T \sib =  \phi(\sib) = \nm \A \nm.
   \end{align*}
Additionally,
\begin{align*}
\nm \G \nm^{*} = \max_{\nm \R \nm \leq 1} \; \tr(\G^T \R)= \max_{\nm \R \nm \leq 1} \;  \tr\big(\R^T \sum_{i=1} \alpha_i \U^{(i)} \D^{(i)} (\V^{(i)})^T \big).
\end{align*}
Since for each $i$,
\[
\nm  \U^{(i)} \D^{(i)} (\V^{(i)})^T \nm^* =  \nm  \D^{(i)}  \nm^* = \phi^{*}(\d_i) \leq 1,
\]
and by Proposition~\ref{pro:dualnorm} we have
\[
\tr(\R^T   \U^{(i)} \D^{(i)} (\V^{(i)})^T) \leq \nm \R \nm \times  \nm  \U^{(i)} \D^{(i)} (\V^{(i)})^T \nm^* \leq  \nm \R \nm.
\]
Thus, $\nm \G \nm^* \leq 1$. In summary, we have $\G \in \partial \nm \A \nm$.

Conversely, assume that $\G \in \partial \nm \A \nm$ but $\G \notin \GM(\A)$. Then by the well-known separation theorem \cite[see Theorem 1.1.1]{BorweinLewis} there exists a matrix $\R \in \RB^{m\times n}$ such that
\[
\tr(\R^T \X) < \tr(\R^T \G) \; \mbox{ for all } \; \X \in \GM(\A).
\]
This implies that
\[
\max_{\d \in \partial \phi(\sib)} \sum_{i=1} d_i \u_i^T \R \v_i = \max_{\X \in \GM(\A)} \tr(\R^T \X) < \max_{\G \in \partial \nm \A \nm } \; \tr(\R^T \G).
\]
This contradicts with Lemma~\ref{lem:dd}. Thus, the theorem follows.
\end{proof}

We are especially interested in the spectral norm $\|\cdot\|_2$ and the nuclear norm $\|\cdot\|_*$.  As corollaries of  Theorem~\ref{thm:subdiff-uin}, we have the following
the results.

\begin{corollary} \label{cor:subdiff01} Let $\A$ have rank $r \leq p= \min\{m, n\}$ and $\A=\U_r\Si_r \V_r^T$ be a condensed SVD.  Then the
subdifferential of $\|\A\|_*$ is give as
\[
\partial \|\A\|_* = \Big\{\U_r \V_r^T + \W: \W \in \RB^{m{\times} n} \mbox{ s.t. } \U_r^T \W=\0, \W \V_r = \0, \|\W\|_2 \leq 1   \Big\}.
\]
\end{corollary}
\begin{proof} For the nuclear norm, the corresponding symmetric gauge function is $\phi(\sib)=\|\sib\|_1=\sum_{i=1}^p \sigma_i$. Moreover,
\[
\partial \|\sib\|_1 = \big \{\u \in \RB^p:  \|\u\|_{\infty} \leq 1 \mbox{ and }  u_i =1 \mbox{ for } i=1, \ldots, r \big\}.
\]
Let $\G \in \partial \|\A\|_*$. By Theorem~\ref{thm:subdiff-uin} and  Corollary~\ref{cor:svd-uni}, we have
\begin{align*}
\G  & = \sum_{i=1} \alpha_i \U \Q^{(i)} \D^{(i)} (\PP^{(i)})^T \V^T  \\
  &= \U_r \V_r^T + \sum_{i=1} \alpha_i \U_{-r} \Q_0^{(i)} \D_{-r}^{(i)} (\PP_0^{(i)})^T \V_{-r}^T,
\end{align*}
where the $\alpha_i\geq 0$ and $\sum_{i=1} \alpha_i =1$,   $\D^{(i)} = \dg(\d_i)$, $\d_i \in \partial \phi(\sib)$, and $\D_{-r}^{(i)}$ is the last $(m-r) \times (n-r)$ principal submatrix of $\D^{(i)}$.  Here $\Q^{(i)}  \in \RB^{m\times m}$, $\PP^{(i)} \in \RB^{n\times n}$, $\Q_0^{(i)} \in \RB^{(m-r) \times (m-r)}$, and $\PP_0^{(i)} \in \RB^{(n-r) \times (n-r)}$  are orthonormal matrices,  which are defined  in Corollary~\ref{cor:svd-uni}.
Let
\begin{equation} \label{eqn:WW}
\W \triangleq  \U_{-r} \Big[  \sum_{i=1} \alpha_i  \Q_0^{(i)} \D_{-r}^{(i)} (\PP_0^{(i)})^T \Big] \V_{-r}^T.
\end{equation}
Obviously, $\U_r^T \W =\0$ and $\W \V_r=\0$. Moreover,
\[
\| \W\|_2 \leq \sum_{i=1} \alpha_i \|\D_{-r}^{(i)} \|_2 \leq 1.
\]
We can also see that any matrix  $\W$ satisfying the above three conditions  always has an expression as in \eqref{eqn:WW}.
\end{proof}

\begin{corollary} \label{cor:subdiff02} Let the largest singular value $\sigma_1$ of $\A \in \RB^{m\times n}$ have multiplicity  $t$, and $\U_t$ and $\V_t$ consist of the first $t$ columns of $\U$ and $\V$ respectively. Then
\[
\partial \|\A\|_2 = \Big\{\U_t \H \V_t^T: \H \in \RB^{t\times t} \mbox{ s.t. } \H \mbox{  is SPSD}, \tr(\H)=1   \Big\}.
\]
\end{corollary}
\begin{proof} The corresponding symmetric gauge function is $\phi(\sib)=\|\sib\|_{\infty}$, and its subdifferential is
\[
\partial \|\sib \|_{\infty}  = \mathrm{conv}\{\e_i:  \;  i  = 1, \ldots, t \},
\]
where $\e_i$ is the $i$th column of the identity matrix. %and $t$ is the multiplicity of $\sigma_1$.  Partition $\U$, $\V$ and $\Si$ as
%\[
%\U = [\underbrace{\U_{t}}_{t}, \underbrace{ \U_{-t}}_{m-t}] , \;  \V=[\underbrace{\V_t}_{t},  \underbrace{\V_{-t}}_{n-t}],  \; \mbox{ and } \;
%\Si = (\sigma_1 \I_t) \oplus \Si_{-t}.
%\]
It then follows from Theorem~\ref{thm:subdiff-uin} that for any $\G \in \partial \|\A\|_2$, it can be written as
\[
\G = \sum_{i=1} \alpha_i \U_t \Q^{(i)} \D_t^{(i)} (\Q^{(i)})^T \V_t^T,
\]
where the $\alpha_i\geq 0$ and $\sum_{i=1} \alpha_i =1$, and $\Q^{(i)} $ is an arbitrary $t\times t$ orthonormal matrix (see Theorem~\ref{thm:svdvector}).  Here $\D^{i} = \dg(\d_i)$, $\d_i \in \partial \phi(\sib)$, and $\D_t^{(i)}$ is the first $t\times t$ principal submatrix of $\D^{(i)}$.  Let
\begin{equation} \label{eqn:HH}
\H = \sum_{i=1} \alpha_i  \Q^{(i)} \D_t^{(i)} (\Q^{(i)})^T,
\end{equation}
which is SPSD and satisfies $\tr(\H)=1$. Conversely, any SPSD matrix $\H$ satisfying $\tr(\H)=1$ can be always expressed as the form of \eqref{eqn:HH}.
\end{proof}

\section{Applications}

In this section we present several examples to illustrate the application of the subdifferential of unitarily invariant norms in solving an optimization problem  regularized by a unitarily invariant norm or built on any unitarily invariant norm loss.

\begin{example} \label{exm:svt} Given a nonzero matrix $\A \in \RB^{m\times n}$, consider the following optimization problem:
\begin{equation} \label{eqn:svtproblrm}
\min_{\X \in \RB^{m\times n} } \; f(\X) \triangleq \frac{1}{2} \|\X-\A\|_F^2 + \tau \|\X\|_*,
\end{equation}
where $\tau>0$ is a constant. Clearly, the problem is  convex in $\X$. This problem is a steppingstone of matrix completion.
Let $\A =\U_r \Si_r \V_r^T$ be a given condensed SVD of  $\A$, and define
\[
\hat{\X} = \U_r [\Si_r - \tau \I_r]_{+} \V_r, \]
where  $[\Si_r - \tau \I_r]_{+} =\diag ([\sigma_1-\tau]_+, \ldots, [\sigma_r-\tau]_+)$ and $[z]_+ = \max(z, 0)$.
Now it can be directly checked that
\[
\partial f(\hat{\X}) = \hat{\X} - \A + \tau \partial \|\hat{\X}\|.
\]
Assume that the first $k$ singular values $\sigma_i$ are greater than $\tau$. Then,
\[\frac{1}{r}(\A - \hat{\X}) = \U_k \V_k^T + \frac{1}{\tau} \U_{k+1:r} \diag(\sigma_{k+1}, \ldots, \sigma_r) \V_{k+1:r}^T,\] which belongs to
 $\partial \|\hat{\X}\|$.
In other words, $\0 \in \partial f(\hat{\X})$ (see Corollary~\ref{cor:subdiff01}). Thus, $\hat{\X}$ is a minimizer of the optimization problem.
It is called the singular value thresholding (SVT) operator~\citep{cai2010singular}.
We can see that the parameter $\tau$ controls the rank of the matrix $\hat{\X}$ and the problem is able to yield a low rank solution to the matrix $\X$. That is,  $\hat{\X}$ is a low rank approximation to the matrix $\A$. 
\end{example}

\begin{example} \label{exm:sva} Given a nonzero matrix $\A \in \RB^{m\times n}$, consider the following optimization problem:
\begin{equation} \label{eqn:spectral_cont}
\min_{\X \in \RB^{m\times n} } \; f(\X) \triangleq \frac{1}{2} \|\X-\A\|_F^2 + \tau \|\X\|_2,
\end{equation}
where $\tau>0$ is a constant. Also, this problem is  convex in $\X$. Let $\A$ have the $k$ distinct positive singular values $\delta_1> \delta_2>\cdots > \delta_k$ among the $\sigma_i$, with respective multiplicities $r_1, \ldots, r_k$. Thus, the rank of $\A$ is $r = \sum_{i=1}^k r_i$.
Let $m_t = \sum_{i=1}^t r_i$ and $\mu_t = \sum_{i=1}^t r_i \delta_i$ for $t=1, \ldots, k$. So $m_k= r$ and $\mu_k = \tr(\Si_r)=\sum_{i=1}^r \sigma_i$.  Assume that $\tau \leq \mu_k$.
We now consider two cases.

In the first case, assume  $l\in [k-1]$ is the smallest integer such that
\[
\sum_{i=1}^l r_i (\delta_i - \delta_{l+1}) = \mu_{l} - \delta_{l+1} m_l> \tau,
\]
and hence, $\delta_{l}\geq \frac{\mu_{l}-\tau}{m_l}> \delta_{l+1}$. Note that
\begin{align*}
\sum_{i=1}^{l+1} r_i (\delta_i - \delta_{l{+}2}) & = \sum_{i=1}^l r_i (\delta_i {-} \delta_{l{+}1}) {+} \sum_{i=1}^{l+1} r_i (\delta_{l{+}1} {-} \delta_{l{+}2})\\
& >\sum_{i=1}^l r_i (\delta_i {-} \delta_{l{+}1}) > \tau.
\end{align*}
This implies that  $l$ is identifiable.
Denoting $\delta = \frac{\mu_{l}-\tau}{m_l}$, we define
$\hat{\Si}$ by replacing the first $m_l$ diagonal elements of $\Si_r$ by $\delta$, and then set $\hat{\X}= \U_r \hat{\Si}_r \V_r^T$.
Now note that
\[
\frac{1}{\tau}(\A - \hat{\X}) = \U_{m_l} \H \V_{m_l}^T,
\]
where $\H = \diag\big((\sigma_1-\delta)/\tau, \ldots, (\sigma_{m_l}-\delta)/\tau \big)$.
Clearly, $\H$ is PSD and $\tr(\H)= \sum_{i=1}^{m_l} \frac{\sigma_i -\delta}{\tau} = \sum_{i=1}^{l} \frac{r_l (\delta_i - \delta)}{\tau} =1$. It follows from Corollary~\ref{cor:subdiff02} that $\frac{1}{\tau}(\A - \hat{\X}) \in \partial \|\hat{\X}\|_2$.
Thus,  $\hat{\X}$ is a minimizer.

In the second case,  otherwise, $\sum_{i=1}^{k-1} r_i (\delta_i-\delta_k) = \mu_{k-1} - m_{k-1} \delta_k \leq \tau \leq \mu_k$. Let $\delta= \frac{\mu_k - \tau}{m_k}$ such that
\[0\leq \delta   \leq \frac{\mu_k - \mu_{k-1} + \delta_k m_{k-1}}{m_k} = \delta_k.
\]
Define $\hat{\X}= \U_r \delta \I_r \V^T$. Then
\[
\frac{1}{\tau} (\A - \hat{\X}) = \frac{1}{\tau} \U_r(\Si_r -\delta \I_r) \V_r^T.
\]
Since $\frac{1}{\tau}(\Si_r -\delta \I_r)$ is PSD and $\frac{1}{\tau}\tr(\Si_r - \delta \I_r) = 1$, we obtain $\0 \in \partial f(\hat{\X})$. This implies that
$\hat{\X}$ is a minimizer of the problem.

As we have seen, the minimizer $\hat{\X}$ has the same rank with $\A$. Thus,  the problem in \eqref{eqn:spectral_cont}
can not give a low-rank solution. However, this problem makes the singular values of $\hat{\X}$ more well-conditioned because  the top singular values decay to $\delta$. Thus, we call it a singular value averaging (SVA) operator.
\end{example}

\begin{example} \label{exm:hybrid}   Given a nonzero matrix $\A \in \RB^{m\times n}$, consider the following convex optimization problem:
\begin{equation} \label{eqn:hproblrm}
\min_{\X \in \RB^{m\times n} } \; f(\X) \triangleq  \|\X-\A\|_2 + \tau \|\X\|_*,
\end{equation}
where $\tau>0$ is a constant.  In the above model  the loss function and regularization term are respectively defined as  the spectral norm and the nuclear norma, which are mutually dual. Moreover, this
model can be regarded as a parallel version of the Dantzig selector~\citep{CandesTao:2007}.
Thus, this model might be potentially interesting.

Let $\A = \U_r \Si_r \V_r^T$ be a condensed SVD.  Assume that $r \tau > 1$. Assume there are the $k$ distinct positive singular values $\delta_1> \delta_2>\cdots > \delta_k$ among the $\sigma_i$, with respective multiplicities $r_1, \ldots, r_k$.
Let $m_t = \sum_{i=1}^t r_i$  for $t=1, \ldots, k$.

Let $l \in [k]$ be the smallest integer such that $m_l \tau \geq  1 > m_{l-1} \tau$.
Define $\hat{\X}=\U_{r} [\Si_r-\delta_l \I_r]_+ \V_{r}^T= \U_{m_{l-1}} \diag(\sigma_1 - \delta_l, \ldots, \sigma_{m_{l-1}}-\delta_l) \V_{m_{l-1}}^T$. Then
$\A - \hat{\X}$
has the maximum singular value $\delta_l$ with multiplicity $m_l$. It follows from Corollaries~\ref{cor:subdiff01} and \ref{cor:subdiff02} that
\[
\partial \|\hat{\X}\|_* = \Big\{\U_{m_{l{-}1}}\V_{m_{l{-}1}}^T + \W: \W^T \U_{m_{l{-}1}}=\0, \W \V_{m_{l{-}1}}=\0, \|\W\|_2 \leq 1  \Big \}
\]
and
\[
\partial \|\A - \hat{\X}\|_2 = \Big \{- \U_{m_{l}} \H \V_{m_{l}}^T: \H \mbox{ is PSD}, \tr(\H)=1  \Big \}.
\]
Take $\W_0 = \U_{[m_{l-1}+1: m_l]} \frac{(1- m_{l-1} \tau)}{r_l \tau} \I_{r_l} \V_{[m_{l-1}+1: m_l]}^T$.  Note that $\W_0 \V_{m_{l-1}}=0$,
$\W_0^T  \U_{m_{l-1}} =\0$, and $\|\W_0\|_2 = \frac{(1- m_{l-1} \tau)}{r_l \tau} \leq 1$ due to $m_{l-1} \tau + r_l \tau = m_l \tau \geq 1$ and $m_{l-1} \tau < 1$. Hence,
\[
\tau \partial \|\hat{\X}\|_* \ni \tau(\U_{m_{l-1}}\V_{m_{l-1}}^T + \W_0) = \U_{m_{l}} \H_0 \V_{m_{l}}^T,
\]
where $\H_0 = \tau(\I_{m_{l-1}} \oplus \frac{(1- m_{l-1} \tau)}{r_l \tau} \I_{r_l})$. Clearly, $\H_0$ is PSD and $\tr(\H_0)=1$. Thus,
\[
 - \U_{m_{l}} \H_0 \V_{m_{l}}^T \in  \partial \|\A - \hat{\X}\|_2.
\]
As a result, $\0 \in \partial \|\A - \hat{\X}\|_2 + \tau \partial \|\hat{\X}\|_*$. Consequently, $\hat{\X}$ is a minimizer of the problem in \eqref{eqn:hproblrm}.
Compared with SVT in  the model \eqref{eqn:svtproblrm} which uses the tuning parameter $\tau$ as the thresholding value, the current model uses $\delta_l$ as the thresholding value. 

We also consider the following convex optimization problem:
\begin{equation} \label{eqn:hproblrm2}
\min_{\X \in \RB^{m\times n} } \; f(\X) \triangleq  \|\X-\A\|_* + \frac{1}{\tau} \|\X\|_2.
\end{equation}
%the solution of which can be directly obtained via the  problem \eqref{eqn:hproblrm}.  
Clearly, the minimizer of the problem is $\A-\hat{\X}$ where $\hat{\X}$ is the minimizer of the  problem \eqref{eqn:hproblrm}. 

\end{example}

%\section{Notes}
%

\begin{example} \label{exm:lowrank}
Finally, we  consider the following optimization problem:
\[
\min_{\X \in \RB^{n\times p}} \;  f(\X) \triangleq \nm \A\X - \B\nm,
\]
where $\A \in \RB^{m\times n}$ and  $\B \in \RB^{m\times p}$ are two given matrices.
This is a novel matrix low rank approximation problem. We will further discuss this problem in  Theorem~\ref{thm:ye}  of Chapter~\ref{ch:lowrank}.
%\begin{proof}[The Second Proof of Theorem~\ref{thm:ye}]
%Denote $f(\X)=\nm \A \X - \B\nm$.
Here we are concerned with  the use of  Theorem~\ref{thm:subdiff-uin} in solving the problem based on unitarily invariant norm loss functions.

Let  $\A= \U_r \Si_r \V_r^T$ be a condensed SVD of $\A$, and $\U_{-r}$ and $\V_{-r}$  be respective orthonormal complements of $\U_r$ and $\V_r$. Now $\B - \A \A^{\dag}\B  = \U_{-r} \U_{-r}^T \B$. Thus, when taking $\hat{\X}= \A^{\dag} \B$, one has
\[
\partial f(\hat{\X}) = \A^T \partial \nm \U_{-r} \U_{-r}^T \B \nm.
\]
Let $ \U_{0} \Si_0 \V_0^T =  \U^T_{-r} \B$ be a thin SVD of $\U_{-r}^T \B$,  $\D$ be a diagonal matrix, and $\phi$ be a symmetric gauge function associated with the norm $\nm \cdot \nm$.  It follows from Theorem~\ref{thm:subdiff-uin} that
\[ \partial \nm \U_{-r} \U_{-r}^T \B \nm = \mathrm{conv} \{ \U_{-r}  \U_{0} \D  \V_0^T :  \U_{0},  \V_0,  \dg(\D)  \in  \phi(\dg(\Si_0)) \}. \]
Thus, for any $\G \in  \partial \nm \U_{-r} \U_{-r}^T \B \nm$, it holds that $\A^T \G =\0$. This implies that $ \partial f(\hat{\X}) =\{\0\}$.  Hence,  $\0 \in \partial f(\hat{\X})$.  This implies that $\hat{\X}$ is a minimizer of the problem. In other words,
\[
\min_{\X \in \RB^{n\times p}} \nm \A \X - \B \nm = \nm \A \A^{\dag} \B - \B \nm.
\]
%\end{proof}
\end{example}

%We are interested in applying Theorem~\ref{thm:subdiff-uin} to prove Theorem~\ref{thm:ye}.

%\begin{proof}[The Proof of Theorem~\ref{thm:ye}] Denote $f(\X)=\nm \A \X - \B\nm$. Let  $\A= \U_r \Si_r \V_r^T$ be a %condensed SVD of $\A$, and $\U_{-r}$ and $\V_{-r}$  be respective orthonormal complements of $\U_r$ and $\V_r$. Now $\A %\A^{\dag}\B - \B = \U_{-r} \U_{-r}^T \B$. Thus, when taking $\hat{\X}= \A^{\dag} \B$, one has
%\[
%\partial f(\hat{\X}) = \A^T \partial \nm \U_{-r} \U_{-r}^T \B \nm = \{\0\},
%\]
%which implies $\0 \in \partial f(\hat{\X})$. The proof completes.
%\end{proof}

\chapter{Matrix Low Rank Approximation}
\label{ch:lowrank}

Matrix low rank approximation is very important, because  it has received wide applications in machine learning and data mining.
On the one hand,  many machine learning methods involve computing linear equation systems, matrix decomposition, matrix determinants,  matrix inverses, etc.  How to compute them efficiently is challenging in big data scenarios.
Matrix low rank approximation is a potentially powerful  approach for addressing  computational challenge. On the other hand,
many machine learning tasks can be modeled as  matrix low rank approximation problems such as matrix completion,  spectral clustering, and multi-task learning.

%%These computations  are intensive when performed exactly on large scale data matrices. Specifically, a dense SVD method on an $m\times n$ matrix takes $O(mn \min(m, n))$ time,  similarly for
%%matrix multiplication and linear regression  \citep{golub2012matrix}.
%
%The computation aforementioned  is intensive when performed exactly. Specifically,
%a dense SVD method on an $m{\times} n$ matrix takes $O(m n \min(m, n))$ time; similarly,
%matrix multiplication and linear regression are
%of the same order \citep{golub2012matrix}. Thus,
%there is rich work to study approximate matrix operations such as a randomized SVD to fill the use on large scale matrices~\citep{halko2011finding,Mahoney:2011,gu2015subspace}.

Approximate matrix multiplication is an inverse process of the matrix low rank approximation problem. 
Recently, many approaches to
approximate matrix multiplication \citep{drineas2006fast,sarlos2006improved,cohen1999approximating,magen2011low,kyrillidis2014approximate,kane2014sparser} have been developed. Meanwhile, they are used to obtain fast solutions
for the  $\ell_{2}$ regression and SVD problems \citep{Drineas:2006:SAL,drineas2011faster,nelson2013osnap,halko2011finding,clarkson2013low,martinsson2011randomized,woolfe2008fast}.
This  makes matrix low rank approximation also become increasingly popular in the theoretical computer science community~\citep{sarlos2006improved,drineas2006fast}.

In this chapter we first present some important theoretical results in matrix low rank approximation. We then discuss approximate matrix multiplication.
In the following chapter we are concerned with large scale matrix approximation. 
We will study randomized SVD and CUR approximation. They can be also cast into the 
matrix low rank approximation framework.

%
%A truncated SVD approach can find the best low-rank approximation to a data matrix. As we have see, SVD  leads us to a geometrical representation, which has little concrete meaning. This makes it difficult for us to understand and interpret the data in question.
%Therefore, it is great interest to represent a data matrix in terms of a small number part  of the matrix. This is a job that
%the CUR decomposition does.

\section{Basic Results}

%This section shows three important theorems that will be used in the later sections.

Usually, matrix low rank approximation is formulated as a least squares estimation problem based on the Frobenius norm loss.
However, \cite{tropp2015introduction}
pointed out that Frobenius-norm error bounds are not acceptable in most cases of practical interest. He even said ``Frobenius-norm error bounds are typically vacuous.''  Thus,  spectral norm as a loss function is also employed.
In this chapter, we present several basic  results, some of which  hold even for every unitarily invariant norm.

\begin{theorem}  \label{thm:ye} 
Let $\A \in \RB^{m \times n}$ and $\C \in \RB^{m\times c} $.  Then for any $\X \in \RB^{c \times n}$
and any unitarily invariant norm $\nm \cdot \nm$,
\[
\nm \A - \C \C^{\dag} \A  \nm \leq \nm \A - \C \X \nm .
\]
%Equivalently, $\X^\star = \C^\dag \A$ minimize the right-hand side.
In other words,
\begin{equation} \label{eqn:lsp-ui}
\C^{\dag} \A = \argmin_{\X \in \RB^{c\times n}} \nm \C \X - \A \nm.
\end{equation}
%where $\nm \cdot \nm$ is an arbitrary unitarily invariant norm.
\end{theorem}

As we have seen,  Theorem~\ref{thm:ye} was discussed in Example~\ref{exm:lowrank}, where the problem is solved via the subdifferentials of unitarily invariant norms given in Theorem~\ref{thm:subdiff-uin}.
Here, we present an alternative proof.

\begin{proof} 
Let $\E_1 = \A - \C \C^{\dag} \A$, $\E_2 = \C \C^{\dag} \A - \C \X$, and $\E = \E_1 + \E_2 = \A - \C \X $.
Since 
\[
\E_1^T \E_2 = \A^T (\I - \C \C^{\dag} ) \C (\C^\dag \A - \X) = \A^T \0 (\C^\dag \A - \X) = \0,
\]
we have $\E^T \E = \E_1^T \E_1 + \E_2^T \E_2$,
and thus $\lambda_i (\E_1) \leq \lambda_i (\E)$.
It then follows that $\sigma_i (\E_1) \leq \sigma_i (\E)$,
and thereby $\sib (\E_1) \prec_w \sib (\E)$.
It then follows from Theorems~\ref{thm:sg-m} and \ref{thm:ui-sg}
that 
\[
\nm \E_1 \nm \; \leq \; \nm \E \nm
\]
for any unitarily invariant norm $\nm \cdot \nm$.
\end{proof}

Recall that Problem \eqref{thm:ye} gives  an extension to the least squares problem \eqref{eqn:lsp} in Section~\ref{sec:pseudo}.
Theorem~\ref{thm:ye} shows that there is an identical solution  w.r.t.\ all unitarily invariant norm errors.
%Theorem~\ref{thm:ye} gives an extension to the least squares problem \eqref{eqn:lsp} in Section~\ref{sec:pseudo}.
The following theorem shows the solution of a more complicated problem.
However, the theorem holds only for the Frobenius norm loss.

\begin{theorem} \label{thm:solution_cur}
Let $\A \in \RB^{m \times n}$, $\C \in \RB^{m\times c} $, and $\R \in \RB^{r\times n}$.  
Then for all $\X \in \RB^{c \times r}$,
\[
\| \A - \C \C^{\dag} \A \R^\dag \R \|_F \leq \| \A - \C \X \R  \|_F.
\]
Equivalently, $\X^\star =  \C^{\dag} \A \R^\dag $ minimizes the following problem:
\begin{equation} \label{eqn:lsp-fni}
 \min_{\X \in \RB^{c\times n}} \| \C \X  \R - \A \|^2_F.
\end{equation}
\end{theorem}

\begin{proof}
Let $\E_1 = (\I_m  - \C \C^\dag) \A$, $\E_2 = \C \C^\dag \A (\I_n - \R^\dag \R)$,  $\E_3 = \C \C^\dag \A \R^\dag \R - \C \X \R$, and $\E = \E_1+\E_2+\E_3$. Then $\E_1+ \E_2 = \A -  \C \C^\dag \A  \R^\dag \R$ and $\E=\A - \C \X \R $. 
Since $\E_1^T \E_2 = \0$, $\E_3 \E_2^T = \0$, $\E_1^T \E_3 = \0$, %we have
%\[
%\E^T \E = 
%\]
it follows from the matrix Pythagorean theorem that
\[
\| \E \|_F^2 
= \| \E_1 \|_F^2 + \| \E_2 \|_F^2 + \| \E_3 \|_F^2 
= \| \E_1 + \E_2 \|_F^2 + \| \E_3 \|_F^2 .
\]  
Thus, $\| \E_1 + \E_2 \|_F^2 \leq \| \E \|_F^2$.
\end{proof}

%\begin{theorem}   \label{thm:proje} Given  an $m\times n$ real matrix $\A$ of rank $r$ ($\leq \min(m, n)$), let  $\A= \U \Si \V^T$ be a %full SVD of $\A$. Let $\U_k$ and $\V_k$ consist of the first $k$ columns of $\U$ and $\V$ respectively, and  $\Si_k$ be the first $k\times k$ principal submatrix of $\Si$.
%Then for all $m\times k$ column orthonormal matrices  $\Q$,
%\[
%\nm \A - \U_k \U_k \A \nm \leq \nm \A - \Q \Q^T \A \nm
%\]
%holds for every unitarily invariant norm $\nm \cdot \nm$. In other words,
%\begin{equation} \label{eqn:project}
%\U_k = \argmin_{\Q \in \RB^{m\times k}, \Q^T \Q = \I_k} \; \nm \A -\Q \Q^T \A \nm.
%\end{equation}
%\end{theorem}

\begin{theorem} \citep{EckartYoung:1936,mirsky:1960}  \label{thm:mlr} 
Given  an $m\times n$ real matrix $\A$ of rank $r$ ($\leq \min\{m, n\}$), let  $\A= \U \Si \V^T$ be the full SVD of $\A$. Define  $\A_k = \U_k \Si_k \V_k^T$, where $\U_k$ and $\V_k$ consist of the first $k$ columns of $\U$ and $\V$ respectively, and  $\Si_k$ is the first $k\times k$ principal submatrix of $\Si$.
Then for all $m\times n$ real matrices  $\B$ of rank at most $k$,
\[
\nm \A - \A_k \nm \leq \nm \A - \B \nm
\]
holds for all unitarily invariant norm $\nm \cdot \nm$. In other words,
\begin{equation} \label{eqn:low-rank}
\A_k = \argmin_{\B \in \RB^{m\times n}, \rk(\B) \leq k} \; \nm \A -\B \nm.
\end{equation}
\end{theorem}

Theorem~\ref{thm:mlr} shows that the rank $k$ truncated SVD produces the best rank $k$ approximation.
The theorem was originally proposed by  \cite{EckartYoung:1936} under the setting of the Frobenius norm, and generalized   to any unitarily invariant norms  by \cite{mirsky:1960}.

\begin{proof} For any $m\times n$ real matrix  $\B$ of rank at most $k$, we can write it as $\B = \Q \C$ where $\Q$ is an $m\times k$ column orthonormal matrix and $\C$ is some $k\times n$ matrix. Thus,
\[
\nm \A - \B \nm =  \nm \A - \Q \C \nm \geq \nm \A - \Q \Q^T \A \nm = \nm  \Q^{\bot} (\Q^\bot)^T \A \nm,
\]
where $\Q^{\bot}$ ($m\times(m{-}k)$) is the orthogonal complement of $\Q$. By Proposition~\ref{pro:032}, we have  $\sigma_i(\Q^{\bot} (\Q^\bot)^T \A) = \sigma_i((\Q^\bot)^T \A)  \geq \sigma_{k+i}$ for $i=1, \ldots, p-k$. This implies that
\[
\sib(\A- \A_k) = (\sigma_{k+i}, \sigma_{p}, 0, \ldots, 0)^T \prec_w \sib(\Q^{\bot} (\Q^\bot)^T \A).
\]
Hence,
$\nm \A - \B \nm \geq \nm \A- \A_k \nm$.
\end{proof}

The above proof procedure  also implies that  for all $m\times k$ column orthonormal matrices  $\Q$,
\[
\nm \A - \U_k \U_k^T \A \nm \leq \nm \A - \Q \Q^T \A \nm
\]
holds for every unitarily invariant norm $\nm \cdot \nm$.

When $k<r$, $\A_k$ is  called a truncated SVD of $\A$ and the closest  rank-$k$ approximation of $\A$.
Note that when the Frobenius norm is used, $\A_k$ is the unique minimizer of the problem in \eqref{eqn:low-rank}.
However, when other unitarily invariant norms are used, the case does not always hold. For example, let us take the spectral norm.
Clearly, if
\[
\tilde{\Si}  = \diag(\sigma_1-\omega \sigma_{k+1}, \sigma_2 - \omega \sigma_{k{+}1},  \ldots, \sigma_{k}-\omega \sigma_{k+1}, 0, \ldots, 0) \;
\]
for any  $\omega \in  [0, 1]$,
then $\U \tilde{\Si} \V^T$ is also a minimizer of the corresponding problem.

\begin{theorem} \label{thm:qb}
Given a matrix $\A \in \RB^{m\times n}$ and a column orthonormal matrix  $\Q \in \RB^{m\times p}$, let $\B_k$ be the rank-$k$ truncated SVD of $\Q^T \A$ for $1 \leq k \leq p$.  Then $\B_k$  is an optimal solution of   the following problem:
\begin{equation} \label{eqn:082}
 \min_{\B \in \RB^{l{\times} n}, \rk(\B) \leq k} \;  \|\A - \Q \B \|_F^2  =  \|\A - \Q \B_k\|_F^2.
\end{equation}
\end{theorem}
\begin{proof} Note that 
$(\A -\Q \Q^T \A)^T (\Q \B - \Q \Q^T \A) = \0$,  
so 
\begin{align*}
\|\A- \Q \B\|_F^2 & = \|\A - \Q \Q^T \A\|_F^2 + \|\Q\B - \Q \Q^T \A|_F^2  \\
& = \|\A - \Q \Q^T \A\|_F^2 + \|\B - \Q^T \A|_F^2. 
\end{align*}
The result of the theorem follows from Theorem~\ref{thm:mlr}. 
\end{proof}

%Let  $\Pi_{\C, k}^{\xi}(\A) \in \RB^{m\times n}$ be the best approximation to $\A$ within the column space of $\C$
%that has rank at most $k$.  Then Theorem~\ref{thm:qb} says that $\prod_{\Q, k}^{F}(\A) = \Q \B_k$.
Theorem~\ref{thm:qb} is a variant of Theorem~\ref{thm:mlr} and of Theorem~\ref{thm:ye}. Unfortunately, $\B_k$ might not be the solution to the above problem in every unitarily invariant norm, even in the spectral norm error. The reason is that the matrix Pythagorean identity  hods only for the Frobenius norm (see Theorem~\ref{thm:pytho}).

However, \citet{tropp2015introduction} pointed out that Frobenius-norm error bounds are not acceptable in most cases of practical interest. He even said ``Frobenius-norm error bounds are typically vacuous'' \citep{tropp2015introduction}.  The following theorem was proposed by \cite{gu2015subspace}, which relates the approximation error in the Frobenius norm to that in the spectral norm.

\begin{theorem} \citep{gu2015subspace}  \label{thm:f-s} Given any matrix $\A \in \RB^{m\times n}$, let  $p=\min\{m, n\}$ and $\B$ be a matrix with rank at most $k$ such that
\[
\|\A-\B\|_F \leq \sqrt{\eta^2 + \sum_{j=k+1}^p \sigma_j^2(\A)}
\]
for some $\eta \geq 0$. Then we must have $\sqrt{\sum_{j=1}^k (\sigma_j(\A) - \sigma_j(\B))^2 } \leq \eta$ and
\[
\|\A - \B\|_2 \leq \sqrt{\eta^2 + \sigma_{k+1} ^2(\A)}.
\]
\end{theorem}

\begin{proof} By Proposition~\ref{pro:032}-(2), we have
\[
\sigma_{i+k}(\A) \leq \sigma_i(\A-\B) + \sigma_{k+1}(\B) = \sigma_i(\A-\B) \; \mbox{ for } \; i \in [p-k]
\]
due to $\rk(\B) \leq k$.  It then follows that 
\begin{align*}
\|\A-\B\|_F^2  &= \sum_{i=1}^p \sigma_i^2(\A-\B) \geq \sigma_1^2(\A-\B) + \sum_{i=2}^{p-k} \sigma_i^2(\A-\B) \\
  & \geq \sigma_1^2(\A-\B) + \sum_{i=2}^{p-k} \sigma_{i+k}^2 (\A).
\end{align*}
We thus obtain 
\[
\|\A-\B\|_2^2 = \sigma_1^2(\A-\B) \leq \eta^2 + \sigma^2_{k+1} (\A). 
\]
Additionally,  it follows from Theorem~\ref{thm:d-b-sv} that 
\[
\sum_{i=1}^k (\sigma_i(\A)-\sigma_i(\B))^2 + \sum_{j=k+1}^p \sigma_j^2(\B) \leq \|\A-\B\|_F^2\leq \eta^2+  \sum_{j=k+1}^p \sigma_j^2(\A),
\]
which leads to the result. 
\end{proof}

Let us apply Theorem~\ref{thm:f-s} to Theorem~\ref{thm:qb} to establish a spectral norm error bound.    
It follows from Theorem~\ref{thm:qb} that 
\[
\|\A-\A_k\|_F \leq \|\A - \Q \B_k\|_F \leq \|\A - \Q \Q^T \A_k\|_F.
\]
Consider that
\begin{align*}
  \|\A - \Q \Q^T \A_k\|_F^2 &= \|\A- \A_k + \A_k- \Q \Q^T \A_k\|_F^2 \\
 & = \|(\I_m - \Q \Q^T) \A_k\|_F^2 + \|\A-\A_k\|_F^2
\end{align*}
due to $(\A-\A_k) \A_k^T (\I_m - \Q \Q^T)=\0$.  Thus,
\[
\|\A - \Q \B_k\|_F^2 \leq  \|(\I_m - \Q \Q^T) \A_k\|_F^2 + \sum_{i=k+1}^n  \sigma_i^2(\A).
\]
By Theorem~\ref{thm:f-s}, we have that
\[
\|\A - \Q \B_k\|_2^2 \leq  \|(\I_m - \Q \Q^T) \A_k\|_F^2 + \sigma_{k+1}^2(\A),
\]
which can give an error bound in the spectral norm.

\section{Approximate Matrix Multiplication}

Given matrices $\A \in \RB^{n \times d}$ and $\B \in \RB^{n \times p}$,
it is well known that the complexity of computing $\A^T \B$ is
$O(dnp)$. Approximate matrix multiplication  aims to obtain a matrix $\C \in \RB^{d \times p}$ with $o(d n p)$  time
complexity such  that for a small $\varepsilon >0$, 
\[
\|\A^T \B-\C\|\leq\varepsilon\|\A\|\|\B\|. 
\]
This shows that  approximate matrix multiplication can be viewed as an inverse process of the conventional matrix low rank approximation problem.  
%Many methods have been proposed for approximating  matrix multiplication with
%time complexity $o(mnp)$  based on the Frobenius norm error bound~\citep{drineas2006fast,clarkson2013low,drineas2011faster,sarlos2006improved,kane2014sparser}. However, there is much fewer work on  matrix multiplication approximation under the error bound w.r.t.\  other norms such as the spectral norm  \citep{magdon2011using,magen2011low,kyrillidis2014approximate}.

Approximate matrix multiplication is a potentially important approach for fast matrix multiplication \citep{drineas2006fast,clarkson2009numerical,cohen1999approximating,kane2014sparser,drineas2011faster,nelson2013osnap,clarkson2013low}.  It is   the foundation of approximate least square methods and matrix low rank approximation methods~\citep{sarlos2006improved,halko2011finding,kyrillidis2014approximate,martinsson2011randomized,woolfe2008fast,magdon2011using,magen2011low,cohen1999approximating,kane2014sparser,drineas2011faster,nelson2013osnap,clarkson2013low}. Moreover, it can be also used in large scalable  k-means clustering \citep{cohen2014dimensionality},  approximate leverage scores~\citep{Petros2011Fast}, etc.

Most of work for matrix approximations is based on  error bounds w.r.t.\  the Frobenius norm ~\citep{drineas2006fast,sarlos2006improved,cohen1999approximating,kane2014sparser,drineas2011faster,nelson2013osnap,clarkson2013low}. 
In contrast, there is a few work based on spectral-norm error bounds  \citep{halko2011finding,kyrillidis2014approximate,martinsson2011randomized,woolfe2008fast,magdon2011using,magen2011low}. As we have mentioned earlier, spectral-norm error bounds are also of great interest.

In  approximate matrix multiplication, 
oblivious subspace embedding matrix is a key ingredient. For example, gaussian matrix and random sign matrix are oblivious matrix. However, leverage score sketching matrix depends on data matrix, hence, it is not an oblivious subspace embedding matrix.

\begin{definition}\citep{woodruff2014sketching}
	Given $\varepsilon>0$ and $\delta>0$, let $\Pi$ be a distribution on $l \times n$ matrices, where $l$ relies on $n$, $d$, $\varepsilon$ and $\delta$. Suppose that
	with probability at lest $1-\delta$, for any fixed $n \times d$ matrix $\A$, a matrix $\S$ drawn from distribution $\Pi$
	is a $(1+\varepsilon)$ $\ell_{2}$-subspace embedding for $\A$, that is, for all $\x \in \RB^{d}$,
	$\|\S\A\x\|_{2}^{2}=(1\pm\varepsilon)\|\A\x\|_{2}^{2}$ with probability
	$1-\delta$. Then we call $\Pi$ an $(\varepsilon,\delta)$-oblivious $\ell_{2}$-subspace
	embedding, 
\end{definition}
Recently, \cite{CohenNelsonWoodrull}  proved optimal approximate matrix multiplication in terms of stable rank by using subspace embedding~\citep{BSS:2014}. 
\begin{theorem}\citep{CohenNelsonWoodrull}
	Given $\varepsilon$, $\delta\in(0,1/2)$, let $\A$ and $\B$ be two conforming matrices, and $\Pi$ be a $(\varepsilon,\delta)$ subspace embedding for the $2\tilde{r}$-dimensional subspace, where $\tilde{r}$ is the maximum of the stable ranks of $\A$ and $\B$. Then,
	\[
	||(\Pi\A)^{T}(\Pi\B)-\A^{T}\B||\leq\varepsilon||\A||||\B||
	\] 
	holds with at least $1-\delta$.
	%\citep{BSS:2014}  
\end{theorem}

To analyze approximate matrix multiplication with the Frobenius error, \cite{kane2014sparser} introduced the JL-moment property.

\begin{definition}
	A distribution $\mathcal{D}$ over $\R^{n \times d}$ has the $(\varepsilon, \delta, \ell)$-JL moment property if for all $\x \in \R^{d}$ with $\|\x\|_2 = 1$,
	\[
	\mathbb{E}_{\Pi\sim\mathcal{D}}\left| \|\Pi\x\|_2^{2}-1\right|^{\ell} \leq \varepsilon^{\ell}\cdot\delta
	\]
\end{definition}

Based on the JL-moment property, these is an approximate matrix multiplication method with the Frobenius error. 

\begin{theorem}
	Given $\varepsilon$, $\delta\in(0,1/2)$, let $\A$ and $\B$ be two conforming matrices, and $\Pi$ be a matrix satisfying the $(\varepsilon, \delta, \ell)$-JL moment property for some $\ell\geq 2$. Then,
	\[
	||(\Pi\A)^{T}(\Pi\B)-\A^{T}\B||_{F}\leq\varepsilon||\A||_{F}||\B||_{F}
	\] 
	holds with at least $1-\delta$.
\end{theorem}
Note that both the subspace embedding property and the JL moment property have close relationships. More specifically, they  can be converted into each other~\citep{kane2014sparser}. 

There are other methods, which do  not use subspace embedding matrices, in the literature. \cite{magen2011low} gave a method based on columns selection. \cite{bhojanapallitighter} proposed a new method with sampling and alternating minimization to directly compute a low-rank approximation to the product of two given matrices.  

For low-rank matrix approximation in the streaming model, \cite{clarkson2009numerical} gave the near-optimal space bounds by the sketches.  \cite{liberty2013simple} came up with a deterministic streaming algorithm, with an improved analysis studied by \cite{ghashami2014relative} and space lower bound obtained by \cite{woodruff2014low}.

\chapter{Large-Scale Matrix Approximation} %via Randomized Approximation}
\label{ch:lsma}

In this chapter we discuss  fast computational methods of the SVD, kernel methods, and CUR decomposition via randomized approximation.
The goal is to  make the  matrix factorizations fill the use on large scale data matrices. 

It is notoriously difficult to compute SVD because  the exact SVD of an  $m\times n$ matrix takes $\OM (m n \min\{m, n\})$ time.
Fortunately, many machine learning methods such as latent semantic indexing \citep{deerwester1990lsa},
spectral clustering \citep{shi2000normalized},
manifold learning \citep{tenenbaum2000global,Belhumeur:2003}
are interested in only the top singular value triples.
The Krylov subspace method computes the top $k$ singular value triples in $\tilde\OM (mnk)$ time \citep{saad2011numerical,musco15stronger},
where the $\tilde\OM$ notation hides the logarithm factors and the data dependent condition number.
If a low precision solution suffices, the time complexity can be even lower.  
Here 
we will make main attention on  randomized approximate algorithms that demonstrate high scalability.
Randomized algorithms are a feasible approach for large scale machine learning models \citep{rokhlinSIAM:2009,Mahoney:2011,TuICML:2014}.
In particular,  we will consider randomized SVD methods~\citep{halko2011finding}.

In contrast to the randomized SVD which is based on random projection, the CUR approximation mainly employs column selection. 
Column selection  has been extensively studied in the theoretical computer science (TCS)
and numerical linear algebra (NLA) communities.
The work in TCS mainly focuses on choosing good columns by randomized algorithms with provable error bounds
\citep{FriezeJACM:2004,deshpande2006matrix,DrineasCUR:2008,deshpande2010efficient,boutsidis2011near,Guruswami2012optimal}.
The focus in NLA is then on deterministic algorithms, especially the rank-revealing QR factorizations, that select columns by pivoting rules
\citep{foster1986rank,chan1987rank,stewart1999four,bischof1991structure,hong1992rank,chandrasekaran1994rank,gu1996efficient,berry2005algorithm}. 
%Note that  the optimal algorithm  proposed by \citet{Guruswami2012optimal} attains  the lower bound. However, 
%this  algorithm is quite inefficient in comparison with the near-optimal algorithm of \citet{boutsidis2011near}.
%The near-optimal algorithm consists of three steps:
%the approximate SVD via random projection \citep{halko2011finding},
%the dual set sparsification algorithm \citep{BSS:2014,boutsidis2011near},
%and the adaptive sampling algorithm \citep{deshpande2006matrix}.

%The CUR decomposition is also an extension of the novel Nystr\"{o}m approximation to a general  matrix. 
%The  Nystr\"{o}m  methods approximate an SPSD matrix only using a subset of its columns,
%so they can alleviate computation and storage costs when the SPSD matrix in question is large in size.
%Thus, the \nystrom methods have been extensively used in the machine learning community.
%For example, they have been applied to Gaussian processes \citep{williams2001using},
%kernel classification \citep{zhang2008improved,jin2013improved}, spectral clustering \citep{fowlkes2004spectral},
%kernel PCA and manifold learning \citep{talwalkar2008large,zhang2008improved,zhang2010clustered}, determinantal processes \citep{affandi2013nystrom}, etc.

\section{Randomized SVD}
\label{sec:rsvd}

All the randomized SVD algorithms essentially have the same idea:
first draw a random projection matrix $\Omeb \in \RB^{n\times c}$, 
then form the sketch $\C = \A \Omeb \in \RB^{m\times c}$ and compute its orthonormal bases $\Q \in \RB^{m\times c}$,
and finally compute a rank $k$ matrix $\X \in \RB^{c\times n}$ such that $\| \A - \Q \X \|_{\xi}^2$ is small compared to $\| \A - \A_k\|_{\xi}^2$.
Here  $\|\cdot \|_{\xi}$ denotes either the Frobenius norm or the spectral norm.

%Let us go through the these steps.

The following lemma is the foundation  in theoretical analysis of  the randomized SVD~\citep{halko2011finding,gu2015subspace}.

\begin{lemma}  \label{lem:rsvd00} Let $\A \in \RB^{m\times n}$ be a given matrix, and $\Z \in \RB^{n\times k}$ be column orthonormal. Let $\Omeb \in \RB^{n\times c}$ be any matrix such that $\rk(\Z^T \Omeb) = \rk(\Z) =k$, and define $\C  = \A \Omeb \in \RB^{m\times c}$ . Then
\[
\|\A -  \Pi_{\C, k}^{\xi}(\A) \|_{\xi}^2 \leq \|\E \|_{\xi}^2 + \|\E \Omeb (\Z^T \Omeb)^{\dag} \|_{\xi}^2,
\]
where $\E = \A {-} \A \Z \Z^T$, and  %$\|\cdot \|_{\xi}$ denotes either the spectral norm or the Frobenius norm, and   
$\Pi_{\C, k}^{\xi}(\A) \in \RB^{m\times n}$ denotes the best approximation to $\A$ within the column space of $\C$
that has rank at most $k$ w.r.t.\ the norm $\|\cdot\|_{\xi}$ loss. 
\end{lemma}

\begin{proof} In terms of definition of $ \Pi_{\C, k}^{\xi}(\A)$, we have
\[
\|\A -  \Pi_{\C, k}^{\xi}(\A) \|_{\xi}^2 \leq \|\A -\X\|_{\xi}^2
\] 
for all matrices $\X \in \RB^{m\times n}$ of rank at most $k$ in the column space of $\C$. Obviously, $\C (\Z^T \Omeb)^{\dag} \Z^T$ is such a matrix. Thus,
\begin{align*}
\|\A -  \Pi_{\C, k}^{\xi}(\A) \|_{\xi}^2  & \leq \|\A - \C (\Z^T \Omeb)^{\dag} \Z^T\|_{\xi}^2 \\
& = \|\A - \A \Z \Z^T + \A \Z \Z^T -  \C  (\Z^T \Omeb)^{\dag} \Z^T \|_{\xi}^2  \\
& = \|\E + (\A\Z \Z^T - \A) \Omeb  (\Z^T \Omeb)^{\dag}  \Z^T\|_{\xi}^2 \\
& = \| \E + \E \Omeb (\Z^T \Omeb)^{\dag} \Z^T\|_{\xi}^2.
\end{align*}
Here we use the fact that $\Z^T \Omeb (\Z^T \Omeb)^{\dag}=\I_k$ because $\rk(\Z^T \Omeb)=k$.  Consider that 
\[
\E \Omeb (\Z^T \Omeb)^{\dag} \Z^T \E^T =  \E \Omeb (\Z^T \Omeb)^{\dag}  \Z^T (\A^T - \Z \Z^T \A^T) = \0.  
\]
The theorem follows from Theorem~\ref{thm:pytho}. 
\end{proof}

Consider the rank-$k$ truncated SVD $\A_k = \U_k \Si_k \V_k^T$. Then we can write $\A$ as
\[
\A = \A \V_k \V_k^T + (\A - \A_k). 
\] 
Let $\Z=\V_k$ and $\E=\A-\A_k$ in Lemma~\ref{lem:rsvd00}. Then the following theorem is an immediate corollary of Lemma~\ref{lem:rsvd00}. 

\begin{theorem}  \label{thm:test} Let $\A=\U \Si \V^T$ be the full SVD of $\A \in \RB^{m\times n}$,  fix $k\geq 0$, and let $\A_k=\U_k \Si_k \V_k^T$ be the best at most rank $k$ approximation of $\A$. Choose a test matrix  $\Omeb$ and construct the sketch $\C  = \A \Omeb$.  Partition $\Si= \begin{bmatrix} \Si_k & \0 \\ \0 & \Si_{-k} \end{bmatrix}$ and $\V=[\V_k, \V_{-k}]$. 
%where $\Si_k$ and $\Si_$ are $k\times k$ and $(n-k) \times (n-k)$. 
Define $\Omeb_1 = \V_k^T \Omeb$ and $\Omeb_2 = \V_{-k}^T \Omeb$. Assume that $\Omeb_1$ has full row rank. Then
\[
\| (\I_m - \C \C^{\dag}) \A  \|^2_{\xi} \leq \|\A -  \Pi_{\C, k}^{\xi}(\A) \|_{\xi}^2  \leq  \|\Si_{-k} \|^2_{\xi} + \|\Si_{-k} \Omeb_2 \Omeb_1^{\dag} \|_{\xi}^2.
\]
%where $\|\cdot \|_{\xi}$ denotes either the spectral norm or the Frobenius norm.
\end{theorem}

In Lemma~\ref{lem:rsvd00} and Theorem~\ref{thm:test}, the condition $\rk(\V_k^T \Omeb) = \rk(\V_k)=k$
is essential for an effective randomized SVD algorithm. An idealized case for meeting this condition is that $\rg(\V_k)\subset  \rg(\Omeb)$. In this case, the randomized SVD degenerates an exact truncated SVD procedure.  Thus, the above condition aims to relax  this idealized case. 
Moreover, the key for an effective randomized SVD is to select a test matrix $\Omeb$ such that  the condition $\rk(\V_k^T \Omeb) = \rk(\V_k)=k$ holds as much as possible.   Lemma~\ref{lem:rsvd00} and Theorem~\ref{thm:test} are also fundamental in random column selection~\citep{boutsidis2011near}.

%%%%%%%%%%%%%%%%%%%%%%%%%%
\subsection{Randomized SVD: Frobenius Norm Bounds}

In this subsection, we describe two randomized SVD algorithms which have $(1+\epsilon)$ relative-error bound.

{\bf Random Projection.}
In order to reduce  computational expenses, randomized algorithms  \citep{FriezeJACM:2004,VempalaBook:2000} have been introduced to truncated SVD and low-rank approximation. 
The Johnson \& Lindenstrauss (JL) transform \citep{JohnsonLindenstrauss:1984,Dasgupta:2003} is known to keep isometry in expectation or with high probability.
\citet{halko2011finding,boutsidis2011near} used the JL transform for sketching and showed relative-error bounds.
However, the Gaussian test matrix is dense and cannot efficiently apply to matrices.
Several improvements have been proposed to make the sketching matrix sparser;
see the review \citep{woodruff2014sketching} for the complete list of the literature.
In particular, the count sketch \citep{clarkson2013low} applies to $\A$ in only $\OM (\nnz (\A))$ time
and exhibits very similar properties as the JL transform.
Specifically, \cite{woodruff2014sketching} showed that an $m\times \OM (k/\epsilon)$ sketch $\C = \A \Omeb$ can be obtained in $\OM(\nnz (\A))$ time
and 
\begin{equation} \label{eq:randsvd_prototype}
\min_{\rk (\X) \leq k} \, \big\| \A - \Q \X \big\|_F^2
\; \leq \; (1+\epsilon)\, \| \A - \A_k\|_F^2
\end{equation}
holds with high probability.

{\bf The Prototype Algorithm.}
\citet{halko2011finding} proposed to directly solve the left-hand side of \eqref{eq:randsvd_prototype},
which has closed-form solution $\X^\star = (\Q^T \A)_k$.
This leads to the prototype algorithm shown in Algorithm~\ref{alg:rsvd_prototype}.
The optimality of $\X^\star$ is given in  Theorem~\ref{thm:qb}.

\begin{algorithm}[tb]
   \caption{Randomized SVD: The Prototype Algorithm.}
   \label{alg:rsvd_prototype}
\algsetup{indent=2em}
%\begin{small}
\begin{algorithmic}[1]
   \STATE  {\bf Input:} a  matrix $\A \in \RB^{m\times n}$ with $m \geq n$,  target rank $k$, the size of sketch $c$ where $0< k\leq c <n$;
   \STATE Draw a sketching matrix $\Omeb \in \RB^{n\times c}$, e.g.\ a Gaussian test matrix or a count sketch
   \STATE Compute $\C = \A \Omeb \in \RB^{m\times c}$ and its orthonormal bases $\Q \in \RB^{m\times c}$;
   \STATE Compute the rank $k$ truncated SVD: $\Q^T \A \approx \bar\U_k \tilde\Si_k \tilde\V_k^T$;
   \RETURN $\tilde\U_k = \Q \bar\U_k$, $\tilde\Si_k$, $ \tilde\V_k$---an approximate rank-$k$ truncated SVD of $\A$.
\end{algorithmic}
%\end{small}
\end{algorithm}

%\begin{theorem} \label{thm:qb}
%Given a matrix $\A \in \RB^{m\times n}$, a column orthonormal matrix  $\Q \in \RB^{m\times c}$, and a target rank $k$ ($\leq c$).
%Then $\X^\star = (\Q^T \A)_k$ is the optimal solution to
%\begin{equation} \label{eqn:082}
% \min_{\rk(\X) \leq k} \;  \|\A - \Q \X \|_F^2 .
%\end{equation}
%\end{theorem}
%\begin{proof} 
%Note that
%$(\A -\Q \Q^T \A)^T (\Q \X - \Q \Q^T \A) = \0$,
%so
%\begin{align*}
%\|\A- \Q \X\|_F^2 & = \|\A - \Q \Q^T \A\|_F^2 + \|\Q\X - \Q \Q^T \A\|_F^2  \\
%& = \|\A - \Q \Q^T \A\|_F^2 + \|\X - \Q^T \A\|_F^2.
%\end{align*}
%The result of the theorem follows from Theorem~\ref{thm:mlr}.
%\end{proof}

The prototype algorithm is not time efficient because the matrix product $\Q^T \A$ costs $\OM (m n c)$ time, 
which is not lower than the exact solutions.
Nevertheless, the prototype algorithm is still useful in large-scale applications because it is pass-efficient---it 
goes only two passes through $\A$.

{\bf Faster Randomized SVD.}
The bottleneck of the prototype algorithm is the matrix product in computing $\X^\star$.
Notice that \eqref{eqn:082} is a strongly over-determined system,
so it can be approximately solved by once more random projection.
Let $\PP = \PP_1 \PP_2 \in \RB^{m\times p}$ be another random projection matrix, where $\PP_1$ is a count sketch and $\PP_2$ is a JL transform matrix.
Then we solve 
\[
\tilde\X \; = \; \min_{\rk(\X) \leq k} \;  \|\PP^T (\A - \Q \X) \|_F^2
\]
instead of \eqref{eqn:082},
and $\tilde\X$ has closed-form solution 
\[
\tilde\X = \tilde\R^\dag (\tilde\Q^T \PP^T \A)_k ,
\]
where $\tilde\Q \tilde\R$ be the economy size QR decomposition of $(\PP^T \Q) \in \RB^{p\times c}$.
Finally, the rank $k$ matrix $\Q \tilde\X$ is the obtained approximation to $\A$,
and its SVD can be very efficiently computed.
\citet{clarkson2013low,woodruff2014sketching}
showed that 
\[
\big\| \A - \Q \tilde\R^\dag (\tilde\Q^T \PP^T \A)_k \big\|_F^2
\; \leq \; (1+\epsilon) \, \|\A - \A_k\|_F^2
\]
for a large enough $p$,
and the overall time cost is $\OM( \nnz (\A) + (m+n) \poly (k/\epsilon) )$.

\subsection{Randomized SVD: Spectral Norm Bounds}

The previous section shows that the approximate truncated SVD can be computed highly efficiently, with the $(1{+}\epsilon)$ Frobenius relative-error  guaranteed.
The Frobenius norm bound tells that the total elementwise distance is small,
but it does not inform us the closeness of their singular vectors.
Therefore, we need spectral norm bounds or  even stronger principal angle bounds;
here we only consider the former.
We seek to find an $m\times k$ column orthogonal matrix $\tilde\U$ such that
\[
\big\| \A - \tilde\U \tilde\U^T \A \big\|_2^2
\; \leq \; \eta \|\A - \A_k\|_2^2,
\]
where $\eta$ will be specified later.

{\bf The Prototype Algorithm.}
Unlike the Frobenius norm bound,
the prototype algorithm is unlikely to attain a constant factor bound (i.e.,\ $\eta$ is independent of $m$, $n$),
letting alone the $1+\epsilon$ bound.
It is because the lower bounds \citep{witten2013randomized,boutsidis2011near} showed that if $\Omeb \in \RB^{n\times c}$ in Algorithm~\ref{alg:rsvd_prototype} is the Gaussian test matrix
or any column selection matrix,
the order of $\eta$ must be at least ${n/c}$.
We apply Gu's theorem \citep{gu2015subspace} (Theorem~\ref{thm:f-s}) to obtain an $\OM (n)$-factor spectral norm bound,
and then introduce iterative algorithms with the $(1{+}\epsilon)$ spectral norm bound.

%\begin{lemma} [Gu] \label{lem:f-s} 
%Given any matrix $\A \in \RB^{m\times n}$, let  $p=\min\{m, n\}$ and $\B$ be a matrix with rank at most $k$ such that
%\[
%\|\A-\B\|_F \leq \sqrt{\eta^2 + \sum_{j=k+1}^p \sigma_j^2(\A)}
%\]
%for some $\eta \geq 0$. Then we must have $\sqrt{\sum_{j=1}^k (\sigma_j(\A) - \sigma_j(\B))^2 } \leq \eta$ and
%\[
%\|\A - \B\|_2 \leq \sqrt{\eta^2 + \sigma_{k+1} ^2(\A)}.
%\]
%\end{lemma}

%\begin{proof} By Proposition~\ref{pro:032}-(2), we have
%\[
%\sigma_{i+k}(\A) \leq \sigma_i(\A-\B) + \sigma_{k+1}(\B) = \sigma_i(\A-\B) \; \mbox{ for } \; i \in [p-k]
%\]
%due to $\rk(\B) \leq k$.  It then follows that
%\begin{align*}
%\|\A-\B\|_F^2  &= \sum_{i=1}^p \sigma_i^2(\A-\B) \geq \sigma_1^2(\A-\B) + \sum_{i=2}^{p-k} \sigma_i^2(\A-\B) \\
%  & \geq \sigma_1^2(\A-\B) + \sum_{i=2}^{p-k} \sigma_{i+k}^2 (\A).
%\end{align*}
%We thus obtain
%\[
%\|\A-\B\|_2^2 = \sigma_1^2(\A-\B) \leq \eta^2 + \sigma^2_{k+1} (\A).
%\]
%Additionally,  it follows from Theorem~\ref{thm:d-b-sv} that
%\[
%\sum_{i=1}^k (\sigma_i(\A)-\sigma_i(\B))^2 + \sum_{j=k+1}^p \sigma_j^2(\B) \leq \|\A-\B\|_F^2\leq \eta^2+  \sum_{j=k+1}^p \sigma_j^2(\A),
%\]
%which leads to the result.
%\end{proof}
Let $\tilde\U_k$, $\tilde\Si_k$, and $\tilde\V_k$ be the outputs of Algorithm~\ref{alg:rsvd_prototype}.
We have that
\begin{eqnarray*}
\big\| \A - \tilde\U_k \tilde\U_k^T \A \big\|_F^2
& \leq & \big\| \A - \tilde\U_k \tilde\Si_k \tilde\V_k^T \big\|_F^2  \\
& = &  \big\| \A - \Q \X^\star \big\|_F^2
\; \leq \; (1+\epsilon) \, \|\A - \A_k \|_F^2,
\end{eqnarray*}
where the first inequality follows from Theorem~\ref{thm:ye},
the equality follows from the definitions,
and the second inequality follows from \eqref{eq:randsvd_prototype} provided that $c = \OM (k/\epsilon)$
and $\Omeb$ is the Gaussian test matrix or the count sketch.
We let $\epsilon=1$ and $c = \OM (k)$ and apply Theorem~\ref{thm:f-s} to obtain
\begin{eqnarray} \label{eq:prototype_spectral}
\big\| \A - \tilde\U_k \tilde\U_k^T \A \big\|_2^2
& \leq & \|\A - \A_k \|_2^2 + \|\A - \A_k \|_F^2 \nonumber \\
& \leq & (n-k+1) \|\A - \A_k \|_2^2 .
\end{eqnarray}
Here the second inequality follows from that 
$\|\A - \A_k \|_F^2=\sum_{i=k+1}^n \sigma_i^2 \leq (n-k) \sigma_{k+1}^2 = (n-k) \|\A - \A_k\|_2^2$.
To this end, we have shown that the prototype algorithm \ref{alg:rsvd_prototype} satisfies $\OM(n)$-factor spectral norm bound.
However, the result itself has little meaning.

{\bf The Simultaneous Power Iteration}
can be used to refine the sketch \citep{halko2011finding,gu2015subspace}.
The algorithm is described in Algorithm~\ref{alg:power_method} and analyzed in the following.
Let $\Omeb \in \RB^{n\times c}$ be a Gaussian test matrix or count sketch
and $\B = (\A \A^T)^t \A$.
Let us take $\B$ instead of $\A$ as the input of the prototype algorithm~\ref{alg:rsvd_prototype} and obtain the approximate left singular vectors $\tilde\U_k$.
It is easy to verify that $\tilde\U_k$ is the same to the output of Algorithm~\ref{alg:power_method}.
We will show that when $t = \OM (\frac{ \log n }{ \epsilon })$,
\begin{eqnarray} \label{eq:power_spectral}
\big\| \A - \tilde\U_k \tilde\U_k^T \A \big\|_2^2
\; \leq \; (1+\epsilon) \|\A - \A_k\|_2^2.
\end{eqnarray}
To show this result, we need the lemma of \cite{halko2011finding}.

\begin{lemma} [Halko, Martinsson, \& Tropp] \label{lem:power_method}
Let $\A$ be any matrix and $\U$ have orthonormal columns. Then for any positive integer $t$,
\[
\big\| (\I - \U \U^T) \A \big\|_2
\; \leq \;
\big\| (\I - \U \U^T) (\A \A^T)^t \A  \big\|_2^{1/(2t+1)}.
\]
\end{lemma}

By Lemma~\ref{lem:power_method}, we have that
\begin{eqnarray*}
\big\| (\I - \tilde\U_k \tilde\U_k^T) \A \big\|_2^2
& \leq & \big\| (\I - \tilde\U_k \tilde\U_k^T) \B  \big\|_2^{2/(2t+1)} \\
& \leq & (n-k+1)^{1/(2t+1)} \sigma_{k+1}^{2/(2t+1)} (\B) \\
& = & (1+\epsilon) \sigma_{k+1}^2 (\A) . 
\end{eqnarray*}
Here the second inequality follows from \eqref{eq:prototype_spectral} and the definitions of $\B$ and $\tilde\U_k$,
and we show the equality in the following.
Let $2t+1 = \frac{ \log (n-k+1) }{ 0.5\epsilon }$.
We have that $\frac{1}{2t+1} \log (n-k+1) = 0.5\epsilon \leq \log (1+\epsilon)$, 
where the inequality holds for all for all $\epsilon \in [0, 1]$.
Taking the exponential of both sides, we have $(n-k+1)^{1/(2t+1)} \leq 1+\epsilon$.
Finally, \eqref{eq:power_spectral} follows from that $\sigma_{k+1}^2 (\A) = \|\A - \A_k\|_2^2$.

\begin{algorithm}[tb]
   \caption{Subspace Iteration Methods.}
   \label{alg:power_method}
\algsetup{indent=2em}
%\begin{small}
\begin{algorithmic}[1]
   \STATE  {\bf Input:} any matrix $\A \in \RB^{m\times n}$, the target rank $k$, the size of sketch $c$ where $0< k\leq c <n$;
   \STATE Generate an $n \times c$ Gaussian test matrix $\Omeb$ and perform sketching $\C^{(0)} = \A \Omeb$;
   \FOR{$i=1$ to $t$}
    \STATE Optional: orthogonalize $\C^{(i-1)}$;
    \STATE Compute $\C^{(i)} = \A \A^T \C^{(i-1)}$;
   \ENDFOR
   \STATE {\bf The Power Method}: orthonalize $\C^{(t)}$ to obtain $\Q \in \RB^{m\times c}$;
   \STATE {\bf The Krylov Subspace Method}: orthonalize $\K = [\C^{(0)}, \cdots , \C^{(t)}]$ to obtain $\Q \in \RB^{m\times (t+1)c}$;
   \STATE Compute the rank $k$ truncated SVD: $\Q^T \A \approx \bar\U_k \tilde\Si_k \tilde\V_k^T$;
   \RETURN $\tilde\U_k = \Q \bar\U_k$, $\tilde\Si_k$, $ \tilde\V_k$---an approximate rank-$k$ truncated SVD of $\A$.
\end{algorithmic}
%\end{small}
\end{algorithm}

{\bf The Krylov Subspace Method.}
From Algorithm~\ref{alg:power_method} we can see that the power iteration repeats $t$ times, 
but only the output of the last iteration $\C^{(t)}$ is used.
In fact, the intermediate results $\C^{(0)} , \cdots , \C^{(t)}$ are also useful.
The matrix $\K = [\C^{(0)}, \cdots , \C^{(t)}]\in \RB^{m\times (t+1)c}$ is well known as the Krylov matrix,
and $\rg (\K)$ is called the Krylov subspace.
We show the Krylov subspace method in Algorithm~\ref{alg:power_method}, which differs from simultaneous power iteration in only one line.
It turns out that the Krylov subspace method converges much faster than the power iteration \citep{saad2011numerical}.
Very recently, \cite{musco15stronger} showed that with $t = \frac{\log n}{\sqrt{\epsilon}}$ power iteration, 
the $1+\epsilon$ spectral norm bound \eqref{eq:power_spectral} holds with high probability.
This result is evidently stronger than the simultaneous power iteration.

It is worth mentioning that the Krylov subspace method described in Algorithm~\ref{alg:power_method} is a simplified version, 
and it may be instable when $t$ is large.
This is because the columns of $\C^{(0)}, \cdots , \C^{(t)}$ tend to be linearly dependent as $t$ grows.
In practice, re-orthogonalization or partial re-orthogonalization are employed to prevent the instability from happening \citep{saad2011numerical}.

\section{Kernel Approximation}
\label{sec:kerapp}

Kernel methods are important tools in machine learning, computer vision, and data mining \citep{ScholkopfBook:2002,ShaweTaylorBook:2004,Vapnik:1998,RassmussenWilliams}.
For example, kernel ridge regression (KRR),  Gaussian processes,  kernel support vector machine (KSVM), spectral clustering, and kernel principal component analysis (KPCA)
are classical nonlinear models for regression, classification, clustering, and dimensionality regression.
Unfortunately, the lack of scalability has always been the major drawback of kernel methods.
The three steps of most kernel methods---forming the kernel matrix, training, generalization---can all be prohibitive in big-data applications.

Specifically, suppose we are given $n$ training data and $m$ test data, all of $d$ dimension.
Firstly, it takes $\OM (n^2 d)$ time to form an $n\times n$ kernel matrix $\K$, e.g.,\ the Gaussian RBF kernel matrix.
Secondly, the training requires either SVD or matrix inversion of the kernel matrix.
For example, spectral clustering, KPCA, Isomap \citep{tenenbaum2000global},  and Laplacian eigenmaps \citep{Belhumeur:2003}
compute  the top $k$ singular vectors of the (normalized) kernel matrix,
where $k$ is the number of classes or the target dimensionality.
This costs $\OM (n^2 k)$ time and $\OM (n^3)$ memory.
Thirdly, to generalize the trained model to the test data,
kernel methods such as KRR, KSVM, KPCA cost $\OM (n m d)$ time to form an $n\times m$ cross kernel matrix between the training and test data.
If $m$ is as large as $n$, generalization is as challenging as training.

Low rank approximation is the most popular approach to scalable kernel approximation.
If we have the low rank approximation $\K \approx \C \X \C^T$,
then the approximate eigenvalue decomposition can be immediately obtained by
\begin{eqnarray*}
\K  \approx  \C \X \C^T
 =  \U_C \underbrace{(\Si_C \V_C^T \X \V_C \Si_C)}_{=\Z } \U_C^T
 =  (\U_C \U_Z) \Lam_Z (\U_C \U_Z)^T .
\end{eqnarray*}
Here $\C = \U_\C \Si_\C \V_\C^T$ is the SVD and
$\Z = \U_\Z \Lam_\Z \U_\Z^T$ is the spectral decomposition.
Since the tall-and-skinny matrix $\U_C \U_Z$ has orthonormal columns
and the diagonal entries of $\Lam_Z$ are in the descending order,
the leftmost columns of $\U_C \U_Z$ are approximately the top singular vectors of $\K$.
This approach only costs $\OM (n c^2)$ time, where $c$ is the number of columns of $\C$.
Our objective is thereby to find such a low rank approximation.

{\bf Difference from Randomized SVD.}
Why cannot we directly use the randomized SVD to approximate the kernel matrix?
The randomized SVD assumes that the matrix is fully observed;
unfortunately, this is not true for  kernel methods.
When the number of data samples is million scale, even forming the kernel matrix is impossible.
Therefore, the primary objective of kernel approximation is to avoid forming the whole kernel matrix.
The existing random projection methods all require the full observation of the matrix,
so random projection is not a feasible option.
We must use column selection in the kernel approximation problem.

{\bf The Prototype Algorithm.}
Let $\S$ be an $n\times c$ sketching matrix
and let $\C = \K \S$.
It remains to find the $c\times c$ intersection matrix $\X$.
The most intuitive approach is to minimize the approximation error by
\begin{equation} \label{eq:prototype}
\X^\star \; = \; \argmin_{\X} \, \big\| \K - \C \X \C^T \big\|_F^2
\; = \; \C^\dag \K (\C^\dag)^T ,
\end{equation}
where the second equality follows from Theorem~\ref{thm:solution_cur}.
This method was proposed by \citet{halko2011finding} for approximating symmetric matrix.
\citet{WangLuoZhang} showed that by randomly sampling $\OM (k/\epsilon)$ columns of $\K$ to form $\C$ by a certain algorithm,
the approximation is high accurate:
\[
\big\| \K - \C \X^\star \C^T \big\|_F^2
\; \leq \;
(1+\epsilon) \big\| \K - \K_k \big\|_F^2 .
\]
This upper bound matches the lower bound $c \geq 2 k / \epsilon$ up to a constant factor \citep{WangLuoZhang}.
Unfortunately, the prototype algorithm has two obvious drawbacks.
Firstly, to compute the intersection matrix $\X^\star$,
every entry of $\K$ must be known.
As is discussed, it takes $\OM (n^2 d)$ time to form the kernel matrix $\K$.
Secondly, the matrix multiplication $\C^\dag \K$ costs $\OM (n^2 c)$ time.
In sum, the prototype algorithm costs $\OM (n^2 c + n^2 d)$ time.
Although it is substantially faster than the exact solution, the prototype algorithm has the same time complexity as the exact solution.

{\bf Faster SPSD Matrix Sketching.}
Since $\C = \K \S$ has much more rows than columns,
the optimization problem \eqref{eq:prototype} is strongly over-determined.
\citet{wang2015towards} proposed to use sketching to approximately solve \eqref{eq:prototype}.
Specifically,  let $\PP$ be a certain $n\times p$ column selection matrix with $p \geq c$ and compute
\begin{equation*}
\tilde\X \; = \; \argmin_{\X} \, \big\| \PP^T ( \K - \C \X \C^T ) \PP \big\|_F^2
\; = \; (\PP^T \C)^\dag (\PP^T \K \PP) (\C^T \PP)^\dag .
\end{equation*}
In this way, we need only $nc + p^2$ entries of $\K$ to form the approximation $\K \approx \C \tilde\X \C^T$.
The intersection matrix $\tilde\X$ can be computed in $\OM (ncd + p^2 d + p^2 c)$ time, given $\S$ and $n$ data points of $d$ dimension.
\citet{wang2015towards} devised an algorithm that sets $p = \sqrt{n} c / \sqrt{\epsilon}$ and very efficiently forms the column selection matrix $\PP$;
and the following error bound holds with high probability:
\[
\big\| \K - \C \tilde\X \C^T \big\|_F^2
\; \leq \;
(1+\epsilon) \, \min_{\X} \big\| \K - \C \X \C^T \big\|_F^2 .
\]
By this choice of $p$, the overall time cost is linear in $n$.

Motivated by the matrix ridge approximation of \cite{ZhangML:2014},  \citet{WangZhangKDD:2014} proposed a spectral shifting kernel approximation method.  
%   It is defined as
%  \[ \K \approx {\C} \U^ {\C}^T + \delta^{\textrm{ss}} \I_n, 
%where  $\delta^{\textrm{ss}} \geq 0$.
When the spectrum of $\K$ decays slowly, the shifting term  helps to improve the approximation accuracy and numerical stability.
\cite{WangLuoZhang}  also showed  that the spectral shifting approach can be used to improve other kernel  approximation models such as 
the memory efficient kernel approximation (MEKA) model \citep{si2014memory}.

{\bf The Nystr\"om Method}
is the most popular kernel approximation approach.
It is named after its inventor \cite{nystrom1930praktische}
and gained its popularity in the machine learning society after its application in Gaussian procession regression \citep{williams2001using}.
Let $\S$ be a column selection matrix, $\C = \K \S$, and $\W = \S^T \K \S$.
The Nystr\"om method approximates $\K$ by $\C \W^\dag \C^T$.
In fact, the Nystr\"om method is a special case of the faster SPSD matrix sketching where $\PP$ and $\S$ are equal.
This also indicates that the Nystr\"om method is an approximate solution to \eqref{eq:prototype}.
\citet{gittens2013revisiting} offered comprehensive error analysis of the Nystr\"om method.
The Nystr\"om method has been applied to solve million scale kernel methods \citep{talwalkar2013large}.
But unlike the faster SPSD matrix sketching, the Nystr\"om method cannot generate high quality approximation.
The lower bound \citep{WangZhangJMLR:2013} indicates that the Nystr\"om method cannot attain $(1{+}\epsilon)$ relative-error bound
unless it is willing to spend $\Omega (n^2 k/\epsilon)$ time.

To this end, we have shown how to efficiently approximate any kernel matrix and use the obtained low rank approximation to speed up training.
We will introduce efficient generalization using the CUR matrix decomposition in the next section.

\section{The CUR Approximation}
\label{sec:curapp}

Let $\A$ by any $m\times n$ matrix.
The CUR matrix decomposition is formed by selecting $c$ columns of $\A$ to form $\C \in \RB^{m\times c}$,
$r$ rows to form $\R \in \RB^{r\times n}$, and computing an intersection matrix $\U \in \RB^{c\times r}$ such that $\C \U \R \approx \A$.
In this section, we first discussion the motivations and then describe algorithms and error analyses.

{\bf Motivations.}
Firstly, let us continue the generalization problem of kernel methods which remains unsolved in the previous section.
Suppose we are given $n$ training data and $m$ test data, all of $d$ dimension.
To generalize the trained model to the test data,
supervised kernel methods such as Gaussian processes and KRR require evaluating the kernel function of every train and test
data pair---that is to form an $m\times n$ cross kernel matrix $\K_*$---which costs $\OM (m n d)$ time.
%As for KPCA with target rank $k$, the trained model is an $n\times k$ matrix $\M$, and it costs $\OM (m n k)$ time to apply $\M$ to $\K_*$.
By the fast CUR algorithm described later in this section,
the approximation $\K_* \approx \C \U \R$ can be obtained in time linear in $d (m+n)$.
With such a decomposition at hand, the matrix product $\K_* \M \approx \C \U \R \M$ can be computed in $\OM ( nrk+mck )$ time.
In this way, the overall time cost of generalization is linear in $m+n$.

Secondly, CUR forms a compressed representation of the data matrix, as well as the truncated SVD,
and it can be very efficiently converted to the SVD-like form:
\begin{eqnarray*}
\A & \approx & \C \U \R
\; = \; \U_C \underbrace{\Si_C \V_C^T \U \U_R \Si_R}_{=\B} \V_R^T
\; = \; (\U_C \U_B) \Si_B (\V_R \V_B)^T .
\end{eqnarray*}
Here $\C = \U_C \Si_C \V_C^T$, $\R = \U_R \Si_R \V_R^T$, $\B = \U_B  \Si_B \V_R$ are the SVD.
Since CUR is formed by sampling columns and rows,
it preserves the sparsity and nonnegativity of the original data matrix.
The sparsity makes CUR cheaper to store than SVD,
and the nonnegativity makes CUR a nonnegative matrix factorization.

Thirdly, CUR consists of the actual columns and rows, and thus it enables human to to understand and interpret the data.
In comparison, the basis vectors of SVD has little concrete meaning.
An example of \cite{DrineasCUR:2008} and \cite{mahoney2009matrix} has well shown this viewpoint; that is,
the vector $[(1/2)\textrm{age} - (1/\sqrt{2})\textrm{height} + (1/2)\textrm{income}]$,
the sum of the significant uncorrelated features from a data set of people's features,
is not particularly informative.
%\citet{kuruvilla2002vector} have also claimed:
%``it would be interesting to try to find basis vectors for all experiment vectors, using actual experiment vectors and not artificial bases that offer little insight.''
Therefore, it is of great interest to represent a data matrix in terms of a small number of actual columns and/or actual rows of the matrix.

{\bf Column Selection.}
Several different column selection strategies have been devised,
among which the leverage score sampling \citep{DrineasCUR:2008} and the adaptive sampling \citep{WangZhangJMLR:2013,boutsidis2014optimal} attain relative error bounds.
In particular, \citet{boutsidis2014optimal} showed that with $c = \OM (k/\epsilon)$ columns and $r = \OM (k/\epsilon)$ rows selected by adaptive sampling
to form $\C$ and $\R$,
\[
\min_\X \| \A - \C \X \R \|_F^2
\; \leq \; (1+\epsilon) \|\A - \A_k\|_F^2
\]
holds in expectation. A further refinement was developed by  \citet{woodruff2014sketching}.
We will not go to the details of the leverage score sampling or adaptive sampling.
The users only need to know that such algorithms randomly sample columns/rows according to some non-uniform distributions.
Unfortunately, it requires observing the whole matrix $\A$ to compute such non-uniform distributions,
thus such column selection algorithms cannot be applied to speed up computation.
It remains an open problem  whether there is a relative-error sampling algorithm that needs not observing the whole of $\A$.
In practice, the users can simply sample columns/rows uniformly without replacement,
which usually has acceptable empirical performance.

{\bf The Intersection Matrix.}
With the selected columns $\C$ and rows $\R$ at hand,
we can simply compute the intersection matrix by
\begin{eqnarray} \label{eq:prototype_cur}
\U^\star
& = & \argmin_\U \, \big\| \A - \C \U \R \big\|_F^2
\; = \; \C^\dag \A \R^\dag.
\end{eqnarray}
Here the second equality follows from Theorem~\ref{thm:solution_cur}.
This approach has been used by \cite{stewart1999four,WangZhangJMLR:2013,boutsidis2014optimal}.
This approach is very similar to the prototype SPSD matrix approximation method in the previous section,
and it costs at least $\OM (m n \cdot \min\{ c, r\})$ time and requires observing every entry of $\A$.
Apparently, it cannot help speed up matrix computation.

\cite{WangZhangZhang} proposed a more practical CUR decomposition method which solves \eqref{eq:prototype_cur} approximately.
The method first draws two column selection matrices $\PP_C \in \RB^{m\times p_c}$ and $\PP_R \in \RB^{n\times p_r}$ ($p_c , p_r \geq c, r$),
which costs $\OM (m c^2 + n r^2)$ time.
It then computes the intersection matrix by
\begin{eqnarray*}
\tilde\U
= \argmin_\U \, \big\| \PP_C^T ( \A - \C \U \R )\PP_R \big\|_F^2
= (\PP_C^T \C)^\dag (\PP_C^T \A  \PP_R) (\R \PP_R)^\dag.
\end{eqnarray*}
This method needs observing only $p_c \times p_r$ entries of $\A$,
and the overall time cost is $\OM (p_c p_r \cdot \min \{ c, r\} + m c^2 + n r^2)$.
When
\[
p_c \geq \OM \Big(c \sqrt{\min \{m,n\} / \epsilon}\Big)
\quad \textrm{ and } \quad
p_r \geq \OM \Big(r \sqrt{\min \{m,n\} / \epsilon}\Big),
\]
the following inequality holds with high probability:
\[
\big\| \A - \C \tilde\U \R \big\|_F^2
\; \leq \; (1+\epsilon) \, \min_\U \, \big\| \A - \C \U \R \big\|_F^2 .
\]
In sum, a high quality CUR decomposition can be computed in time linear in $\min\{ m , n\}$.

\begin{acknowledgements}
\addcontentsline{toc}{chapter}{Acknowledgements} 
I would like to thank my graduate students Cheng Chen, Luo Luo, Shusen Wang, Haishan Ye, and Qiaomin Ye.  Specifically, Cheng Chen, Luo Luo and Qiaomin Ye helped to proofread
the whole manuscript.  Haishan Ye helped to revise Chapter 9.2, and Shusen Wang helped to revise  Chapter 10.   I would also like to thank  other  students who
took my course `` Matrix Methods in Massive Data Analysis'' in the summer term 2015.  They  helped to improve the lecture notes, 
which provide the main materials  for this tutorial. 
%Xinmei Zhu presented a nice solution to Example~\ref{exm:sva}, which is a homework assignment of my course.   
\end{acknowledgements}

\backmatter  % references

\bibliographystyle{plainnat}
\bibliography{matrix}

%}
%\end{small}

%\newpage

\end{document}